\documentclass[twoside,11pt]{article}
\usepackage{amsthm}
\let\savedproof\proof
\let\savedendproof\endproof
\usepackage[preprint]{jmlr2e}
\usepackage{amsmath,mathrsfs,enumerate,color}
\usepackage{algorithm,algorithmic,caption,multirow,booktabs,tikz,subcaption}
\let\proof\savedproof
\let\endproof\savedendproof

\usepackage{mathtools,bbding,fdsymbol,microtype,tabularx,placeins}
\newtheorem{assumption}[theorem]{Assumption}

\newcommand{\mc}[1]{\mathcal{#1}}
\newcommand{\mb}[1]{\mathbb{#1}}
\newcommand{\mf}[1]{\mathfrak{#1}}

\ShortHeadings{Latent-Space Importance Weighting}{Jiao, Kang, Xiang, Yang and Zhou}
\firstpageno{1}
\hypersetup{pdftitle={Latent-Space Importance Weighting under Covariate Shift},
            pdfauthor={Yuling Jiao, Lican Kang, Yang Xiang, Jerry Zhijian Yang, Fusheng Zhou}}

\begin{document}
\title{Latent-Space Importance Weighting under Covariate Shift}
\author{\name Yuling Jiao \email yulingjiaomath@whu.edu.cn \\
\addr School of Artificial Intelligence, \\
and Hubei Key Laboratory of Computational Science\\
Wuhan University, Wuhan, 430072, China
\AND
\name Lican Kang \email kanglican@whu.edu.cn \\
\addr Institute for Math and AI, \\
School of Artificial Intelligence \\
Hubei Key Laboratory of Computational Science,\\
Wuhan University, Wuhan, 430072, China
\AND
\name Yang Xiang \email maxiang@ust.hk \\
\addr Department of Mathematics, \\
The Hong Kong University of Science and Technology, Hong Kong, China\\
Algorithms of Machine Learning and Autonomous Driving Research Lab\\
HKUST Shenzhen-Hong Kong Collaborative Innovation Research Institute, Shenzhen, China
\AND
\name Jerry Zhijian Yang \email zjyang.math@whu.edu.cn \\
\addr School of Mathematics and Statistics, \\
and Hubei Key Laboratory of Computational Science\\
Wuhan University, Wuhan, 430072, China
\AND
\name Fusheng Zhou\thanks{Corresponding author.} \email fushengzhou@ust.hk \\
\addr School of Mathematics and Statistics, \\
Wuhan University, Wuhan, 430072, China\\
Department of Mathematics, \\
The Hong Kong University of Science and Technology\\
Hong Kong, China}
\maketitle
\begin{abstract}%
We propose a three-stage algorithm that combines representation learning with importance weighting to address covariate shift. The algorithm learns a representation, estimates the density ratio between domains in the learned space, and fits a weighted predictor. We establish finite-sample excess-risk bounds for classification that quantify the effect of representation insufficiency. Under sufficiency and regularity conditions, we derive explicit convergence rates for neural networks, replacing the ambient dimension by the latent dimension in the rate exponents. For fixed finite-dimensional predictor classes, we characterize strict reductions in the leading excess-risk constant relative to original-space weighting, including a model with learned representations and estimated weights. Experiments on synthetic and real datasets demonstrate the benefits of combining representation learning and latent weighting, with the proposed method achieving the highest mean accuracy among four controlled variants across six datasets.

\end{abstract}

\section{Introduction}\label{introduction}
Covariate shift arises when the covariate distribution changes between source and target populations while the conditional response law remains unchanged \citep{quinonero2022dataset,sugiyama2012machine}. It is an important setting in transfer learning (TL) \citep{pan2009survey}, with applications in language processing \citep{blitzer2006domain}, visual recognition \citep{saenko2010adapting}, and clinical risk prediction \citep{zhou2024doubly}. The goal is to learn a predictor for the target population using labeled source data and unlabeled target covariates.

Importance weighting addresses this goal by expressing target risk as a source expectation weighted by the target-to-source covariate density ratio \citep{shimodaira2000improving,huang2006correcting,cortes2008sample,cortes2010learning}; it also supports model selection through importance-weighted cross-validation \citep{sugiyama2007covariate}. In high dimensions, however, the ratio must describe distributional changes in every covariate, even when only a few features determine the response. This creates two statistical costs: estimating a high-dimensional ratio and carrying weight variation unrelated to prediction. A shift in an independent nuisance coordinate, for example, changes the weights without changing the predictive relationship and can increase weighted-loss variance even when the ratio is known.

Sufficient representations provide a way to address both costs. By retaining the conditional response law in a lower-dimensional space, they permit latent importance weighting that preserves population target risk \citep{stojanov2019low} and averages out weight variation among inputs with the same representation. The challenge is to realize these gains with a learned representation and estimated weights. A fitted nonlinear map can lose predictive information and induce different conditional response laws across domains. Its quality must therefore be connected to density-ratio estimation and predictor fitting in a target-risk analysis.

We make this connection through a three-stage algorithm that combines supervised representation learning with latent-space importance weighting. The algorithm learns a representation from source labels, estimates the density ratio between the domain distributions induced by this fitted map, and fits a predictor by minimizing the weighted source empirical risk. The first stage uses the dependence reward and reference-distribution penalty of Huang et al.~\citeyearpar{huang2024deep} to retain predictive information; the second corrects the remaining marginal shift. This separation allows the source and target latent distributions to differ. We analyze the resulting predictor for binary classification with logistic loss.

Our results quantify gains in both estimation rates and the risk of the fitted predictor. Under strict sufficiency and uniform regularity, the neural excess-risk bounds in Table \ref{tab_com} replace the ambient dimension by the latent dimension in the ratio-estimation exponent while retaining the same low-dimensional prediction term. Under equal smoothness and comparable fold sizes, this yields a larger overall convergence exponent than original-space weighting. For a fixed sufficient map and a fixed finite-dimensional predictor class, a separate risk expansion establishes a strict reduction in the leading expected excess-risk constant under an explicit nondegeneracy condition. A learnable coordinate-selection model realizes this improvement with estimated weights, while the general finite-sample bound also quantifies the effect of representation insufficiency. These results lead to the following contributions.

\begin{itemize}
\item \textit{A three-stage method and its target-risk analysis.} We combine learned representations, latent density-ratio estimation and weighted prediction, and establish a finite-sample excess-risk bound that tracks the errors from these stages. The bound includes information loss and conditional response mismatch when the fitted representation is insufficient.
\item \textit{Explicit neural-network convergence rates.} We derive rates for fully connected ReLU networks under strict sufficiency and uniform regularity. For common H\"older smoothness $\beta_{\mathcal H}$, prediction sample size $n$ and domain sample size $m\asymp n$, the excess-risk upper bound improves from $\widetilde O(n^{-\beta_{\mathcal H}/(2\beta_{\mathcal H}+p)})$ to $\widetilde O(n^{-\beta_{\mathcal H}/(2\beta_{\mathcal H}+d)})$ when weighting in a $d$-dimensional representation of $p$-dimensional inputs, with $d<p$ (Table \ref{tab_com}).
\item \textit{Risk comparisons after fitting.} For a fixed sufficient map and a fixed finite-dimensional predictor class, we derive $n^{-1}$ expected excess-risk expansions and establish a strictly smaller leading constant for latent weighting than for original-space weighting under an explicit nondegeneracy condition. A coordinate-selection model realizes this gain with a learned representation and estimated weights.
\item \textit{Empirical evaluation.} Across six synthetic and real datasets, the proposed method achieves the highest mean target accuracy among four controlled variants, demonstrating the benefit of combining representation learning and latent weighting. We also compare with domain-adaptation baselines on Office-Caltech under random and pretrained initialization.
\end{itemize}

Table \ref{tab_com} summarizes Theorem \ref{Th5} and Propositions \ref{th4} and \ref{contl gen} under their sufficiency and regularity conditions. Here $p,d$ are the ambient and latent dimensions; $n,m$ are the prediction- and domain-fold sizes; and $\widetilde O$ suppresses logarithmic factors. The H\"older indices $\beta_{\mathcal H},\beta_{\ddot{\mathcal H}}$ correspond to the learned-space and original-space Bayes scores, and $\beta_{\ddot\Phi},\beta_\Phi$ to the corresponding log density ratios. The representation-based methods share the same sufficient fitted map and an additional labeled fold of size $n_R$. Section \ref{comparison} analyzes the rate differences and the risk gains after fitting.

\begin{table*}[htbp]
\centering
\captionsetup{font=normalsize}

{\normalsize
\renewcommand{\arraystretch}{1.25}
\begin{tabular}{lcl}
\toprule
Method & Representation & Upper bound\\
\midrule
Latent weights (ours) \eqref{MeA} & \Checkmark &
$\widetilde O (n^{-\frac{\beta_{\mathcal H}}{2\beta_{\mathcal H}+d}}
       +m^{-\frac{\beta_{\ddot\Phi}}{2\beta_{\ddot\Phi}+d}})$\\[4pt]
Original weights \eqref{MeB} & \Checkmark &
$\widetilde O (n^{-\frac{\beta_{\mathcal H}}{2\beta_{\mathcal H}+d}}
       +m^{-\frac{\beta_\Phi}{2\beta_\Phi+p}})$\\[4pt]
No representation \eqref{MeD} & \XSolidBrush &
$\widetilde O(n^{-\frac{\beta_{\ddot{\mathcal H}}}{2\beta_{\ddot{\mathcal H}}+p}}
       +m^{-\frac{\beta_\Phi}{2\beta_\Phi+p}})$\\
\bottomrule
\end{tabular}

}
\caption{Target-domain excess-risk bounds}
\label{tab_com}
\end{table*}

\FloatBarrier
\subsection{Related Work}\label{sec:rw}
\noindent\textit{Low-dimensional weighting and learned representations.}
Stojanov et al.~\citeyearpar{stojanov2019low} combine sufficient dimension reduction with importance weighting, establishing population risk equivalence and analyzing projected predictors with kernel mean matching under a fixed target design. We extend the analysis to the population target risk of a predictor trained with an independently learned nonlinear map, retaining representation insufficiency and deriving explicit neural-network rates. Fang et al.~\citeyearpar{fang2020rethinking} address the coupling between feature extraction and weight estimation through dynamic importance weighting. Their ratio-preservation result assumes a fixed invertible transformation; our analysis permits noninvertible sufficient maps and estimates their induced marginal ratios.

\noindent\textit{Statistical efficiency under covariate shift.}
Maia Polo and Vicente~\citeyearpar{polo2023effective} relate effective sample size to R\'enyi divergence and generalization, and show how linear dimension reduction and feature selection can improve effective sample size. Kiyohara et al.~\citeyearpar{kiyohara2024off} prove unbiasedness and variance reduction for latent importance sampling in off-policy evaluation under sufficient slate abstraction. We carry the conditional-averaging mechanism through weighted fitting to an expected target-risk expansion, including estimated weights. Shimodaira~\citeyearpar{shimodaira2000improving} distinguishes the benefits of weighting under model misspecification from the efficiency of ordinary likelihood under correct specification. Our learnable model exhibits both the bias from ignoring relevant shift and the risk inflation from weighting irrelevant shift.

Related approaches improve weighting through regression adjustment \citep{li2020robust} or study independence-driven weights for stable feature selection under unknown test distributions \citep{xu2022theoretical}. Kernel learning rates under bounded-ratio and second-moment conditions include regimes where unweighted learning is rate-optimal \citep{feng2024towards}; deep regression under covariate shift is studied in \cite{feng2024deep}. Doubly robust transfer inference and subgroup adaptation address high-dimensional adjustment, heterogeneous populations and imbalanced target groups \citep{zhou2024doubly,zhou2024model}. Our analysis uses observed target covariates to quantify how learned representations affect weight estimation and the risk of the fitted predictor.

\noindent\textit{Density-ratio estimation.}
Kernel mean matching \citep{huang2006correcting,gretton2008covariate,yu2012analysis}, direct ratio estimation \citep{sugiyama2008direct,kanamori2009least}, and probabilistic classification \citep{bickel2007discriminative,bickel2009discriminative} estimate importance weights without estimating the two densities separately. Variational convex-risk minimization provides estimators of ratios and divergences \citep{nguyen2007estimating}, while nonnegative correction controls overfitting in deep Bregman ratio estimation \citep{kato2021non}. We use bounded domain-logistic classification to estimate the ratio induced by the learned representation.

\noindent\textit{Sufficient representations and domain alignment.}
Huang et al.~\citeyearpar{huang2024deep} develop the supervised representation objective used in our first stage. We incorporate its fitted output into covariate-shift correction and track representation quality in the target-risk bound. Feature learning for adaptation also includes structural correspondence learning \citep{blitzer2006domain}, maximum mean discrepancy embedding and transfer component analysis \citep{pan2008transfer,pan2010domain}, domain-invariant projection \citep{baktashmotlagh2013unsupervised}, joint distribution adaptation \citep{long2013transfer}, and intermediate tasks constructed from landmarks \citep{gong2013connecting}. Our representation stage seeks predictive sufficiency, and the weighting stage corrects the remaining marginal shift without requiring identical latent distributions across domains. Neural approximation theory relates approximation rates to network depth, smoothness and compositional structure \citep{yarotsky2018optimal,yarotsky2020phase,shen2019nonlinear,shen2019deep,petersen2018optimal,lu2021deep}.

\subsection{\texorpdfstring{Outline}{Outline}}
We develop the three-stage procedure in Section \ref{problem setting}, then analyze its target risk in Section \ref{convergence rate}. Section \ref{comparison} examines where latent weighting gains over original-space weighting, moving from rate comparisons to fitted-risk expansions and a learnable example. The experiments in Section \ref{experiments} assess how representation, weighting and initialization affect accuracy. Section \ref{conclusion} discusses the implications and open problems. Proofs, general bounds for insufficient representations and experimental details are collected in the appendix.

\section{Problem Formulation}\label{problem setting}
\subsection{\texorpdfstring{Preliminaries}{Preliminaries}}

\noindent
\textit{
Notation.
}
For the binary classification analysis, let $\mc X\subseteq\mathbb R^p$ and $\mathcal Y=\{0,1\}$ denote the input and output spaces. Let $(X^{\mc S},Y^{\mc S})$ and $(X^{\mc T},Y^{\mc T})$ be covariate-response pairs with joint laws $P_{\mc S}$ and $P_{\mc T}$ in the source and target domains. For $k\in\{\mc S,\mc T\}$, $P_k(X)$ denotes the feature marginal.
Let $P_\mu:=(P_{\mc S}(X)+P_{\mc T}(X))/2$ be the balanced feature mixture. Its joint domain-labeled law assigns $Z=-1$ to source observations and $Z=1$ to target observations.  
The prediction fold is $D_{\rm pred}:=\{(X_i^{\mc S},Y_i^{\mc S})\}_{i=1}^n$, with i.i.d. observations from $P_{\mc S}$. The unlabeled target sample $\{X_j^{\mc T}\}_{j=1}^m$ is i.i.d. from $P_{\mc T}(X)$ and is used to construct the domain-training fold $D_\mu$ in \eqref{ratio_data}.
For a probability measure $\mu$ on a measurable space and $1\leq s<\infty$, write $\|f\|_{L^s(\mu)}=(\int\|f\|^s\,d\mu)^{1/s}$, using absolute value for scalar functions and the Euclidean norm for vector functions. Write $\|f\|_{\infty,K}:=\sup_{x\in K}\|f(x)\|$, omitting $K$ when the domain is clear. Unless explicitly fixed, $C$ and $C_M$ denote positive constants that may change between occurrences; the subscript $M$ indicates dependence on $M$. For nonnegative $A,B$, write $A\asymp B$ if $c_1A\leq B\leq c_2A$ for constants $c_1,c_2>0$, and $A\lesssim B$ if $A\leq cB$ for some $c>0$.
Let $\mc M_d^p$ denote the measurable maps $\mathbb R^p\to\mathbb R^d$.

\noindent
\textit{
Sufficient Representation.
}
Given any $(X, Y) \in \mc{X} \times \mc{Y}$, let $R^* : \mb{R}^p \to \mb{R}^{d}$ be a representation function for any input $X \in \mc{X}$, where $d  \ll p$ and $d$ is prespecified. 
Sufficiency means that $Y$ and $X$ are conditionally independent given $R^*(X)$, i.e., $Y\Vbar X\mid R^*(X)$. Component independence (disentanglement) is a separate property encouraged by the representation-learning penalty \cite{huang2024deep}.

\begin{definition}[H\"older class]\label{def:hc}
The set $\mho^{\beta}([0,1]_d^p,B)$  represents the class of H\"older continuous functions mapping $[0,1]^p$ to $\mathbb{R}^d$,
where $p,d \in \mathbb{N}^+$, 
$\beta > 0$, with $\beta=\alpha+\tilde r$, $\alpha=\lceil\beta\rceil-1$, $\tilde r\in(0,1]$, and 
$B>0$ is a constant.
Specifically,  this class is  defined as
\begin{align*}
    \mho^{\beta}([0,1]_d^p,B) := \bigg\{
    &f=\{f_1,\ldots, f_d\}; f_i:[0,1]^p \to \mathbb{R}, \\
    &\max_{\|s\|_{1}\leq \alpha}\|\partial^{s}f_i\|_{\infty}\leq B,
    \max_{\|s\|_{1} = \alpha}\sup_{x \neq y}\dfrac{|\partial^{s}f_i(x) - \partial^{s}f_i(y)|}{\|x-y\|_2^{\tilde{r}}} \leq B,~ i=1,\ldots, d  \bigg\},
\end{align*}
where 
$s=(s_1,\ldots,s_p)\in\mathbb N_0^p$, $\|s\|_1=\sum_{j=1}^p s_j$, and $\partial^s f_i=\partial_{x_1}^{s_1}\cdots\partial_{x_p}^{s_p}f_i$, $i = 1,\ldots,d$. 
For convenience, 
we denote by $\mho^{\beta}([0,1]^p,B)$  the function class for which $d=1$. 
\end{definition}

\begin{definition}[ReLU DNNs]\label{def:dnn}
Let $\mathcal F_{\mathcal D_{\rm net},\mathcal W,\mathcal S,\mathcal B}$ be the class of ReLU DNNs with hidden-layer count $\mathcal D_{\rm net}$, width $\mathcal W$, parameter count $\mathcal S$ and output bound $\mathcal B$. Its networks $f_\theta$ can be parameterized as
\begin{equation*}
\begin{split}
   &f_{0}(x) = x,\\
    &f_{i + 1 }(x) = \varrho(W_if_i(x) + b_i), \quad  i = 0, \dots, \mathcal D_{\rm net}-1,\\
    &f_{\theta}(x)=W_{\mathcal D_{\rm net}}f_{\mathcal D_{\rm net}}(x)+b_{\mathcal D_{\rm net}}.
\end{split}
\end{equation*}
Here $\theta=(W_i,b_i)_{i=0}^{\mathcal D_{\rm net}}$ is the network parameter vector. We denote $\varrho(x):= \max\{0, x\}$ as the ReLU function applied componentwise. 
For $i=0,\ldots,\mathcal D_{\rm net}$, $W_i\in\mathbb R^{d_{i+1}\times d_i}$ and $b_i\in\mathbb R^{d_{i+1}}$. There are $\mathcal D_{\rm net}$ hidden layers and one affine output layer; $d_0$ is the input dimension and $d_{\mathcal D_{\rm net}+1}$ is the output dimension. Set $\mathcal W=\max_{0\leq i\leq\mathcal D_{\rm net}+1}d_i$ and $\mathcal S=\sum_{i=0}^{\mathcal D_{\rm net}}d_{i+1}(d_i+1)$, counting all weights and biases of a fully connected network, including positions whose values are zero. For a network class $J$, the subscripts in $\mathcal D_J,\mathcal W_J,\mathcal S_J$ specify its depth, width and total parameter budget. The parameter $\mathcal B$ specifies the output bound; score clipping is defined below.
\end{definition}

\par

\subsection{\texorpdfstring{Covariate Shift and Weighted Prediction}{Covariate Shift and Weighted Prediction}}
We begin with a shared sufficient representation and the covariate-shift conditions.

\begin{assumption}\label{ass2.0}
Let $(X^{\mc{S}}, Y^{\mc{S}}) \in \mc{X} \times \mc{Y}$ represent the covariate-response pair in the source domain $\mc{S}$, and $(X^{\mc{T}}, Y^{\mc{T}}) \in \mc{X} \times \mc{Y} $  represent the corresponding pair in the target domain $\mc{T}$. We assume that there exists a shared sufficient representation function $R^* : \mathbb{R}^p \to \mathbb{R}^{d}$ 
for both $(X^{\mc{S}}, Y^{\mc{S}})$ and $(X^{\mc{T}}, Y^{\mc{T}})$. 
\end{assumption} 
Define $U^{\mc{S}}:=R^*(X^{\mc{S}})$ and $U^{\mc{T}}:=R^*(X^{\mc{T}})$, with their respective distributions denoted by 
$P_{\mc S}^{R^*}$ and $P_{\mc T}^{R^*}$ over $\mathbb R^d$, where $P_k^R$ denotes the law of $R(X)$ under domain $k\in\{\mc S,\mc T\}$. Write $U=R^*(X)$ when the domain superscript is omitted.
In the latent space $\mathbb{R}^{d}$, let $\eta^*(u)$ denote the conditional class probability and let $h^*(u)$ denote the corresponding real-valued prediction score:
\begin{equation}\label{eq:hf}
\begin{split}
    \eta^*(u)&:=P_{\mc{T}}(Y^{\mc{T}}=1\mid U^{\mc{T}}=u)
    =\mb{E}[Y^{\mc{T}}\mid U^{\mc{T}}=u],\\
    h^*(u)&:=\log\frac{\eta^*(u)}{1-\eta^*(u)},\qquad
    \eta^*(u)=\sigma(h^*(u)),\quad \sigma(t):=\frac{1}{1+e^{-t}}.
\end{split}
\end{equation}

We assume $0<\eta^*(u)<1$ for $P_{\mc T}^{R^*}$-almost every $u$, so that $h^*$ is finite. Throughout the binary classification analysis, $h$, $\hat h$, and $h^*$ are scores, whereas $\sigma(h)$, $\sigma(\hat h)$, and $\eta^*$ are probabilities.
For the density-ratio constructions, assume $P_{\mc T}(X)\ll P_{\mc S}(X)$ and define the original-space ratio $r^*(x):=\frac{dP_{\mc T}(X)}{dP_{\mc S}(X)}(x)$. Finite optimal domain scores require $0<r^*<\infty$ $P_{\mc S}(X)$-almost everywhere; Assumption \ref{assum1} imposes uniform bounds for the finite-sample analysis. The oracle latent ratio is $r_d^*(u):=\frac{dP_{\mc T}^{R^*}}{dP_{\mc S}^{R^*}}(u)$, with the pushforward notation defined above. Absolute continuity passes to these pushforward laws.

\begin{assumption}\label{ass2.1}
Assume that $P_{\mc{S}}(Y^{\mc{S}}|U^{\mc{S}})=P_{\mc{T}}(Y^{\mc{T}}|U^{\mc{T}})$, where $P_{\mc{S}}(Y^{\mc{S}}|U^{\mc{S}})$ and $P_{\mc{T}}(Y^{\mc{T}}|U^{\mc{T}})$ represent the conditional probability distributions of the response variable  %
given the latent representation %
in the source and target domains, respectively.
\end{assumption}
Together with sufficiency of $R^*$, Assumption \ref{ass2.1} yields

\begin{equation}\label{eq2.02}
  P_{\mc{S}}(Y|X) =P_{\mc{T}}(Y|X),\quad P_{\mc{S}}(Y|U)=P_{\mc{T}}(Y|U).  
\end{equation}
Write $\eta(x):=P_{\mc S}(Y=1\mid X=x)=P_{\mc T}(Y=1\mid X=x)$ for the common class probability.

Let $\mc{H}$ be the hypothesis class from which we select an estimator for  $h^*$.
Our analysis includes predictor-class approximation error when $h^*$ does not belong to $\mc H$.
To measure the performance of %
this estimator on the target   data $(X^{\mc{T}}, Y^{\mc{T}})$,
we use the logistic loss for labels $y\in\{0,1\}$ and real-valued scores $a$ \citep{mohri2018foundations}:
\begin{equation*}
   \ell(a,y):=\log(1+e^a)-ya,\qquad y\in[0,1].
\end{equation*}

For binary labels this equals $\log(1+\exp(-(2y-1)a))$. The interval $[0,1]$ defines an affine extension in the label argument; the observed label $Y$ remains binary. In particular, $\mathbb E[\ell(a,Y)\mid U=u]=\ell(a,\eta^*(u))$.
The conditional expected loss has derivative $\sigma(a)-\eta^*(u)$ with respect to $a$, so its unique minimizer is the log-odds score $h^*(u)$ in \eqref{eq:hf}. Hence,
\begin{equation}\label{eq2.3}
     h^*\in\arg\min_{h\in \mc{M}^{d}_1} \mb{E}_{P_{\mc{T}}}\big[ \ell(h(U^{\mc{T}}),  Y^{\mc{T}})\big].
\end{equation}
However, it is not feasible to estimate $h^*$ directly via \eqref{eq2.3}, 
as the target response $Y^{\mc{T}}$ is typically inaccessible in practice.
To tackle this issue, we consider the equivalent form of \eqref{eq2.3} as follows:
\begin{equation}\label{eq_es_h}
    \mb{E}_{P_{\mc{S}}}\big[r_d^*(U^{\mc{S}}) \cdot \ell(h(U^{\mc{S}}), Y^{\mc{S}})\big] = \mb{E}_{P_{\mc{T}}}\big[\ell(h(U^{\mc{T}}), Y^{\mc{T}})\big].
\end{equation}

\begin{lemma}\label{th1}
    Under Assumption \ref{ass2.1} and $P_{\mc T}^{R^*}\ll P_{\mc S}^{R^*}$, the score $h^*$ minimizes $\mb E_{P_{\mc S}}[r_d^*(U^{\mc S})\allowbreak\ell(h(U^{\mc S}),Y^{\mc S})]$ over $h\in\mc M_1^d$, and every minimizer agrees with $h^*$ $P_{\mc T}^{R^*}$-almost everywhere.
\end{lemma}
Lemma \ref{th1} expresses the usual importance-weighting identity in sufficient coordinates \citep{stojanov2019low}. We estimate the ratio induced by the fitted map and use a domain classifier for this purpose \citep{qin1998inferences,bickel2007discriminative}. Binary predictions threshold the fitted score at zero. Sections \ref{sec:rep}--\ref{sec:tlr} define the representation, ratio and prediction stages; Lemma \ref{lem:latent-oracle} quantifies the target-risk discrepancy for general fitted maps.

\subsection{Learn the Representation}\label{sec:rep}
In this section, we focus on obtaining effective representations for high-dimensional raw data. However, determining an efficient and simple criterion for representation learning remains a challenging problem \citep{bengio2013representation}. Inspired by \cite{huang2024deep}, we adopt sufficiency and disentanglement as goals of representation learning. We use a dependence reward $\mc V[R(X),Y]$ and a Gaussian-reference penalty to encourage task relevance and component independence. Let $\mc F$ be the representation-network class mapping $\mathbb R^p$ to $\mathbb R^d$; in Appendix \ref{app:representation-recovery}, $\mc F=\mc F_{n_R}$ is the prescribed sample-size-dependent sieve. On a representation-training fold of size $n_R$, with regularization coefficient $\lambda\geq0$, the fitted representation is defined by
\begin{equation}\label{eq:rep-training}
 \hat R\in\arg\min_{R\in\mc F}\widehat{\mc L}_{n_R}(R),\qquad
 \widehat{\mc L}_{n_R}(R):=-\widehat{\mc V}_{n_R}[R(X),Y]
                +\lambda\widehat{\mb D}_{f,n_R}(R).
\end{equation}

Here $\widehat{\mc V}_{n_R}$ is the squared distance covariance statistic computed from the labeled fold \cite{szekely2007measuring}; the empirical variational $f$-divergence penalty is described in Appendix \ref{Au_results} \citep{ali1966general}. The source-domain regularization reference is $\gamma_d:=\mc N(0,I_d)$, where $I_d$ is the $d\times d$ identity matrix; write $W\sim\gamma_d$ for a reference observation.

\subsection{Estimate the Density Ratio}\label{sec:dre}
We use logistic regression \cite{sugiyama2012density} to estimate density ratios, first in the original space and then in the learned latent space.
Given independent i.i.d. samples $\{X_i^{\mc S}\}_{i=1}^m$ from $P_{\mc S}(X)$ and $\{X_i^{\mc T}\}_{i=1}^m$ from $P_{\mc T}(X)$, assign domain labels $Z_i^{\mc S}=-1$ and $Z_i^{\mc T}=1$. Construct the mixed sample $D_\mu:=\{(X_i^\mu,Z_i^\mu)\}_{i=1}^m$ by
\begin{equation}\label{ratio_data}
X^\mu_i = \Delta_i X^\mc{{S}}_i + (1-\Delta_i) X^\mc{{T}}_i,~  Z^\mu_i = \Delta_i Z^\mc{{S}}_i + (1-\Delta_i) Z^\mc{{T}}_i,
\end{equation}
The variables $\Delta_i$ are i.i.d. Bernoulli$(1/2)$, independent of the source and target samples. Let $P_{\mu,Z}$ denote the resulting joint law of $(X^\mu,Z^\mu)$, whose feature marginal is $P_\mu$. 
The population-optimal domain score is $\phi^*=\log r^*$ under the balanced domain mixture; equivalently, $r^*$ is recovered from the global population minimizer of the domain logistic risk \cite{sugiyama2012density}, which is formulated as:
\begin{equation}\label{eq2.8}
    r^* = \exp(\phi^*),
\end{equation}
with
\begin{equation}\label{eq2.2}
\begin{split}
 \mathcal L_r(\phi)&:=\mathbb E_{(X,Z)\sim P_{\mu,Z}}[\log(1+\exp(-Z\phi(X)))],\\
 \phi^*&\in\arg\min_{\phi\in\mathcal M_1^p}\mathcal L_r(\phi).
\end{split}
\end{equation}

The empirical domain-score fit corresponding to $\phi^*$ 
in \eqref{eq2.2} is defined as 
\begin{equation}\label{eq2.9}
\begin{split}
 \widehat{\mathcal L}_r(\phi)&:=\frac1m\sum_{i=1}^m\log(1+\exp(-Z_i^\mu\phi(X_i^\mu))),\\
 \hat\phi&\in\arg\min_{\phi\in\Phi}\widehat{\mathcal L}_r(\phi).
\end{split}
\end{equation}

Here $\Phi$ and $\ddot\Phi$ are domain-score network classes mapping $\mathbb R^p$ and $\mathbb R^d$ to $\mathbb R$, respectively. For the theoretical variant, with domain-score bound $M\geq0$ and clipping level $K\geq0$, define
\begin{equation*}
\begin{gathered}
 \operatorname{clip}_K(t):=\max\{-K,\min\{K,t\}\},\\
 \Phi_M:=\{\operatorname{clip}_M\circ\phi:\phi\in\Phi\},\quad
 \ddot\Phi_M:=\{\operatorname{clip}_M\circ\phi:\phi\in\ddot\Phi\}.
\end{gathered}
\end{equation*}

Domain-score fitting in \eqref{eq2.9} and its latent analogue uses these clipped classes in the theoretical variant. Thus $e^{-M}\leq\hat r,\hat r_d\leq e^M$ wherever the corresponding weight is evaluated. Predictor classes are similarly constrained to $[-\mc B,\mc B]$.
Consequently,  we can use $\hat{r}:= \exp{(\hat{\phi})}$  as the estimator of 
$r^*$.

To provide deeper insight into the deep density ratio estimation procedure, 
Figure \ref{fig:density ratio} in Appendix \ref{appendix} presents an example that compares the logarithm of the true density ratio with its estimation obtained through the proposed method. The results highlight the effectiveness and accuracy of density ratio estimation using  DNNs.

Now, we estimate the density ratio in the latent space, following the same framework as described above.
For oracle features, use the pushforward laws $P_{\mc S}^{R^*}$ and $P_{\mc T}^{R^*}$. Their density ratio is $r_d^*$ as defined above.

Analogous to \eqref{eq2.2}, with oracle features $U=R^*(X)$ and the induced joint law $P_{\mu,Z}^{R^*}$ of $(U,Z)$, we define
\begin{equation*}
    \phi^*_{d}
    \in \arg\min_{\phi \in \mc{M}^{d}_1}
    \mb{E}_{(U,Z)\sim P_{\mu,Z}^{R^*}}
    \left[\log (1+\exp(-Z\phi(U)))\right].
\end{equation*}
and $r^*_{d}$ 
still can be expressed as 
$ 
r^*_{d}= \exp(\phi^*_{d}).
$
For the learned features, the fitted domain score is defined as 
\begin{equation*}
    \hat{\phi}_{d}\in
        \arg\min_{\phi \in \ddot{\Phi}_M}
        \frac{1}{m} \sum_{i = 1}^m \log (1+\exp(-Z^{\mu}_i\phi(\hat R(X_i^\mu)))),  
    \end{equation*}
where $(X_i^\mu,Z_i^\mu)\in D_\mu$ is constructed by \eqref{ratio_data}.
Consequently, we denote  by $\hat{r}_{d}:=\exp(\hat{\phi}_{d})$
an estimator of $r_{\hat R}:=dP_{\mc T}^{\hat R}/dP_{\mc S}^{\hat R}$, the density ratio induced by the learned representation. This need not equal $r_d^*$; agreement of both pushforward distributions with those induced by $R^*$ is one sufficient condition for equality.

\subsection{\texorpdfstring{Fit the Weighted Predictor}{Fit the Weighted Predictor}}\label{sec:tlr}
With $\hat R$ and $\hat r_d$ fixed, fit the predictor by weighted empirical risk minimization:       
\begin{equation*}
\hat{h}\in\arg\min_{h\in \mc{H}} \frac{1}{n}\sum_{i=1}^n\hat{r}_{d}(\hat R(X_i^{\mathcal S})) \cdot\ell(h(\hat R(X_i^{\mathcal S})), Y_i^{\mathcal{S}}), 
\end{equation*}
using the source prediction fold, independently of the representation and domain-score fits.
For a target input $x$, the estimated class probability is $\sigma(\hat h \circ \hat R(x))$, and the predicted binary label is $\mathbf{1}\{\hat h \circ \hat R(x)>0\}$. Algorithm \ref{alg:al2} summarizes the three stages.
\begin{algorithm}[H]
    \caption{Reweighting in the latent space}\label{alg:al2}
    \begin{algorithmic}[1]
        \REQUIRE Independent source labeled folds $D_{\rm rep}$ of size $n_R\geq4$ and $D_{\rm pred}=\{(X_i^{\mc S},Y_i^{\mc S})\}_{i=1}^n$; independent source and target unlabeled samples for constructing $D_\mu$ of size $m$.
        \STATE Estimate the representation function $\hat R$:
        \begin{equation*}
          \hat R\in\arg\min_{R\in\mc F}\widehat{\mc L}_{n_R}(R)
          \quad\text{using }D_{\rm rep}\text{ in \eqref{eq:rep-training}}.
        \end{equation*}
        
        \STATE Estimate the density ratio $\hat r_d$:
        \begin{align*}
          \hat\phi_d&\in\arg\min_{\phi\in\ddot\Phi_M}\frac1m\sum_{i=1}^m
                 \log\big(1+\exp(-Z_i^\mu\phi(\hat R(X_i^\mu)))\big),\quad (X_i^\mu,Z_i^\mu)\in D_\mu,\\
          \hat r_d&:=\exp(\hat\phi_d).
        \end{align*}
        
        \STATE Estimate the target score $\hat h$:
        \begin{equation*}
          \hat h\in\arg\min_{h\in\mc H}\frac1n\sum_{i=1}^n
             \hat r_d(\hat R(X_i^{\mc S}))\ell(h \circ \hat R(X_i^{\mc S}),Y_i^{\mc S})
             \quad\text{using }D_{\rm pred}.
        \end{equation*}
        
        \RETURN $\hat h$, $\hat R$.
    \end{algorithmic}
\end{algorithm}
Algorithm \ref{alg:al2} gives the sample-split procedure analyzed by the finite-sample theorems. In both theoretical algorithms, $D_\mu$ is independent of all data, reference samples and training randomness used to fit $\hat R$; all inputs to the first two stages are independent of $D_{\rm pred}$. These requirements support domain-score analysis conditional on representation training and prediction analysis conditional on both fitted stages.
Stage 1 learns a representation from labeled source data, Stage 2 estimates the density ratio in the learned space, and Stage 3 fits the weighted predictor. The representation and weights remain fixed during Stage 3.

To assess the roles of representation learning and weighting, we compare the proposed method with three control methods. These methods include reweighting in the ambient space (as described in  \eqref{MeB}), unweighted transformation in the latent space (as outlined in  \eqref{MeC}), and reweighting in the original space (as specified in  \eqref{MeD}).
We present all the comparison methods below. In these formulas, $(X_i,Y_i)=(X_i^{\mc S},Y_i^{\mc S})$ are the source observations in $D_{\rm pred}$; the source superscript is omitted.
\begin{itemize}
    \item \textbf{Reweighting in the latent space (Proposed method)}
        \begin{equation}\label{MeA}
         \hat h\in\arg\min_{h\in\mc H}\frac1n\sum_{i=1}^n\hat r_d(\hat R(X_i)) \cdot \ell(h \circ \hat R(X_i),Y_i).   
        \end{equation}

    \item \textbf{Reweighting in the  original space}
    \begin{equation}\label{MeB}
             \hat{h}\in\arg\min_{h\in \mc{H}} \frac{1}{n}\sum_{i =1}^{n} \hat{r}(X_i) \cdot\ell(h \circ \hat R(X_i), Y_i),  
            \end{equation}

    \item \textbf{Unweighted in the latent space}
        \begin{equation}\label{MeC}
             \hat{h}\in\arg\min_{h\in \mc{H}} \frac{1}{n}\sum_{i =1}^{n}\ell(h \circ \hat R(X_i), Y_i),   
            \end{equation}
            
    \item \textbf{Reweighting without representation} \citep{yu2012analysis}
        \begin{equation}\label{MeD}
             \hat{h}\in\arg\min_{h\in \ddot{\mc{H}}} \frac{1}{n}\sum_{i =1}^{n} \hat{r}(X_i) \cdot\ell(h(X_i), Y_i).
            \end{equation}

\end{itemize}
Method \eqref{MeB} fits a latent-space predictor using original-space weights, isolating the effect of the space used for ratio estimation. Method \eqref{MeC} fits an unweighted predictor in the learned space, providing a control for weighting. Method \eqref{MeD} estimates both the ratio and the predictor in the original space, providing a control for representation learning. In \eqref{MeD}, $\hat h$ denotes the fitted predictor in $\ddot{\mc H}$, the class of DNNs mapping $\mathbb R^p$ to $\mathbb R$. Section \ref{comparison} compares these methods theoretically.

\section{Convergence Analysis}\label{convergence rate}
We analyze the proposed latent-space reweighting method in Algorithm \ref{alg:al2}. For a score $f:\mathcal X\to\mathbb R$, define its true target risk by
\begin{equation*}
 \mathcal R_{\mathcal T}(f):=\mathbb E_{(X,Y)\sim P_{\mathcal T}}\ell(f(X),Y).
\end{equation*}
For the fitted predictor, define the excess risk by
\begin{equation*}
 \mathcal E_{\mathcal T}(\hat h,\hat R)
 :=\mathcal R_{\mathcal T}(\hat h\circ\hat R)
                  -\mathcal R_{\mathcal T}(h^*\circ R^*).
\end{equation*}
The reference score $h^*\circ R^*$ is Bayes optimal under Assumptions \ref{ass2.0}--\ref{ass2.1}. Population risks integrate fresh observations with fitted functions held fixed; outer expectations below average over the training data and training randomness. For a measurable latent weight $r_d:\mathbb R^d\to[0,\infty)$ and a representation $R$, define the latent-weighted source risk by
\begin{equation}\label{excess error}
 \mathcal R_{r_d,R}(h):=\mathbb E_{P_{\mathcal S}}
 [r_d(R(X))\ell(h \circ R(X),Y)].
\end{equation}
The main theorem displays the sample-size rates directly. The more general finite-sample inequality, which does not require approximation or representation-recovery rates, is retained in Appendix \ref{pfth6}.

\subsection{Conditions for the Proposed Method}\label{sec:explicit-rates}
For a fixed representation $R$, use the pushforward law $P_k^R$ defined in Section \ref{problem setting}, put $U=R(X)$ within these pushforward expressions, and define
\begin{equation*}
 r_R:=\frac{dP_{\mathcal T}^R}{dP_{\mathcal S}^R},\qquad
 \eta_k^R(u):=P_k(Y=1\mid R(X)=u),\quad k\in\{\mathcal S,\mathcal T\}.
\end{equation*}
In particular, $r_{R^*}=r_d^*$, whereas the ratio estimated using learned features is $r_{\hat R}$. These ratios need not coincide. The balanced pushforward feature law is $P_\mu^R=(P_{\mathcal S}^R+P_{\mathcal T}^R)/2$, and $P_{\mu,Z}^R$ is the joint law of $(R(X^\mu),Z^\mu)$. The conditional label discrepancy is
\begin{equation*}
 b_R:=\mathbb E_{P_{\mathcal T}^R}
 |\eta_{\mathcal S}^R(U)-\eta_{\mathcal T}^R(U)|.
\end{equation*}
The general bound in Lemma \ref{lem:latent-oracle} permits $b_{\hat R}>0$. The main theorem treats the strictly sufficient learned-representation case below.

\begin{assumption}\label{assum1}
Suppose $P_{\mathcal T}(X)\ll P_{\mathcal S}(X)$ and $0<c_r\leq r^*(x)\leq C_r<\infty$ for $P_{\mathcal S}(X)$-almost every $x$. The same bounds hold for $r_R$ by \eqref{eq:target-risk-bridge}.
\end{assumption}

\begin{assumption}\label{assum4}
For $k\in\{\mathcal S,\mathcal T\}$, almost surely,
\begin{equation*}
 \sup_{h\in\mathcal H}\|h\|_{\infty,\operatorname{supp}(P_k^{\hat R})}\leq\mathcal B,
 \qquad \|h^*\|_{\infty,\operatorname{supp}(P_k^{R^*})}\leq\mathcal B.
\end{equation*}
\par
\end{assumption}
For $a,a'\in[-\mathcal B,\mathcal B]$ and $y\in[0,1]$, logistic loss satisfies
\begin{equation*}
 |\ell(a,y)-\ell(a',y)|\leq|a-a'|,
 \qquad 0\leq\ell(a,y)\leq L_{\mathcal B}:=\log(1+e^{\mathcal B}).
\end{equation*}
\par

\begin{proposition}\label{prop:learned-sufficiency}
Under Assumptions \ref{ass2.0}--\ref{ass2.1} and \ref{assum1}, suppose that, for almost every realization of the representation-training data, reference sample and training randomness, the fitted map is sufficient in the source population:
\begin{equation*}
 Y \Vbar X\mid\hat R(X)\quad\text{under }P_{\mathcal S}.
\end{equation*}
Then it is also sufficient in the target population and
\begin{equation}\label{eq:learned-sufficiency}
 P_{\mathcal S}(Y\mid\hat R(X))=P_{\mathcal T}(Y\mid\hat R(X)),\qquad b_{\hat R}=0.
\end{equation}
Consequently, the term $2\mathcal B\mathbb E b_{\hat R}$ disappears from the general bound \eqref{eq:latent-oracle-bound}. Predictor-class approximation and density-ratio estimation errors remain. The proof is in Appendix \ref{app:sufficiency}.
\end{proposition}

\begin{assumption}\label{ass:learned-sufficient}
For almost every realization of the representation-training data, reference sample and training randomness, the actual fitted map $\hat R$ satisfies the following conditions, with fixed cubes and constants uniform in that realization and in $n_R$.
\begin{enumerate}[(i)]
\item $Y \Vbar X\mid\hat R(X)$ under $P_{\mathcal S}$, and $\hat R(X)$ is supported on a fixed compact cube $K_U\subset\mathbb R^d$ in both domains.
\item The source conditional Bayes score in the learned coordinates,
\begin{equation}\label{eq:learned-bayes-score}
 h_{\hat R}^*(u):=\log\frac{\eta_{\mathcal S}^{\hat R}(u)}{1-\eta_{\mathcal S}^{\hat R}(u)},
\end{equation}
has a $\beta_{\mathcal H}$-H\"older extension to $K_U$ with a fixed radius and absolute bound $\mathcal B$.
\item $\log r_{\hat R}$ has a $\beta_{\ddot\Phi}$-H\"older extension to $K_U$ with a fixed radius and absolute bound $M$, where $M\geq\max\{|\log c_r|,|\log C_r|\}$.
\end{enumerate}
The H\"older classes on $K_U$ are those of Definition \ref{def:hc} after an affine map from $[0,1]^d$ to $K_U$. For an integer index $k\geq1$, derivatives of order $k-1$ are Lipschitz. The subscript in $h_{\hat R}^*$ distinguishes the learned coordinates from the oracle score $h^*$; their compositions with $\hat R$ and $R^*$ agree almost surely by sufficiency.
\end{assumption}

\begin{definition}\label{ass:downstream-sieves}
For a network class $J$, use the depth $\mathcal D_J$, width $\mathcal W_J$ and total parameter budget $\mathcal S_J$ of Definition \ref{def:dnn}. Depth counts hidden layers, so there are $\mathcal D_J+1$ affine layers. For fixed input dimension $q$, H\"older index $\beta>0$ and sample size $t\geq2$, take the fully connected ReLU class with depth and width budgets
\begin{equation}\label{eq:sieve-architecture}
\begin{split}
 \mathcal D_J&=\left\lceil C_Dt^{\frac{q}{4(2\beta+q)}}\log(et)\right\rceil,\qquad
 \mathcal W_J=\left\lceil C_Wt^{\frac{q}{4(2\beta+q)}}\log(et)\right\rceil,\\
 \mathcal S_J&\leq C(\mathcal D_J+1)\mathcal W_J^2
 \leq Ct^{\frac{3q}{4(2\beta+q)}}\log^3(et).
\end{split}
\end{equation}
The fixed constants are sufficiently large for the target H\"older ball and cube; output clipping is included in these budgets. All weight and bias positions are counted in $\mathcal S_J$. Use $(q,\beta,t)=(d,\beta_{\mathcal H},n)$ for $\mathcal H\subset[-\mathcal B,\mathcal B]$ and $(d,\beta_{\ddot\Phi},m)$ for $\ddot\Phi_M\subset[-M,M]$. Lemma \ref{lem:uniform-network-approximation} proves the uniform approximation error $O(t^{-\beta/(2\beta+q)})$ for this construction.
\end{definition}

\subsection{Explicit Target-Risk Rate}
\begin{theorem}\label{Th5}
Suppose Assumptions \ref{ass2.0}--\ref{ass2.1}, \ref{assum1} and \ref{ass:learned-sufficient} hold. Use the independent-fold Algorithm \ref{alg:al2}, with representation-fold size $n_R\geq4$, domain-fold size $m\geq2$ and prediction-fold size $n\geq2$. Choose the clipped fully connected ReLU predictor class $\mathcal H$ and domain-score class $\ddot\Phi_M$ of Definition \ref{ass:downstream-sieves}, with respective depths and widths
\begin{equation*}
\begin{aligned}
 \mathcal D_{\mathcal H},\mathcal W_{\mathcal H}
 &\asymp n^{\frac{d}{8\beta_{\mathcal H}+4d}}\log(en),\\
 \mathcal D_{\ddot\Phi_M},\mathcal W_{\ddot\Phi_M}
 &\asymp m^{\frac{d}{8\beta_{\ddot\Phi}+4d}}\log(em).
\end{aligned}
\end{equation*}
Then
\begin{equation}\label{eq:strict-sufficient-rate}
\begin{split}
 \mathbb E\left[\mathcal R_{\mathcal T}(\hat h\circ\hat R)
                 -\mathcal R_{\mathcal T}(h^*\circ R^*)\right]
 \leq \widetilde O \left(
 n^{-\frac{\beta_{\mathcal H}}{2\beta_{\mathcal H}+d}}
 +m^{-\frac{\beta_{\ddot\Phi}}{2\beta_{\ddot\Phi}+d}}
 \right).
\end{split}
\end{equation}
The implied constants depend only on the fixed dimensions, cubes, regularity radii, score bounds and overlap constants, and are uniform in $n,m,n_R$ and the training realization. The proof is in Appendix \ref{app:sufficiency}.
\end{theorem}

The first term controls prediction and the second estimates the density ratio in the learned $d$-dimensional space. Strict sufficiency removes the representation-information loss and conditional label mismatch; the remaining predictor approximation error is included in the first term. Representation training uses an additional labeled fold, giving a total labeled source budget of $n+n_R$. Lemma \ref{lem:latent-oracle} gives the general finite-sample bound.

\section{Comparisons}\label{comparison}
We explain the rate differences in Table \ref{tab_com} by separating the effects of representation learning on prediction and density-ratio estimation. We then compare fitted target risks and give a model in which the gains hold with a learned representation and estimated weights.

The rate comparisons use independent folds with matched prediction- and domain-fold sizes $n,m$. The representation-based methods share a strictly sufficient fitted map learned from an additional labeled fold of size $n_R$, giving total labeled source size $n+n_R$. The uniform regularity conditions on Bayes scores and log density ratios are given in Assumptions \ref{ass:learned-sufficient} and \ref{ass:comparison-rates}. Dimensions and regularity constants are fixed.

\begin{assumption}\label{ass:comparison-rates}
\begin{enumerate}[(i)]
\item Both feature laws are supported on a fixed compact cube $K_X\subset\mathbb R^p$. The raw domain score $\phi^*=\log r^*$ has a $\beta_\Phi$-H\"older extension to $K_X$ with a fixed radius and absolute bound $M$. Use the network rule of Definition \ref{ass:downstream-sieves} with $(q,\beta,t)=(p,\beta_\Phi,m)$ for $\Phi_M\subset[-M,M]$.
\item For prediction without representation, the ambient Bayes score $h^*\circ R^*$ has a $\beta_{\ddot{\mathcal H}}$-H\"older extension to $K_X$ with a fixed radius and absolute bound $\mathcal B$. Every $g\in\ddot{\mathcal H}$ is bounded by $\mathcal B$, and use the same network rule with $(q,\beta,t)=(p,\beta_{\ddot{\mathcal H}},n)$.
\end{enumerate}
\end{assumption}
The index $\beta_{\ddot{\mathcal H}}$ describes the regularity of the ambient Bayes score in Assumption \ref{ass:comparison-rates}(ii). Under the oracle smoothness conditions in Assumption \ref{assum5}, cube coverage $R^*(K_X)\subset K_U$, and a bounded extension of $h^*$ on $K_U$, one admissible choice is
\begin{equation}\label{eq:composition-index}
 \beta_{\ddot{\mathcal H}}=\min\{\beta_{\mathcal H},\beta_{\mathcal F},
                                      \beta_{\mathcal H}\beta_{\mathcal F}\}.
\end{equation}
Lemma \ref{prop:composition-smoothness} proves this sufficient condition. The comparison propositions also apply when the ambient score's regularity is established directly. The raw-score conditions do not imply the learned-ratio regularity required by the proposed method.

\subsection{Original-Space Weights with a Learned Representation}
Method \eqref{MeB} retains the representation but estimates weights in the original space. For a weight $r$ on $\mathcal X$ and an ambient score $f$, write
\begin{equation*}
 \mathcal R_r(f):=\mathbb E_{P_{\mathcal S}}[r(X)\ell(f(X),Y)].
\end{equation*}
Thus its population weighted risk is $\mathcal R_r(h\circ R)$, distinct from the latent-weighted risk $\mathcal R_{r_d,R}(h)$ in \eqref{excess error}. Original-space covariate shift gives $\mathcal R_{r^*}(h\circ R)=\mathcal R_{\mathcal T}(h\circ R)$ for every fixed $R$.

\begin{proposition}\label{th4}
Use the same fitted representation and predictor architecture as in Theorem \ref{Th5}, and the independent-fold Algorithm \ref{alg:al1}. Suppose Assumptions \ref{ass2.0}--\ref{ass2.1}, \ref{assum1}, parts (i)--(ii) of Assumption \ref{ass:learned-sufficient}, and part (i) of Assumption \ref{ass:comparison-rates} hold. Choose the clipped fully connected ReLU raw-score class $\Phi_M$ of Definition \ref{ass:downstream-sieves} with depth and width
\begin{equation*}
 \mathcal D_{\Phi_M},\mathcal W_{\Phi_M}
 \asymp m^{\frac{p}{8\beta_\Phi+4p}}\log(em).
\end{equation*}
For $n_R\geq4$ and $m,n\geq2$,
\begin{equation}\label{eq:strict-comparison-representation-rate}
\mathbb E \left[\mathcal R_{\mathcal T}(\hat h\circ\hat R)
                 -\mathcal R_{\mathcal T}(h^*\circ R^*)\right]
\leq\widetilde O \left(
 n^{-\frac{\beta_{\mathcal H}}{2\beta_{\mathcal H}+d}}
 +m^{-\frac{\beta_\Phi}{2\beta_\Phi+p}}
 \right).
\end{equation}
This result uses the same learned representation and prediction class as the proposed method, while estimating the ratio in $p$ dimensions. Appendix \ref{app:comparison-proofs} gives the proof.
\end{proposition}

Comparing \eqref{eq:strict-sufficient-rate} and \eqref{eq:strict-comparison-representation-rate} isolates the effect of where the weights are estimated: the fitted representation, prediction class and labeled folds are shared. Their prediction terms coincide. Original-space weighting retains an ambient-dimensional ratio-estimation term, so learning a low-dimensional predictor alone does not remove that statistical cost.

\begin{algorithm}[H]
    \caption{Reweighting in original space}\label{alg:al1}
    \begin{algorithmic}[1]
        \REQUIRE Independent source labeled folds $D_{\rm rep}$ and $D_{\rm pred}$ of sizes $n_R\geq4,n$, and independent source and target unlabeled samples for $D_\mu$ of size $m$.
        \STATE Estimate the representation function $\hat R$:
        \begin{equation*}
          \hat R\in\arg\min_{R\in\mc F}\widehat{\mc L}_{n_R}(R)
          \quad\text{using }D_{\rm rep}\text{ in \eqref{eq:rep-training}}.
        \end{equation*}
        
        \STATE Estimate the original-space density ratio $\hat r$:
        \begin{align*}
          \hat\phi&\in\arg\min_{\phi\in\Phi_M}\frac1m\sum_{i=1}^m
             \log\big(1+\exp(-Z_i^\mu\phi(X_i^\mu))\big),\quad (X_i^\mu,Z_i^\mu)\in D_\mu,\\
          \hat r&:=\exp(\hat\phi).
        \end{align*}
        
        \STATE Estimate the target score $\hat h$:
        \begin{equation*}
          \hat h\in\arg\min_{h\in\mc H}\frac1n\sum_{i=1}^n
             \hat r(X_i^{\mc S})\ell(h(\hat R(X_i^{\mc S})),Y_i^{\mc S})
             \quad\text{using }D_{\rm pred}.
        \end{equation*}
        
        \RETURN $\hat h$, $\hat R$.
    \end{algorithmic}
\end{algorithm}

\subsection{Original-Space Weights without Representation}
Method \eqref{MeD} fits the prediction score directly in the ambient space. The fitted score belongs to $\ddot{\mathcal H}$, whose input is $X$ rather than $\hat R(X)$.
\begin{proposition}\label{contl gen}
Suppose Assumptions \ref{ass2.0}--\ref{ass2.1}, \ref{assum1} and \ref{ass:comparison-rates} hold. Train $\hat r=\exp(\hat\phi)$ by domain-logistic ERM over the raw-score sieve of Proposition \ref{th4} on an independent balanced domain fold of size $m\geq2$, where $M\geq\max\{|\log c_r|,|\log C_r|\}$. Fit $\hat h\in\ddot{\mathcal H}$ by weighted ERM on an independent source fold of size $n\geq2$, using the clipped fully connected ReLU class of Definition \ref{ass:downstream-sieves} with depth and width
\begin{equation*}
 \mathcal D_{\ddot{\mathcal H}},\mathcal W_{\ddot{\mathcal H}}
 \asymp n^{\frac{p}{8\beta_{\ddot{\mathcal H}}+4p}}\log(en).
\end{equation*}
Then
\begin{equation}\label{eq:comparison-ambient-rate}
 \mathbb E\!\left[\mathcal R_{\mathcal T}(\hat h)
                 -\mathcal R_{\mathcal T}(h^*\circ R^*)\right]
 \leq\widetilde O\!\left(
 n^{-\frac{\beta_{\ddot{\mathcal H}}}{2\beta_{\ddot{\mathcal H}}+p}}
 +m^{-\frac{\beta_\Phi}{2\beta_\Phi+p}}\right).
\end{equation}
Both prediction and ratio estimation operate in the ambient space. The proof is in Appendix \ref{app:comparison-proofs}.
\end{proposition}

Comparing Methods \eqref{MeB} and \eqref{MeD} in Table \ref{tab_com} holds the original-space ratio term fixed. At the same prediction-fold size $n$, the prediction exponent improves precisely when $p\beta_{\mathcal H}>d\beta_{\ddot{\mathcal H}}$. Comparing Methods \eqref{MeA} and \eqref{MeB} holds the learned-space prediction term fixed and changes the ratio-estimation exponent, whose improvement is characterized next.

\subsection{Dimension and Variance Comparisons}
The gain from latent ratio estimation depends on both dimension and smoothness. Comparing the second terms in \eqref{eq:strict-sufficient-rate} and \eqref{eq:strict-comparison-representation-rate} gives
\begin{equation*}
 \frac{\beta_{\ddot\Phi}}{2\beta_{\ddot\Phi}+d}
 -\frac{\beta_\Phi}{2\beta_\Phi+p}
 =\frac{p\beta_{\ddot\Phi}-d\beta_\Phi}
 {(2\beta_{\ddot\Phi}+d)(2\beta_\Phi+p)}.
\end{equation*}
\par
Thus the ratio-estimation exponent improves precisely when $p\beta_{\ddot\Phi}>d\beta_\Phi$. This criterion accounts for changes in regularity induced by the representation. With $m\asymp n$, this gain increases the overall upper-bound exponent exactly when $\beta_\Phi/(2\beta_\Phi+p)<\beta_{\mathcal H}/(2\beta_{\mathcal H}+d)$, so original-space ratio estimation is strictly slower than prediction. When prediction is the slower or tied term, the overall exponent remains unchanged. The following corollary gives a regime in which the ratio-estimation gain determines the overall comparison.

\begin{corollary}\label{cor:dimension-advantage}
Suppose the conditions of Theorem \ref{Th5} and Propositions \ref{th4} and \ref{contl gen} hold, with $d<p$, $m\asymp n$ and
$\beta_{\ddot\Phi}=\beta_\Phi=\beta_{\ddot{\mathcal H}}=\beta_{\mathcal H}>0$.
The respective expected excess-risk upper bounds are
\begin{equation}\label{eq:dimension-advantage}
\begin{aligned}
 \text{Method \eqref{MeA}:}\quad&
 \widetilde O\!\left(n^{-\frac{\beta_{\mathcal H}}{2\beta_{\mathcal H}+d}}\right),\\
 \text{Methods \eqref{MeB} and \eqref{MeD}:}\quad&
 \widetilde O\!\left(n^{-\frac{\beta_{\mathcal H}}{2\beta_{\mathcal H}+p}}\right).
\end{aligned}
\end{equation}
Thus the proposed method has the larger convergence exponent in these excess-risk upper bounds for the specified sieves and smoothness conditions. Representation training uses the same additional labeled fold for Methods \eqref{MeA} and \eqref{MeB}; Method \eqref{MeD} has no such stage.
\end{corollary}

Under these conditions, the ambient-dimensional ratio term determines the bound for Method \eqref{MeB}, despite its low-dimensional prediction stage. Moving ratio estimation to the learned space makes both terms depend on $d$, which explains the overall exponent gain in \eqref{eq:dimension-advantage}. The source budget is $n+n_R$ for both representation-based methods; Proposition \ref{prop:learnable-coordinate} below explicitly accounts for learning the representation in a model where a logarithmic representation-fold size suffices.

The following conditional-variance identity explains the risk comparison in Theorem \ref{thm:fitted-risk-comparison}. Related latent-weight variance identities arise in off-policy evaluation \citep{kiyohara2024off}.
\begin{proposition}\label{prop:weight-variance}
Under Assumptions \ref{ass2.0}--\ref{ass2.1} and \ref{assum1}, fix a realization of a source-sufficient fitted representation $\hat R$ and any score $h$ with $|h|\leq\mathcal B$. For a fresh source observation let $U=\hat R(X)$. Then
\begin{equation}\label{eq:weight-variance}
\begin{split}
 &\operatorname{Var}_{\mathcal S}\!\left[r^*(X)\ell(h(U),Y)\right]
 -\operatorname{Var}_{\mathcal S}\!\left[r_{\hat R}(U)\ell(h(U),Y)\right]\\
 &\quad=\mathbb E_{P_{\mathcal S}}\!\left[
 \ell(h(U),Y)^2\operatorname{Var}_{\mathcal S}(r^*(X)\mid U)
 \right]\geq0.
\end{split}
\end{equation}
For logistic loss and bounded scores, the inequality is strict whenever
$P_{\mathcal S}\{\operatorname{Var}_{\mathcal S}(r^*(X)\mid U)>0\}>0$.
The statement concerns true weights and a fixed predictor; estimated weights are covered by the preceding risk bounds. The proof is in Appendix \ref{app:comparison-proofs}.
\end{proposition}

The variance difference in \eqref{eq:weight-variance} has a direct interpretation. For $U=\hat R(X)$, the latent ratio is the conditional average of $r^*(X)$ given $U$. Sufficiency makes the response independent of the discarded covariates given $U$, so this averaging preserves target risk and removes residual weight variation. The gain persists even for known ratios and is distinct from the ratio-estimation improvement above. A smooth sufficient family is given in \eqref{eq:sufficient-family} in Appendix \ref{app:sufficiency}.

\subsection{Prediction-Risk Comparison and a Learnable Representation}

The same conditional-averaging argument applies to weighted loss gradients. The following theorem translates the resulting covariance reduction into an expected target-risk expansion for a fixed sufficient map and a fixed finite-dimensional predictor class, allowing model misspecification.

\begin{theorem}\label{thm:fitted-risk-comparison}
Assume original-space covariate shift and Assumption \ref{assum1}. Fix a source-sufficient map $R$, put $U=R(X)$, and take
\begin{equation*}
 \mathcal H=\{h_\theta(u)=\theta^\top z(u):\theta\in\Theta\},
 \qquad \Theta\subset\mathbb R^k,
\end{equation*}
where $k$ is fixed, $\Theta$ is compact and convex with nonempty interior, $z:\mathbb R^d\to\mathbb R^k$ is fixed, bounded and measurable, and
$\sup_{\theta,u}|h_\theta(u)|\leq\mathcal B$.
Suppose $\mathbb E_{P_{\mathcal T}}[z(U)z(U)^\top]\succ0$ and
\begin{equation*}
 \theta_{\mathcal T}:=\arg\min_{\theta\in\Theta}
 \mathcal R_{\mathcal T}(h_\theta\circ R)\in\operatorname{int}\Theta.
\end{equation*}
Using the same $n$ i.i.d. source prediction observations, let $\hat\theta_A,\hat\theta_B$ minimize, respectively,
\begin{equation*}
 \frac1n\sum_{i=1}^n\hat r_d(U_i)\ell(h_\theta(U_i),Y_i),
 \qquad
 \frac1n\sum_{i=1}^n\hat r(X_i)\ell(h_\theta(U_i),Y_i).
\end{equation*}
The weights are independent of these observations, satisfy $c\leq\hat r_d,\hat r\leq C$ for fixed $0<c<C<\infty$, and, as $n\to\infty$ and $m=m(n)\to\infty$, obey
\begin{equation}\label{eq:weight-fourth-moment}
 \mathbb E\|\hat r_d-r_R\|_{L^2(P_{\mathcal S}^R)}^4
 +\mathbb E\|\hat r-r^*\|_{L^2(P_{\mathcal S}(X))}^4
 =o(n^{-2}).
\end{equation}
Define the curvature and score second-moment matrices
\begin{equation*}
\begin{aligned}
 \mathbf H_{\mathcal T}
 &:=\mathbb E_{P_{\mathcal T}}
       [\sigma'(h_{\theta_{\mathcal T}}(U))z(U)z(U)^\top],\\
 \mathbf V_A
 &:=\mathbb E_{P_{\mathcal S}}
       [r_R(U)^2(\sigma(h_{\theta_{\mathcal T}}(U))-Y)^2z(U)z(U)^\top],\\
 \mathbf V_B
 &:=\mathbb E_{P_{\mathcal S}}
       [r^*(X)^2(\sigma(h_{\theta_{\mathcal T}}(U))-Y)^2z(U)z(U)^\top].
\end{aligned}
\end{equation*}
For $j\in\{A,B\}$,
\begin{equation}\label{eq:fitted-risk-expansion}
 \mathbb E\!\left[\mathcal R_{\mathcal T}(h_{\hat\theta_j}\circ R)
       -\mathcal R_{\mathcal T}(h_{\theta_{\mathcal T}}\circ R)\right]
 =\frac{\operatorname{tr}(\mathbf H_{\mathcal T}^{-1}\mathbf V_j)}{2n}
       +o(n^{-1}),
\end{equation}
and
\begin{equation}\label{eq:fitted-covariance-order}
 \mathbf V_B-\mathbf V_A
 =\mathbb E_{P_{\mathcal S}}\!\left[
   \operatorname{Var}_{\mathcal S}(r^*(X)\mid U)
   (\sigma(h_{\theta_{\mathcal T}}(U))-Y)^2z(U)z(U)^\top\right]\succeq0.
\end{equation}
The leading target-risk constant is strictly smaller for latent weighting if
\begin{equation}\label{eq:fitted-strict-condition}
 P_{\mathcal S}\{\operatorname{Var}_{\mathcal S}(r^*(X)\mid U)>0,\
                          \|z(U)\|>0\}>0.
\end{equation}
\end{theorem}
With the true density ratios, both population objectives coincide with target risk and share its minimizer and curvature matrix $\mathbf H_{\mathcal T}$. Equation \eqref{eq:fitted-covariance-order} identifies the additional covariance of the loss gradient carried by original-space weights. Subtracting the two expansions in \eqref{eq:fitted-risk-expansion} yields
\begin{equation*}
 \mathbb E\!\left[\mathcal R_{\mathcal T}(h_{\hat\theta_B}\circ R)
                  -\mathcal R_{\mathcal T}(h_{\hat\theta_A}\circ R)\right]
 =\frac{\operatorname{tr}\!\left[
       \mathbf H_{\mathcal T}^{-1}(\mathbf V_B-\mathbf V_A)\right]}{2n}
       +o(n^{-1}).
\end{equation*}
\par
Since $\mathbf H_{\mathcal T}\succ0$ and $\mathbf V_B-\mathbf V_A\succeq0$, the leading coefficient is nonnegative and is positive under \eqref{eq:fitted-strict-condition}. Latent weighting then has strictly smaller expected target risk for all sufficiently large $n$. This is an ordering of the fitted predictors' risks within the fixed class, complementing the neural upper-bound comparison. Condition \eqref{eq:weight-fourth-moment} makes weight-estimation error negligible at this scale; true weights satisfy it with zero error, and Proposition \ref{prop:coordinate-comparison} verifies it for estimated weights. Proofs are in Appendix \ref{app:fitted-risk-proofs}.

The next proposition makes the representation itself learnable. It identifies the relevant coordinate from labeled source observations and carries the selection error into the downstream neural risk bound.

\begin{proposition}\label{prop:learnable-coordinate}
Let $p\geq2$ be fixed, $j^*\in\{1,\ldots,p\}$ unknown, $0<a<1$, and $|\tau_j|\leq\rho$ with $0\leq\rho<1$. Consider
\begin{equation}\label{eq:coordinate-model}
\begin{gathered}
 P_{\mathcal S}(X)=\operatorname{Unif}([0,1]^p),\qquad
 \frac{dP_{\mathcal T}(X)}{dP_{\mathcal S}(X)}(x)
       =\prod_{j=1}^p[1+\tau_j(2x_j-1)],\\
 P_{\mathcal S}(Y=1\mid X=x)=P_{\mathcal T}(Y=1\mid X=x)
       =\tfrac12+a(x_{j^*}-\tfrac12),\qquad R^*(x)=x_{j^*}.
\end{gathered}
\end{equation}
Restrict Stage 1 in \eqref{eq:rep-training} to the ReLU maps
$\mathcal F=\{x\mapsto x_j:1\leq j\leq p\}$ and write $\hat R(x)=x_{\hat j}$.
On $[0,1]^p$, each candidate has a fully connected ReLU realization with
$\mathcal D_{\mathcal F}=1$, $\mathcal W_{\mathcal F}=p$, and
$\mathcal S_{\mathcal F}=p+3$, counting all parameter positions as in Definition \ref{def:dnn}.
Use the same nonempty discriminator class for all candidates, the same reference sample, fixed $\lambda\geq0$, and
\begin{equation*}
 |D(u)|\leq A_D,\qquad |D(u)-D(v)|\leq K_D|u-v|,
 \qquad f(t)=t\log t .
\end{equation*}
For $n_R\geq4$, there are $C,c>0$, depending only on $a,\lambda,K_D$, such that
\begin{equation}\label{eq:coordinate-selection}
 P\{\hat j\ne j^*\}\leq Cp e^{-cn_R}.
\end{equation}
For the population representation objective $\mathcal L$ in \eqref{eq2.7}, every $R\in\mathcal F$ satisfies
\begin{equation}\label{eq:coordinate-recovery}
 \|R-R^*\|_{L^1(P_{\mathcal S})}
 =\frac{10}{a^2}\{\mathcal L(R)-\mathcal L(R^*)\}.
\end{equation}
Use independent domain and prediction folds and the clipped fully connected classes of Theorem \ref{Th5}, with
\begin{equation*}
 \mathcal D_{\mathcal H},\mathcal W_{\mathcal H}
       \asymp n^{1/(8\beta_{\mathcal H}+4)}\log(en),\qquad
 \mathcal D_{\ddot\Phi_M},\mathcal W_{\ddot\Phi_M}
       \asymp m^{1/(8\beta_{\ddot\Phi}+4)}\log(em).
\end{equation*}
For any fixed $\beta_{\mathcal H},\beta_{\ddot\Phi}>0$, sufficiently large fixed regularity radii and score bounds give
\begin{equation}\label{eq:coordinate-neural-rate}
 \mathbb E\mathcal E_{\mathcal T}(\hat h,\hat R)
 \leq\widetilde O\!\left(
 n^{-\frac{\beta_{\mathcal H}}{2\beta_{\mathcal H}+1}}
 +m^{-\frac{\beta_{\ddot\Phi}}{2\beta_{\ddot\Phi}+1}}\right)
 +L_{\mathcal B}Cp e^{-cn_R}.
\end{equation}
\end{proposition}
Equation \eqref{eq:coordinate-neural-rate} accounts for all three learned stages: the one-dimensional prediction and ratio terms, and the probability of selecting an insufficient coordinate. At fixed $p$, choosing $n_R=\lceil(2/c)\log n\rceil\vee4$ makes the selection contribution at most $L_{\mathcal B}Cp n^{-2}$, with total labeled source size $n+O(\log n)$. All candidates have the same source marginal, so their population variational biases cancel in objective differences even when $\varepsilon_{\mathcal D,n_R}$ does not vanish. This model thus obtains an explicit representation-learning contribution without requiring a vanishing discriminator bias. The general bound for insufficient representations and recovery analysis are in Appendices \ref{pfth6} and \ref{app:representation-recovery}.

\begin{proposition}\label{prop:coordinate-comparison}
In model \eqref{eq:coordinate-model}, keep the representation learner of Proposition \ref{prop:learnable-coordinate}. Using its known uniform source density and an independent fold of $m$ target covariates, define
\begin{equation}\label{eq:coordinate-ratio-estimator}
\begin{gathered}
 \hat\tau_j=\Pi_{[-\rho,\rho]}
     \left(\frac3m\sum_{i=1}^m(2X_{ij}^{\mathcal T}-1)\right),\\
 \hat r(x)=\prod_{j=1}^p[1+\hat\tau_j(2x_j-1)],
 \qquad \hat r_d(u)=1+\hat\tau_{\hat j}(2u-1),
\end{gathered}
\end{equation}
where $\Pi_{[-\rho,\rho]}$ is interval projection. These are parametric ratio estimators for this comparison. Use the constant-score class
$\mathcal H=\{h(u)=t:|t|\leq\mathcal B\}$, with
$\mathcal B>\log((1+a)/(1-a))$.
Let $\hat h_A,\hat h_B,\hat h_C$ be the empirical minimizers with latent weights, raw weights, and unit weights on the same prediction fold.
The target-optimal constant score is
\begin{equation*}
 q_{\mathcal T}:=\mathbb E_{P_{\mathcal T}}Y
     =\tfrac12+\frac{a\tau_{j^*}}6,\qquad
 h_{\mathcal H,\mathcal T}^*(u)
     :=\log\frac{q_{\mathcal T}}{1-q_{\mathcal T}}.
\end{equation*}
If $m/n\to\infty$ and $np e^{-cn_R}\to0$, with $c$ from \eqref{eq:coordinate-selection}, then
\begin{equation}\label{eq:coordinate-risk-comparison}
\begin{aligned}
 \mathbb E[\mathcal R_{\mathcal T}(\hat h_A\circ\hat R)
          -\mathcal R_{\mathcal T}(h_{\mathcal H,\mathcal T}^*\circ R^*)]
       &=\frac{C_A}{n}+o(n^{-1}),\\
 \mathbb E[\mathcal R_{\mathcal T}(\hat h_B\circ\hat R)
          -\mathcal R_{\mathcal T}(h_{\mathcal H,\mathcal T}^*\circ R^*)]
       &=\frac{C_A}{n}\prod_{j\ne j^*}(1+\tau_j^2/3)+o(n^{-1}),\\
 \mathbb E[\mathcal R_{\mathcal T}(\hat h_C\circ\hat R)
          -\mathcal R_{\mathcal T}(h_{\mathcal H,\mathcal T}^*\circ R^*)]
       &\longrightarrow
       \operatorname{KL}(\operatorname{Bern}(q_{\mathcal T})
                          \,\|\,\operatorname{Bern}(1/2)),
\end{aligned}
\end{equation}
where
\begin{equation*}
 C_A=
 \frac{\frac14+\frac{(1-a^2)\tau_{j^*}^2}{12}
                    +\frac{a^2\tau_{j^*}^4}{108}}
      {2(\frac14-\frac{a^2\tau_{j^*}^2}{36})}>0.
\end{equation*}
Thus a nonzero shift in an irrelevant coordinate gives a strict leading-risk advantage over raw weighting, and $\tau_{j^*}\ne0$ gives a positive limiting excess risk for unweighted learning.
\end{proposition}
Equation \eqref{eq:coordinate-risk-comparison} separates the effects of relevant and irrelevant shift. Each shifted irrelevant coordinate multiplies the leading risk constant of original-space weighting by $1+\tau_j^2/3$, whereas the latent-weighted constant $C_A$ depends only on the relevant-coordinate shift and the response parameter. When all irrelevant shifts vanish, the two weighted methods have the same leading constant. When the relevant coordinate shifts, $q_{\mathcal T}\ne1/2$, and unweighted fitting incurs the positive limiting excess risk displayed in \eqref{eq:coordinate-risk-comparison}. Latent weighting therefore corrects the shift that affects prediction and excludes the extra risk from weighting irrelevant changes.

These gains hold with a learned representation and estimated weights: $m/n\to\infty$ makes parametric ratio-estimation error negligible on the $n^{-1}$ risk scale, while logarithmically many representation observations suffice for $np e^{-cn_R}\to0$ at fixed $p$. The comparison uses total labeled source size $n+n_R$ and target size $m$. Its benchmark is the best constant predictor, so the unrestricted Bayes approximation error is common to all three methods. Equation \eqref{eq:coordinate-neural-rate} gives the separate guarantee for growing neural predictor classes.

\section{Experiments}\label{experiments}
We conduct two sets of experiments on synthetic and real datasets. The first compares Method \eqref{MeA} with the three controls, Methods \eqref{MeB}, \eqref{MeC} and \eqref{MeD}, to assess the effects of representation learning and weighting. The second compares our method with domain-adaptation baselines implemented in the Transfer Learning Library \citep{tllib}, across initialization and domain-transfer settings. We set the random seed for the runs. All experiments use PyTorch on a Tesla V100 GPU.

\subsection{Data Set and Processing}
We introduce the datasets used in the experiment, comprising two simulated data sets: Circle and  Moon, 
and four real datasets: 
Office-Caltech \citep{saenko2010adapting,gong2012geodesic}\footnote{https://people.eecs.berkeley.edu/~jhoffman/domainadapt/}, 
Breast cancer, 
Blood, 
Haberman 
where the last three real datasets are obtained from UCI \footnote{http://archive.ics.uci.edu/ml} and \cite{SegMazLiu2023}. 
The Circle and  Moon datasets 
are generated by the related packages in scikit-learn\footnote{https://scikit-learn.org/stable/}. 
For Circle and Moon, we first generate the source data with a noise level of   $0.1$ and a factor of $0.6$, 
where the noise represents the standard deviation of Gaussian noise added to the data, and the factor denotes the ratio between the inner and outer circles.

To introduce covariate shift, we apply biased sampling to construct the source and target datasets, following \cite{huang2006correcting, SegMazLiu2023}. 
After obtaining the covariate shift datasets, we map them into a higher-dimensional space by multiplying them with a random high-dimensional matrix. Table \ref{tab1} presents detailed information on the simulated datasets, where the ``Covariates" column denotes the number of intrinsic covariates, and the ``Features" column indicates the high-dimensional feature shape after mapping the data into the high-dimensional space.

Biased sampling preserves the conditional label law at the original covariates when selection depends only on those covariates and both domains have positive selection probabilities. An injective embedding also preserves that law.

The real dataset Office-Caltech consists of  four domains: 
Caltech(C), Amazon(A), Webcam(W), DSLR(D). 
It includes $10$ classes and a total of $2533$ samples. The data distribution varies across different domains, even for the same labels. This variation alone does not establish the covariate-shift assumption used in the theory. The ten-class experiment is an empirical extension of the binary classification setting analyzed by the theorems. Further details regarding this dataset can be found in Table \ref{tab1} in Appendix \ref{appendix}. In this work, we consider different combinations of source and target domains. During model training, the data undergoes transformations such as normalization, with specific transformation details provided in Appendix \ref{appendix}.

\subsection{Model Structure}
\begin{itemize}
    \item To learn the representation, we employ two kinds of models for simulated and real datasets. 
    For binary classification on the simulated datasets, we use a $10$-layer DenseNet \citep{huang2017densely} as the representation network $R$. For multiclass classification on Office-Caltech, we use ResNet-18 \citep{he2016deep}. A $4$-layer fully connected DNN serves as the discriminator $D$. We evaluate both random and pretrained initialization.
    Additional details regarding the model architecture can be found in Appendix \ref{appendix}.
        \item For the density ratio model, we utilize DenseNet for Methods \eqref{MeB} and \eqref{MeD}, as these methods handle high-dimensional data. In contrast, for Method \eqref{MeA}, which deals with low-dimensional latent representations, we employ a 
    4-layer  DNN.
    
    \item 
    We use a $3$-layer fully connected DNN as the final predictor $\hat{h}$.
\end{itemize}

\subsection{\texorpdfstring{Hyperparameter Settings}{Hyperparameter Settings}}
We use Adam with a batch size of $64$ and train for $200$ epochs. The dataset-specific learning rate $lr$ and latent dimension for representation learning (Step $1$ in Algorithm \ref{alg:al2}) are listed in Table \ref{tab3}. The density-ratio estimation and prediction stages use a learning rate of $10^{-3}$. We use the 
exponentially-flattened technique \citep{shimodaira2000improving,zhang2020one} to reduce the variance caused by density ratio \cite{zhang2020one}. 
This technique involves replacing the estimated ratio $\hat r$ with $\hat r^{\alpha_w}$, where $\alpha_w\in(0,1]$ is the weight-tempering exponent, distinct from the representation-penalty coefficient $\lambda$ in \eqref{eq:rep-training}. Table \ref{tab3} in Appendix \ref{appendix} lists the dataset-specific hyperparameters.

Tempering with an exponent below one generally introduces bias relative to full-ratio target-risk correction. The rate theorems analyze full-ratio weighting with clipped fully connected network sieves and empirical risk minimization; the experiments use tempered weights, fixed architectures and finite-epoch Adam training.

\subsection{Baselines}
In our experiments, we select a series of TL methods as comparison baselines, including DANN \citep{ganin2015unsupervised}, DAN \citep{long2015learning}, CDAN~\citep{long2018conditional}, AFN \citep{xu2019larger}, BSP \citep{chen2019transferability}, JAN \citep{long2017deep}, MCC \cite{jin2020minimum}, MCD \citep{saito2018maximum}, MDD \citep{zhang2019bridging}, and ERM. We use their implementations in the Transfer Learning Library \citep{tllib}. For a fair comparison, we use the same experimental data, parameter settings and training strategies. We report both random and pretrained initialization. Random initialization assesses feature learning from scratch, whereas pretrained initialization assesses transfer with existing features; both provide empirical comparisons.

The Office-Caltech comparisons evaluate complete adaptation pipelines with different objectives. The four-variant comparison separately examines the roles of representation and weight estimation.

\subsection{Performance Evaluation}\label{per eval}
We evaluate the performance of our proposed method by comparing the top-1 accuracy on the target test dataset. The experiments are conducted 10 times for each method, and the average result is taken as the final performance metric.
Figure \ref{fig2} and Table \ref{tab2} compare Methods \eqref{MeA}--\eqref{MeD}. Tables \ref{tab_compare2} and \ref{tab_compare1} compare our method with the baselines from \cite{tllib} under random initialization; Tables \ref{tab_compare3} and \ref{tab_compare4} report the corresponding comparisons under pretrained initialization.
We conclude the following:
\begin{itemize}
\item 
In Table \ref{tab2}, Method \eqref{MeB} improves on Method \eqref{MeC} on several datasets, whereas Moon gives the opposite ordering. Thus the accuracy gain from original-space weighting with learned features depends on the dataset.
\item 
Method \eqref{MeA} has the highest reported accuracy among the four variants in Table \ref{tab2}. The combination of learned representations and latent reweighting outperforms each control in these configurations.
\item 
Tables \ref{tab_compare2} and \ref{tab_compare1} show competitive results against many baselines under random initialization. With pretrained initialization, Table \ref{tab_compare3} shows the highest reported accuracy for target D but lower accuracy than all listed baselines for targets A, C and W. Table \ref{tab_compare4} also shows substantial gaps on several transfers. The gains over domain-adaptation baselines are concentrated in the random-initialization comparisons.
\end{itemize}

\begin{figure}[tbp]%
    \centering
    \includegraphics[width=0.85\textwidth]{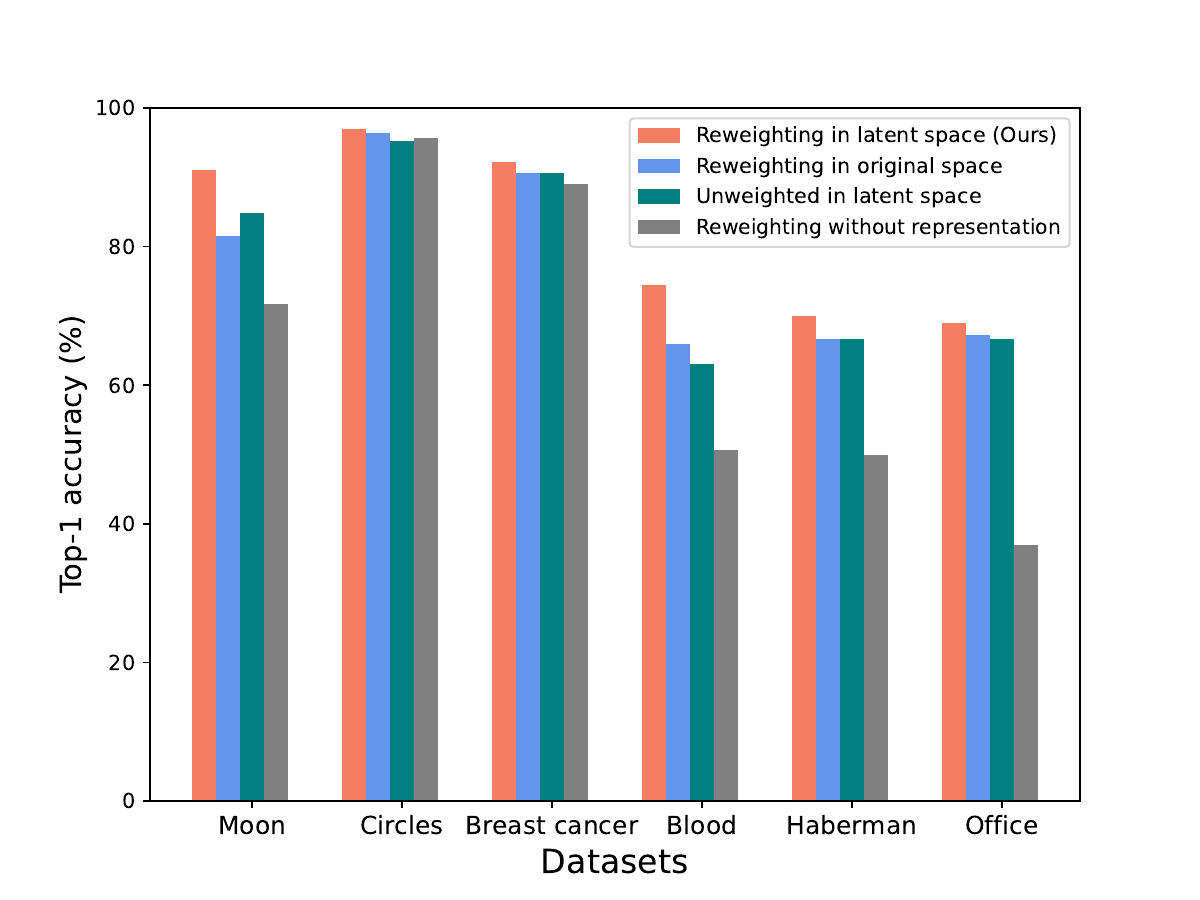}
    \caption{The Top-1 accuracy of four methods on simulated and real datasets.}
    \label{fig2}
\end{figure}

Tables \ref{tab_compare2}, \ref{tab_compare1}, \ref{tab_compare3} and \ref{tab_compare4} compare random initialization with ResNet weights pretrained on ImageNet \citep{deng2009imagenet}. These results identify initialization and domain choice as key factors in prediction accuracy. Yosinski et al.~\citeyearpar{yosinski2014transferable} likewise show that feature transfer depends on the source task and the layers transferred. Measuring representation sufficiency and density-ratio error would help explain their effects.

\begin{table}[tbp]

\centering
\begin{small}

\begin{tabular}{ccccc}
    \toprule
        \multirow{2}{*}{\textbf{Dataset}} & \multicolumn{4}{c}{\textbf{Top-1 accuracy(\%)}} \\ \cline{2-5}
        ~       & Method \eqref{MeA}          &Method \eqref{MeB}      &Method \eqref{MeC}  &Method \eqref{MeD}     \\ \cline{2-5}
        Circle  & \textbf{97.02}    & 96.40        &95.21      &95.70       \\ 
        Moon    & \textbf{91.08}    &81.55         &84.82      &71.72       \\ 
        Blood   & \textbf{74.42}    &65.98         &62.99      &50.62         \\ 
Breast cancer  &\textbf{92.19}      & 90.63        &90.63      &89.06        \\
    Haberman   &\textbf{70.00}      & 66.67        & 66.67     & 50.0        \\
Office-Caltech  &\textbf{68.93}     & 67.19        & 66.67     &36.98         \\
                    \bottomrule  
    \end{tabular}

\end{small}
\caption{The performance of four methods on simulated and real datasets.}
    \label{tab2}
\end{table}

\begin{table*}[tbp]

\centering
\begin{small}

\begin{tabularx}{\linewidth}{@{}X*{4}{c}@{}}
    \toprule
Methods     & \shortstack{\{C,D,W\}\\$\to$ A}         &\shortstack{\{A,D,W\}\\$\to$ C}     &\shortstack{\{A, C, D \}\\$\to$ W}   &\shortstack{\{A, C, W\}\\$\to$ D}   \\ \cline{1-5}
                AFN~\citep{xu2019larger}  &  47.6   & 37.9     & 45.8     &  61.3   \\
        BSP~\citep{chen2019transferability}  & 60.7    & 35.3      & 55.9    &54.8    \\
         CDAN \citep{long2018conditional} &61.3      &  40.2     &57.6    &64.5  \\
        DANN \citep{ganin2015unsupervised}  &59.2    &42.0      & 64.4     & 58.1   \\ 
        DAN \citep{long2015learning}   &59.7      &40.2       & 40.7     &64.5    \\ 
        ERM     &59.2      &35.3       & 52.5     &71.0 \\ 
                JAN \citep{long2017deep}   &56.5      & 35.7      &54.2      & 67.7     \\
        MCC \citep{jin2020minimum}    &64.4      & 34.4      &62.7      &64.5      \\
        MCD \citep{saito2018maximum}   &48.7      & 30.8      & 40.7     & 58.1   \\
        MDD \citep{zhang2019bridging}   &52.4      & 32.6      &44.1      &64.5    \\
        Reweighting in latent space (Ours) \eqref{MeA}   & \textbf{76.56}     &\textbf{56.88}       &\textbf{67.80}      &\textbf{87.5}    \\
    \bottomrule    
    \end{tabularx}

\end{small}
\caption{The Top-1 accuracy(\%) comparison of proposed method with others on Office-Caltech dataset without pretrained model weight.}
\label{tab_compare2}
\end{table*}

\begin{table}[tbp]

\centering
\begin{small}

\begin{tabularx}{\linewidth}{@{}X*{6}{c}@{}}
    \toprule
        \textbf{Methods}              &A $\to$ W      &D $\to$ W   &W $\to$ D  & A $\to$ D  &D $\to$ A  &W $\to$ A \\ \cline{1-7}
                 AFN \citep{xu2019larger}     &23.7       &47.5      & 77.4   & 9.7    &9.9   &19.9 \\
         BSP~\citep{chen2019transferability}      & 42.4      & 33.9     &64.5    & 16.1    & 12.6   & 12.6\\
         CDAN \citep{long2018conditional}     &13.6       & 50.8      & 35.5    &25.8     & 9.4   &13.6 \\
        DANN \citep{ganin2015unsupervised}     &30.5       & 32.2     &32.3    & 25.8    &16.2    &6.8     \\ 
        DAN \citep{long2015learning}      & 22.0      & 50.8     &51.6    &9.7     &5.2    &13.1   \\ 
        ERM      & 11.9      &35.6      &38.7    &9.7     & 7.9   & 8.4          \\ 
                JAN \citep{long2017deep}   & 28.8      & 23.7     &45.2    & 16.1   &15.7    &16.8         \\
        MCC~\citep{jin2020minimum}     & 10.2     & 30.5     &51.6    &12.9     &9.4    &  8.9        \\
        MCD \citep{saito2018maximum}     & 11.9      &25.4      &48.4    & 38.7    &18.8    &   24.1         \\
        MDD \citep{zhang2019bridging}     & 11.9      &45.8      &61.3    & 19.4    & 9.4   &8.9 \\
        Reweighting in latent space (Ours) \eqref{MeA}     & 38.98      &42.37      &65.63    &34.38     &13.54    & 24.48        \\
    \bottomrule    
    \end{tabularx}

\end{small}
\caption{The Top-1 accuracy (\%) comparison of proposed method with others on Office-Caltech dataset without pretrained model weight.}
\label{tab_compare1}
\end{table}

\begin{table}[tbp]

\centering
\begin{small}

\begin{tabularx}{\linewidth}{@{}X*{4}{c}@{}}
    \toprule
Methods     & \shortstack{\{C,D,W\}\\$\to$ A}         &\shortstack{\{A,D,W\}\\$\to$ C}     &\shortstack{\{A, C, D \}\\$\to$ W}   &\shortstack{\{A, C, W\}\\$\to$ D}   \\ \cline{1-5}
        AFN~\citep{xu2019larger}  & 96.3   &92.9      & 96.6     & 96.8    \\
        BSP~\citep{chen2019transferability}  &95.3    &86.2      &96.8      & 96.6    \\
         CDAN \citep{long2018conditional} &95.8    &87.9      &86.4      & 93.5    \\
        DANN \citep{ganin2015unsupervised}  &95.8    &87.5      &88.1      & 96.8    \\ 
        DAN \citep{long2015learning}   &95.3    & 87.1     & 100.0     & 96.6    \\ 
        ERM     &95.8    &87.5      & 100.0     &91.5     \\
        JAN \citep{long2017deep}   &95.8    & 90.2     & 90.3     &  94.9   \\
        MCC \citep{jin2020minimum}    &96.9    &92.9      & 100.0     & 96.8    \\
        MCD \citep{saito2018maximum}   & 95.3   &87.9      & 94.9     & 96.8    \\
        MDD \citep{zhang2019bridging}   &94.8    & 84.4     &98.3      & 96.8    \\
        Reweighting in latent space (Ours) \eqref{MeA}   &85.9   &76.9     & 83.4    & 100.0    \\
    \bottomrule    
    \end{tabularx}

\end{small}
\caption{The Top-1 accuracy(\%) comparison of proposed method with others on Office-Caltech dataset using pretrained model weight.}
\label{tab_compare3}
\end{table}

\begin{table}[tbp]

\centering
\begin{small}

\begin{tabularx}{\linewidth}{@{}X*{6}{c}@{}}
    \toprule
        \textbf{Methods}              &A $\to$ W      &D $\to$ W   &W $\to$ D  & A $\to$ D  &D $\to$ A  &W $\to$ A \\ \cline{1-7}
         AFN \citep{xu2019larger}     &93.2       & 100.0     &100.0    & 100.0  &83.8   & 86.9\\
         BSP~\citep{chen2019transferability}       &91.5       & 100.0     &100.0   & 96.8  & 84.8  & 84.3\\
         CDAN \citep{long2018conditional}      & 83.1      &96.6      &100.0   &93.5   &88.0   &84.3 \\
        DANN \citep{ganin2015unsupervised}      & 81.4      & 100.0     &100.0   &64.5   &88.5   &83.8 \\ 
        DAN \citep{long2015learning}       & 89.8      & 96.6     &100.0   & 90.3  &86.4   & 84.8\\ 
        ERM       &79.7       & 81.4     &100.0   &83.9   &73.3   &78.5 \\ 
        JAN \citep{long2017deep}    & 88.1      &94.9      &100.0   & 90.3  &80.6   &86.9 \\
        MCC~\citep{jin2020minimum}      &96.6       & 100.0     & 100.0  & 100.0  &86.9   &81.7 \\
        MCD \citep{saito2018maximum}      &93.2       &96.6      &100.0   &71.0   &80.6   & 79.1 \\
        MDD \citep{zhang2019bridging}     & 61.0      & 91.5     &96.8   & 67.7  &61.8   & 58.6\\
        Reweighting in latent space (Ours) \eqref{MeA}     & 57.6      & 93.2     & 93.8  &71.9   &63.0   &55.7 \\
    \bottomrule    
    \end{tabularx}

\end{small}
\caption{The Top-1 accuracy (\%) comparison of proposed method with others on Office-Caltech dataset using pretrained model weight.}
\label{tab_compare4}
\end{table}

\FloatBarrier
\section{Conclusion}\label{conclusion}
We developed a three-stage algorithm that combines representation learning with latent-space importance weighting for covariate shift. Our finite-sample analysis quantifies how representation insufficiency affects target risk. Under strict sufficiency and uniform regularity, latent weighting replaces the ambient dimension by the lower latent dimension in the density-ratio estimation term, yielding a faster neural excess-risk upper bound under matched smoothness and comparable fold sizes. For a fixed sufficient map and a fixed finite-dimensional predictor class, conditional averaging of the weights yields a strictly smaller leading expected excess-risk constant under the stated nondegeneracy and weight-estimation conditions. A coordinate-selection model realizes these gains with learned representations and estimated weights. Across six synthetic and real datasets, the proposed method achieves the highest mean target accuracy among four controlled variants.

Extending the representation-learning guarantees to general nonlinear classes and quantifying how weight estimation changes the leading fitted-risk constants at comparable domain and prediction sample sizes remain open problems. Together, our results show how sufficient representations improve covariate-shift correction by reducing both the dimension of ratio estimation and the residual variation of importance weights.

\appendix
\section{}\label{appendix}

\subsection{Proof of Lemma \ref{th1}}\label{pflemm1}
\begin{proof}
Let $U=R^*(X)$ and $m_h(u):=\mb E_{P_{\mc S}}[\ell(h(U),Y)\mid U=u]$. By Assumption \ref{ass2.1}, the same conditional loss applies in the target domain. Since $r_d^*=dP_{\mc T}^{R^*}/dP_{\mc S}^{R^*}$,
\begin{equation}\label{eq2.5}
\begin{split}
 \mb E_{P_{\mc S}}[r_d^*(U)\ell(h(U),Y)]
 &=\int r_d^*(u)m_h(u)\,dP_{\mc S}^{R^*}(u)\\
 &=\int m_h(u)\,dP_{\mc T}^{R^*}(u)
 =\mb E_{P_{\mc T}}[\ell(h(U),Y)].
\end{split}
\end{equation}
The two risks therefore have the same minimizers, and the logistic Bayes score is unique $P_{\mc T}^{R^*}$-almost everywhere.
\end{proof}

\subsection{General Finite-Sample Bound for Learned Representations}\label{pfth6}
This subsection establishes the proposed method's bound without assuming approximation or representation-recovery rates. Retain \eqref{excess error} and define its empirical version by
\begin{equation*}
 \widehat{\mathcal R}_{r_d,R}(h):=\frac1n\sum_{i=1}^n
 r_d(R(X_i^{\mathcal S}))\ell(h(R(X_i^{\mathcal S})),Y_i^{\mathcal S}).
\end{equation*}
Using the excess risk defined in Section \ref{convergence rate}, put
\begin{equation*}
\begin{split}
 \epsilon_{r,A}&:=\|\hat r_d-r_{\hat R}\|_{L^1(P_{\mathcal S}^{\hat R})},\qquad
 a_R:=\inf_{h\in\mathcal H}\mathcal R_{\mathcal T}(h\circ R)
                  -\mathcal R_{\mathcal T}(h^*\circ R^*).
\end{split}
\end{equation*}
The excess risk is nonnegative; $a_R$ includes representation information loss and predictor-class approximation. The subscript $A$ in the ratio error and complexity identifies Method \eqref{MeA}. Let $\mathscr G$ be the $\sigma$-field generated by the first two stages' data, reference samples and training randomness. Their fitted functions are $\mathscr G$-measurable, and $D_{\rm pred}$ is independent of $\mathscr G$. Define the conditional absolute weighted-loss complexity
\begin{equation*}
 \mathfrak C_{n,A}:=\mathbb E_{D_{\rm pred},\zeta}\!\left[
 \sup_{h\in\mathcal H}\left|\frac1n\sum_{i=1}^n
 \zeta_i\hat r_d(\hat R(X_i))\ell(h(\hat R(X_i)),Y_i)\right|
 \;\middle|\;\mathscr G\right].
\end{equation*}
Here $(X_i,Y_i)\in D_{\rm pred}$, and the signs $\zeta_i$ are independent Rademacher variables. Conditioning fixes the representation and weights while leaving the prediction sample i.i.d. from $P_{\mathcal S}$. It does not make a reused training sample independent. All estimators and suprema are assumed measurable.

\begin{lemma}\label{lem:latent-oracle}
Under Assumptions \ref{ass2.0}--\ref{ass2.1}, \ref{assum1} and \ref{assum4}, use the independent-fold Algorithm \ref{alg:al2} with $\hat h$ minimizing $\widehat{\mathcal R}_{\hat r_d,\hat R}$ over $\mathcal H$. Suppose $|\hat\phi_d|\leq M$ with $M\geq\max\{|\log c_r|,|\log C_r|\}$. Then
\begin{equation}\label{eq:latent-oracle-bound}
 \mathbb E\mathcal E_{\mathcal T}(\hat h,\hat R)
 \leq4\mathbb E\mathfrak C_{n,A}
 +2L_{\mathcal B}\mathbb E\epsilon_{r,A}
 +2\mathcal B\mathbb E b_{\hat R}+\mathbb E a_{\hat R}.
\end{equation}
For the balanced pushforward domain law, let
\begin{equation*}
 \mathcal L_{\mu,R}(\phi):=\mathbb E_{(U,Z)\sim P_{\mu,Z}^R}\log(1+e^{-Z\phi(U)}),\qquad
 q_A:=\mathcal L_{\mu,\hat R}(\hat\phi_d)
       -\mathcal L_{\mu,\hat R}(\log r_{\hat R}).
\end{equation*}
With $\kappa_M:=e^M/(1+e^M)^2$,
\begin{equation}\label{eq:ratio-root-conversion}
 \mathbb E\epsilon_{r,A}\leq\frac{2e^M}{\sqrt{\kappa_M}}\sqrt{\mathbb E q_A}.
\end{equation}
Neither a vanishing representation error nor smoothness of $\log r_{\hat R}$ is needed for these inequalities.
\end{lemma}

\begin{proof}
For any fixed $R$ and bounded measurable test function $\psi$ on its range, change of measure gives
\begin{equation*}
\begin{split}
 \mathbb E_{P_{\mathcal S}}[\psi(R(X))r^*(X)]
 &=\mathbb E_{P_{\mathcal T}}\psi(R(X))\\
 &=\int\psi(u)r_R(u)\,dP_{\mathcal S}^R(u)\\
 &=\mathbb E_{P_{\mathcal S}}[\psi(R(X))r_R(R(X))].
\end{split}
\end{equation*}
Since $r_R(R(X))$ is measurable with respect to $R(X)$, the defining property of conditional expectation yields
\begin{equation}\label{eq:target-risk-bridge}
 r_R(R(X))=\mathbb E_{P_{\mathcal S}}[r^*(X)\mid R(X)],\qquad
 c_r\leq r_R\leq C_r\quad P_{\mathcal S}^R\text{-a.e.}
\end{equation}
The equality holds $P_{\mathcal S}$-almost surely; the bounds follow by taking the conditional expectation of $c_r\leq r^*(X)\leq C_r$. No Lebesgue density for the pushforward is required. Conditioning the logistic loss on $U=R(X)$ gives
\begin{equation*}
\begin{split}
 \mathcal R_{\mathcal T}(h\circ R)
 &=\int[\log(1+e^{h(u)})-h(u)\eta_{\mathcal T}^R(u)]\,dP_{\mathcal T}^R(u),\\
 \mathcal R_{r_R,R}(h)
 &=\int r_R(u)[\log(1+e^{h(u)})-h(u)\eta_{\mathcal S}^R(u)]\,dP_{\mathcal S}^R(u)\\
 &=\int[\log(1+e^{h(u)})-h(u)\eta_{\mathcal S}^R(u)]\,dP_{\mathcal T}^R(u).
\end{split}
\end{equation*}
Subtracting cancels the log terms and gives
\begin{equation*}
 \mathcal R_{\mathcal T}(h\circ R)-\mathcal R_{r_R,R}(h)
 =\mathbb E_{P_{\mathcal T}^R}
 [h(U)(\eta_{\mathcal S}^R(U)-\eta_{\mathcal T}^R(U))].
\end{equation*}
Its absolute value is at most $\mathcal B b_R$. For $R=\hat R$, insert the true learned-space ratio between the target and estimated-weight risks. For every $h\in\mathcal H$,
\begin{equation*}
\begin{split}
 &|\mathcal R_{\mathcal T}(h\circ\hat R)-\mathcal R_{\hat r_d,\hat R}(h)|\\
 &\leq|\mathcal R_{\mathcal T}(h\circ\hat R)-\mathcal R_{r_{\hat R},\hat R}(h)|
       +|\mathcal R_{r_{\hat R},\hat R}(h)-\mathcal R_{\hat r_d,\hat R}(h)|\\
 &\leq\mathcal B b_{\hat R}
       +L_{\mathcal B}\mathbb E_{P_{\mathcal S}}
       |r_{\hat R}(\hat R(X))-\hat r_d(\hat R(X))|\\
 &=\mathcal B b_{\hat R}+L_{\mathcal B}\epsilon_{r,A}.
\end{split}
\end{equation*}
Here $0\leq\ell\leq L_{\mathcal B}$, and the last expectation is the $L^1(P_{\mathcal S}^{\hat R})$ norm. Taking the supremum over $h$ yields
\begin{equation}\label{eq:explicit-discrepancy}
 \Delta_A:=\sup_{h\in\mathcal H}
 |\mathcal R_{\mathcal T}(h\circ\hat R)-\mathcal R_{\hat r_d,\hat R}(h)|
 \leq L_{\mathcal B}\epsilon_{r,A}+\mathcal B b_{\hat R}.
\end{equation}
This is a pathwise risk comparison and does not itself require sample independence.

For any comparison function $h\in\mathcal H$, expand
\begin{equation}\label{eq3.1}
\begin{split}
 &\mathcal R_{\mathcal T}(\hat h\circ\hat R)
             -\mathcal R_{\mathcal T}(h^*\circ R^*)\\
 &=\mathcal R_{\mathcal T}(\hat h\circ\hat R)
             -\mathcal R_{\hat r_d,\hat R}(\hat h)
   +\mathcal R_{\hat r_d,\hat R}(\hat h)
             -\widehat{\mathcal R}_{\hat r_d,\hat R}(\hat h)\\
 &\quad+\widehat{\mathcal R}_{\hat r_d,\hat R}(\hat h)
             -\widehat{\mathcal R}_{\hat r_d,\hat R}(h)
   +\widehat{\mathcal R}_{\hat r_d,\hat R}(h)
             -\mathcal R_{\hat r_d,\hat R}(h)\\
 &\quad+\mathcal R_{\hat r_d,\hat R}(h)
             -\mathcal R_{\mathcal T}(h\circ\hat R)
   +\mathcal R_{\mathcal T}(h\circ\hat R)
             -\mathcal R_{\mathcal T}(h^*\circ R^*).
\end{split}
\end{equation}
The empirical-ERM comparison is nonpositive \citep{geer2000empirical}. Put, within this proof,
$G_n:=\sup_{h\in\mathcal H}|\mathcal R_{\hat r_d,\hat R}(h)-\widehat{\mathcal R}_{\hat r_d,\hat R}(h)|$.
Taking the infimum over $h$ gives
\begin{equation}\label{eq_errdecomp}
 \mathcal E_{\mathcal T}(\hat h,\hat R)\leq2G_n+2\Delta_A+a_{\hat R}.
\end{equation}
Conditional on $\mathscr G$, the weighted loss class is fixed. For an independent ghost prediction sample $D_{\rm pred}'=\{(X_i',Y_i')\}_{i=1}^n$, write its empirical risk as $\widehat{\mathcal R}'$. Jensen's inequality and conditional symmetrization \citep{vaart1997weak,gine2021mathematical} give
\begin{equation*}
\begin{split}
 \mathbb E[G_n\mid\mathscr G]
 &\leq\mathbb E_{D_{\rm pred},D_{\rm pred}'}\sup_h
 |\widehat{\mathcal R}'_{\hat r_d,\hat R}(h)
              -\widehat{\mathcal R}_{\hat r_d,\hat R}(h)|\\
 &=\mathbb E_{D_{\rm pred},D_{\rm pred}',\zeta}\sup_h
 \left|\frac1n\sum_i\zeta_i[q_h(X'_i,Y'_i)-q_h(X_i,Y_i)]\right|
 \leq2\mathfrak C_{n,A},
\end{split}
\end{equation*}
where $q_h(x,y):=\hat r_d(\hat R(x))\ell(h(\hat R(x)),y)$ is used only in this proof, and all expectations in the display are conditional on $\mathscr G$. Substituting \eqref{eq:explicit-discrepancy} into \eqref{eq_errdecomp} and taking full expectations proves \eqref{eq:latent-oracle-bound}.

For the ratio conversion, fix the learned functions and use $P_\mu^{\hat R}=(P_{\mathcal S}^{\hat R}+P_{\mathcal T}^{\hat R})/2$. Equation \eqref{eq:target-risk-bridge} bounds the true log ratio by $M$. Cauchy--Schwarz and $P_{\mathcal S}^{\hat R}\leq2P_\mu^{\hat R}$ give
\begin{equation*}
 \|\hat r_d-r_{\hat R}\|_{L^1(P_{\mathcal S}^{\hat R})}
 \leq\|\hat r_d-r_{\hat R}\|_{L^2(P_{\mathcal S}^{\hat R})}
 \leq\sqrt2\|\hat r_d-r_{\hat R}\|_{L^2(P_\mu^{\hat R})}.
\end{equation*}
Both log scores lie in $[-M,M]$, where exponentiation is $e^M$-Lipschitz. Hence
\begin{equation*}
 \|\hat r_d-r_{\hat R}\|_{L^2(P_\mu^{\hat R})}
 \leq e^M\|\hat\phi_d-\log r_{\hat R}\|_{L^2(P_\mu^{\hat R})}.
\end{equation*}
Under the balanced domain law, $P(Z=1\mid U=u)=r_{\hat R}(u)/(1+r_{\hat R}(u))$. The conditional domain loss at a score $t$ is
\begin{equation*}
 \log(1+e^t)-\frac{r_{\hat R}(u)}{1+r_{\hat R}(u)}t.
\end{equation*}
Its derivative vanishes at $t=\log r_{\hat R}(u)$, and its second derivative on $[-M,M]$ is at least $\kappa_M$. Taylor's theorem and integration therefore give, as in Lemma \ref{lem2.1},
\begin{equation*}
 q_A\geq\frac{\kappa_M}{2}
 \|\hat\phi_d-\log r_{\hat R}\|_{L^2(P_\mu^{\hat R})}^2.
\end{equation*}
Combining the three bounds yields
\begin{equation*}
 \epsilon_{r,A}\leq\sqrt2\|\hat r_d-r_{\hat R}\|_{L^2(P_\mu^{\hat R})}
 \leq\sqrt2e^M\|\hat\phi_d-\log r_{\hat R}\|_{L^2(P_\mu^{\hat R})}
 \leq\frac{2e^M}{\sqrt{\kappa_M}}\sqrt{q_A}.
\end{equation*}
Jensen's inequality proves \eqref{eq:ratio-root-conversion}. This pathwise conversion does not analyze how the domain classifier was trained; its estimation bound in Appendix \ref{app:explicit-rates} conditions only on the independent representation fold.
\end{proof}

\begin{corollary}\label{cor:conditional-rate}
Under Lemma \ref{lem:latent-oracle}, suppose additionally that, for some sequences $s_n,u_m,a_n,b_n$,
\begin{equation*}
 \mb E[\mathfrak C_{n,A}]\leq s_n,\quad\mb E[q_A]\leq u_m,\quad
 \mb E[a_{\hat R}]\leq a_n,\quad\mb E[b_{\hat R}]\leq b_n.
\end{equation*}
Then
\begin{equation*}
 \mb E[\mathcal E_{\mc T}(\hat h,\hat R)]
 \leq4s_n+\frac{4L_{\mc B}e^M}{\sqrt{\kappa_M}}\sqrt{u_m}+a_n+2\mc B b_n.
\end{equation*}
In particular, with $\mc B,M$ fixed, if
$s_n+a_n=\widetilde O(n^{-\beta_{\mc H}/(2\beta_{\mc H}+d)})$ and
$u_m=\widetilde O(m^{-2\beta_{\ddot\Phi}/(2\beta_{\ddot\Phi}+d)})$, then
\begin{equation*}
 \mb E[\mathcal E_{\mc T}(\hat h,\hat R)]
 \leq\widetilde O\!\left(n^{-\beta_{\mc H}/(2\beta_{\mc H}+d)}
     +m^{-\beta_{\ddot\Phi}/(2\beta_{\ddot\Phi}+d)}\right)+2\mc B b_n.
\end{equation*}
Here $\beta_{\ddot\Phi}>0$ is the stipulated learned-space domain-score rate parameter, and $\widetilde O$ suppresses logarithmic factors. The sequences may also depend on the representation-training fold size $n_R$. These stagewise rate bounds are additional conditions. Appendix \ref{app:sufficiency} derives the main theorem's stagewise bounds from strict sufficiency and learned-space regularity; Appendix \ref{app:explicit-rates} separately proves the recovery-based bounds. In particular, an oracle smoothness assumption for $\log r_d^*$ does not establish the same smoothness or estimation rate for $\log r_{\hat R}$.
\end{corollary}

For completeness, a sufficient coordinate comparison can bound the two representation terms without regularity of every $h\in\mc H$. Suppose a chosen representative $R^*$ and the learned $R$ have a common latent range on which
$|h^*(u)-h^*(v)|\leq C_h\|u-v\|^\nu$, $0<\nu\leq1$. Define
$t_R:=\mb E_{P_{\mc S}}\|R(X)-R^*(X)\|^\nu$ and
$\delta_H:=\inf_{h\in\mc H}\|h-h^*\|_\infty$ on this common range. Then
\begin{equation}\label{eq:representation-comparison}
 b_R\leq\frac{C_rC_h}{2}t_R,\qquad
 a_R\leq\delta_H+C_rC_h t_R.
\end{equation}
Indeed, sufficiency and covariate shift give $\eta(X)=\sigma(h^*(R^*(X)))$ in both domains. Set $v_R(X)=\sigma(h^*(R(X)))$, which is measurable with respect to $R(X)$. The tower property gives $\eta_k^R(R(X))=\mathbb E_{P_k}[\eta(X)\mid R(X)]$, for $k\in\{\mathcal S,\mathcal T\}$. Subtract $v_R$ from these two conditional means and apply conditional Jensen's inequality to obtain
\begin{equation*}
\begin{split}
 b_R&\leq\mb E_{P_{\mc T}^R}\mb E_{P_{\mc S}}[|\eta(X)-v_R(X)|\mid R(X)]
       +\mb E_{P_{\mc T}}|\eta(X)-v_R(X)|\\
 &=\mb E_{P_{\mc S}}[r_R(R(X))|\eta(X)-v_R(X)|]
       +\mb E_{P_{\mc S}}[r^*(X)|\eta(X)-v_R(X)|]\\
 &\leq2C_r\mb E_{P_{\mc S}}|\eta(X)-v_R(X)|.
\end{split}
\end{equation*}
Since $\sigma'(t)=\sigma(t)(1-\sigma(t))\leq1/4$,
\begin{equation*}
 |\eta(X)-v_R(X)|
 \leq\tfrac14|h^*(R^*(X))-h^*(R(X))|
 \leq\tfrac14C_h\|R(X)-R^*(X)\|^\nu.
\end{equation*}
This proves $b_R\leq C_rC_h t_R/2$. For $a_R$, the 1-Lipschitz property of the logistic loss in its score gives, for each $h\in\mathcal H$,
\begin{equation*}
\begin{split}
 &\mathcal R_{\mathcal T}(h\circ R)-\mathcal R_{\mathcal T}(h^*\circ R^*)\\
 &\leq\mathbb E_{P_{\mathcal T}}|h(R(X))-h^*(R(X))|
       +\mathbb E_{P_{\mathcal T}}|h^*(R(X))-h^*(R^*(X))|\\
 &\leq\|h-h^*\|_\infty+C_h\mathbb E_{P_{\mathcal T}}\|R(X)-R^*(X)\|^\nu\\
 &\leq\|h-h^*\|_\infty+C_rC_h t_R.
\end{split}
\end{equation*}
Taking the infimum over $h$ proves the second bound in \eqref{eq:representation-comparison}. If an aligned representation bound satisfies $\mathbb E\|\hat R-R^*\|_{L^1(P_{\mathcal S})}\leq\rho_{n_R}$, Jensen's inequality with $0<\nu\leq1$ gives
\begin{equation*}
 \mathbb E[t_{\hat R}]
 \leq\left(\mathbb E\mathbb E_{P_{\mathcal S}}\|\hat R(X)-R^*(X)\|\right)^\nu
 \leq\rho_{n_R}^\nu.
\end{equation*}
Thus $2\mathcal B\mathbb E b_{\hat R}+\mathbb E a_{\hat R}\leq\mathbb E\delta_H+(\mathcal B+1)C_rC_h\rho_{n_R}^\nu$. The recovery analysis below quantifies $\rho_{n_R}$ under Assumptions \ref{assum3}, \ref{assum5} and \ref{ass:explicit-rates}, with an aligned choice of $R^*$. Theorem \ref{Th5} instead uses strict sufficiency of the fitted representation and requires no coordinate recovery.

\subsection{Representation Recovery with Variational Bias}\label{app:representation-recovery}
This section controls representation error through aligned coordinate recovery and propagates it into the target-risk bound. The recovery and discriminator assumptions replace strict sufficiency of the fitted representation, and the variational bias appears explicitly in the resulting rate.

\begin{assumption}\label{assum3}
For the fitted representation and a fixed constant $0<C_{\rm id}<\infty$,
\begin{equation*}
 \|\hat R-R^*\|_{L^1(P_{\mathcal S})}
 \leq C_{\rm id}|\mathcal L(\hat R)-\mathcal L(R^*)|,
\end{equation*}
where $\mathcal L$ is the population representation objective in \eqref{eq2.7}.
\end{assumption}

\begin{assumption}\label{assum5}
For fixed compact cubes $K_X\subset\mathbb R^p$, $K_U\subset\mathbb R^d$ and fixed $\beta_{\mathcal F},\beta_{\mathcal H},B>0$,
\begin{equation*}
\begin{gathered}
 R_j^*|_{K_X}\in\mho^{\beta_{\mathcal F}}(K_X,B),\quad j=1,\ldots,d,\\
 h^*|_{K_U}\in\mho^{\beta_{\mathcal H}}(K_U,B).
\end{gathered}
\end{equation*}
The H\"older classes are those of Definition \ref{def:hc} after affine rescaling.
\par
\end{assumption}

\begin{remark}
For the specified loss on $[-\mathcal B,\mathcal B]\times[0,1]$,
\begin{equation*}
 |\partial_a\ell|=|\sigma(a)-y|\leq1,\qquad
 |\partial_y\ell|=|a|\leq\mathcal B,\qquad
 0\leq\ell(a,y)\leq\log(1+e^{\mathcal B})=L_{\mathcal B}.
\end{equation*}
These are derived properties. Assumption \ref{assum3} is a linear recovery condition. For $Q^\top Q=I_d$, orthogonal invariance of Euclidean distances and of $\gamma_d$ gives
\begin{equation*}
 \mathcal V[QR(X),Y]=\mathcal V[R(X),Y],\qquad
 \mathbb D_f(P_{\mathcal S}^{QR}\|\gamma_d)
 =\mathbb D_f(P_{\mathcal S}^{R}\|\gamma_d),\qquad
\mathcal L(QR)=\mathcal L(R).
\end{equation*}
For the divergence identity, $dP_{\mathcal S}^{QR}/d\gamma_d(w)=(dP_{\mathcal S}^{R}/d\gamma_d)(Q^\top w)$ and the change of variable $w=Qz$ preserves $\gamma_d$.
Thus outputs $QR^*$ with $\|QR^*-R^*\|_{L^1(P_{\mathcal S})}>0$ violate Assumption \ref{assum3}; coordinates must be identified or aligned. Even after alignment, a linear recovery bound is stronger than quadratic growth: if fitted outputs can follow a path with
\begin{equation*}
 \|R_t-R^*\|_{L^1(P_{\mathcal S})}\geq ct,\qquad
 0<\mathcal L(R_t)-\mathcal L(R^*)\leq Ct^2,
\end{equation*}
then $\|R_t-R^*\|_{L^1}/[\mathcal L(R_t)-\mathcal L(R^*)]\geq c/(Ct)\to\infty$. Theorem \ref{Th5} and Lemma \ref{lem:latent-oracle} do not use this recovery condition.
\end{remark}

\begin{assumption}\label{ass:explicit-rates}
Use the cubes of Assumption \ref{assum5}. All constants below are independent of $n_R,m,n$ and the training realization.
\begin{enumerate}[(i)]
\item
\begin{equation*}
\begin{gathered}
 P_k(X\in K_X)=1,\qquad k\in\{\mathcal S,\mathcal T\},\\
 R(K_X)\subseteq K_U\quad(R\in\mathcal F_{n_R}\cup\{R^*\}),
 \qquad\|h^*\|_{\infty,K_U}\leq\mathcal B.
\end{gathered}
\end{equation*}
Almost surely, $\log r_{\hat R}$ has a $\beta_{\ddot\Phi}$-H\"older extension to $K_U$ with a fixed radius and absolute bound $M$.
\item Choose $\mathcal H$, $\ddot\Phi_M$ and $\mathcal F_{n_R}$ by Definition \ref{ass:downstream-sieves}, with respective $(q,\beta,t)$ values
\begin{equation*}
 (d,\beta_{\mathcal H},n),\qquad
 (d,\beta_{\ddot\Phi},m),\qquad
 (p,\beta_{\mathcal F},n_R).
\end{equation*}
The representation class may be restricted for alignment and finite divergence, subject to
\begin{equation*}
 \delta_F:=\inf_{R\in\mathcal F_{n_R}}
 \|R-R^*\|_{\infty,K_X}
 \leq Cn_R^{-\beta_{\mathcal F}/(2\beta_{\mathcal F}+p)}.
\end{equation*}
\item Stage 1 uses \eqref{eq:rep-training} with $n_R$ labeled observations and an independent reference sample $W_1,\ldots,W_{n_R}\overset{\mathrm{i.i.d.}}{\sim}\gamma_d$. Require
\begin{equation*}
\begin{gathered}
 R^*\in\arg\min_{R\in\mathcal F_{n_R}\cup\{R^*\}}\mathcal L(R),\\
 \mathbb D_f(P_{\mathcal S}^R\|\gamma_d)<\infty,
 \qquad R\in\mathcal F_{n_R}\cup\{R^*\},
\end{gathered}
\end{equation*}
and Assumption \ref{assum3} holds almost surely. For every $D\in\mathcal D_{n_R}$ and $u,v\in\mathbb R^d$,
\begin{equation*}
 \|D\|_{\infty,\mathbb R^d}\leq A_D,\qquad
 |D(u)-D(v)|\leq K_D\|u-v\|.
\end{equation*}
The conjugate $f^*$ in \eqref{conjugate-f} is bounded and Lipschitz on $[-A_D,A_D]$, and
\begin{equation*}
 \mathcal D_{\mathcal D_{n_R}}+1
 \leq C(\mathcal D_{\mathcal F_{n_R}}+1),\qquad
 \mathcal S_{\mathcal D_{n_R}}\leq C\mathcal S_{\mathcal F_{n_R}}.
\end{equation*}
The variational approximation bias satisfies
\begin{equation}\label{eq:disc-bias}
\begin{split}
 0\leq\varepsilon_{\mathcal D,n_R}:={}&
 \sup_{R\in\mathcal F_{n_R}\cup\{R^*\}}
 \bigg[\mathbb D_f(P_{\mathcal S}^R\|\gamma_d)\\
 &-\sup_{D\in\mathcal D_{n_R}}\big\{\mathbb E_{P_{\mathcal S}}D(R(X))
 -\mathbb E_{\gamma_d}f^*(D(W))\big\}\bigg]<\infty.
\end{split}
\end{equation}
\end{enumerate}
\end{assumption}

\begin{theorem}\label{thm:representation-recovery}
Suppose Assumptions \ref{ass2.0}--\ref{ass2.1}, \ref{assum1}, \ref{assum3}, \ref{assum5} and \ref{ass:explicit-rates} hold. Use the independent-fold version of Algorithm \ref{alg:al2}, with $n_R\geq4$, $m,n\geq2$. Choose the clipped fully connected ReLU predictor and domain-score classes of Definition \ref{ass:downstream-sieves}, and the admissible fully connected ReLU representation class of Assumption \ref{ass:explicit-rates}(ii), with respective depths and widths
\begin{equation*}
\begin{aligned}
 \mathcal D_{\mathcal H},\mathcal W_{\mathcal H}
 &\asymp n^{\frac{d}{8\beta_{\mathcal H}+4d}}\log(en),\\
 \mathcal D_{\ddot\Phi_M},\mathcal W_{\ddot\Phi_M}
 &\asymp m^{\frac{d}{8\beta_{\ddot\Phi}+4d}}\log(em),\\
 \mathcal D_{\mathcal F_{n_R}},\mathcal W_{\mathcal F_{n_R}}
 &\asymp n_R^{\frac{p}{8\beta_{\mathcal F}+4p}}\log(en_R).
\end{aligned}
\end{equation*}
The representation parameter budget is $\mathcal S_{\mathcal F_{n_R}}=O(n_R^{3p/(8\beta_{\mathcal F}+4p)}\log^3(en_R))$. Choose the fully connected discriminator class $\mathcal D_{n_R}$ with $\mathcal D_{\mathcal D_{n_R}}+1\leq C(\mathcal D_{\mathcal F_{n_R}}+1)$ and $\mathcal S_{\mathcal D_{n_R}}\leq C\mathcal S_{\mathcal F_{n_R}}$.
The discriminator bounds and bias are those of Assumption \ref{ass:explicit-rates}(iii). Then
\begin{equation}\label{eq:complete-polynomial-rates}
\begin{split}
 &\mathbb E\!\left[\mathcal R_{\mathcal T}(\hat h\circ\hat R)
                 -\mathcal R_{\mathcal T}(h^*\circ R^*)\right]
 \leq \widetilde O\!\left(
 n^{-\frac{\beta_{\mathcal H}}{2\beta_{\mathcal H}+d}}
 +m^{-\frac{\beta_{\ddot\Phi}}{2\beta_{\ddot\Phi}+d}}
 +n_R^{-\frac{(\beta_{\mathcal H}\wedge1)\beta_{\mathcal F}}{2\beta_{\mathcal F}+p}}
 \right)
 +O\!\left((\lambda\varepsilon_{\mathcal D,n_R})^{\beta_{\mathcal H}\wedge1}\right).
\end{split}
\end{equation}
The implied constants depend only on the fixed dimensions, radii, score bounds and recovery constants, and are independent of the three fold sizes. The first term combines prediction estimation and approximation, the second estimates the ratio induced by $\hat R$, and the remaining terms control representation error and variational bias. The representation contribution still depends on $p$. Proof and intermediate finite-sample bounds are given in Appendices \ref{pfth6}--\ref{app:explicit-rates}.
\end{theorem}

\begin{proposition}\label{prop:original-weights-recovery}
Use the independent-fold version of Algorithm \ref{alg:al1}, with $n_R\geq4$, $m,n\geq2$ and $M\geq\max\{|\log c_r|,|\log C_r|\}$. Assume Assumptions \ref{ass2.0}--\ref{ass2.1}, \ref{assum1}, \ref{assum3} and \ref{assum5}, the cube coverage and bounded extension in Assumption \ref{ass:explicit-rates}(i), its predictor and representation sieve choices in (ii), its representation-training conditions in (iii), and part (i) of Assumption \ref{ass:comparison-rates}. Use the predictor architecture stated in Theorem \ref{Th5}, the representation and discriminator architectures stated in Theorem \ref{thm:representation-recovery}, and the raw-score architecture stated in Proposition \ref{th4}. Then
\begin{equation}\label{eq:comparison-representation-rate}
 \mathbb E\!\left[\mathcal R_{\mathcal T}(\hat h\circ\hat R)
                 -\mathcal R_{\mathcal T}(h^*\circ R^*)\right]
 \leq\widetilde O\!\left(
 n^{-\frac{\beta_{\mathcal H}}{2\beta_{\mathcal H}+d}}
 +m^{-\frac{\beta_\Phi}{2\beta_\Phi+p}}
 +n_R^{-\frac{(\beta_{\mathcal H}\wedge1)\beta_{\mathcal F}}{2\beta_{\mathcal F}+p}}
 \right)
 +O\!\left((\lambda\varepsilon_{\mathcal D,n_R})^{\beta_{\mathcal H}\wedge1}\right).
\end{equation}
Original-space weighting gives the exact target-risk identity before representation. The proof in Appendix \ref{app:comparison-proofs} tracks the logarithmic factors.
\end{proposition}

\begin{corollary}\label{cor:recover-latent-rate}
Under Theorem \ref{thm:representation-recovery}, suppose
\begin{equation}\label{eq:latent-rate-allocation}
\begin{split}
 m\gtrsim n^{\frac{\beta_{\mathcal H}(2\beta_{\ddot\Phi}+d)}{\beta_{\ddot\Phi}(2\beta_{\mathcal H}+d)}},\qquad
 n_R\gtrsim n^{\frac{\beta_{\mathcal H}(2\beta_{\mathcal F}+p)}{(\beta_{\mathcal H}\wedge1)\beta_{\mathcal F}(2\beta_{\mathcal H}+d)}},\qquad
 (\lambda\varepsilon_{\mathcal D,n_R})^{\beta_{\mathcal H}\wedge1}
 \lesssim n^{-\frac{\beta_{\mathcal H}}{2\beta_{\mathcal H}+d}}\log^3(en).
\end{split}
\end{equation}
Then
\begin{equation}\label{eq:recovered-latent-rate}
 \mathbb E\!\left[\mathcal R_{\mathcal T}(\hat h\circ\hat R)
                 -\mathcal R_{\mathcal T}(h^*\circ R^*)\right]
 \leq \widetilde O\!\left(n^{-\frac{\beta_{\mathcal H}}{2\beta_{\mathcal H}+d}}\right).
\end{equation}
For $n_R\asymp n$, the sufficient representation condition is
$(\beta_{\mathcal H}\wedge1)\beta_{\mathcal F}/(2\beta_{\mathcal F}+p)
\geq\beta_{\mathcal H}/(2\beta_{\mathcal H}+d)$.
When $\beta_{\mathcal H}\geq1$, it reduces to the dimension condition
\begin{equation}\label{eq:recovered-dimension-condition}
 d\geq p\,\beta_{\mathcal H}/\beta_{\mathcal F}.
\end{equation}
For $\beta_{\mathcal H}<1$, \eqref{eq:recovered-dimension-condition} alone is insufficient. The size $n$ counts the prediction fold; the labeled source requirement is $n+n_R$. The bias condition is substantive and fails asymptotically for compact-support KL with fixed $\lambda>0$ and a fixed global discriminator bound. In that setting, Theorem \ref{thm:representation-recovery} retains its bias term. The proof is in Appendix \ref{app:sufficiency}.
\end{corollary}

\subsection{Network Estimates for Downstream Prediction and Representation Recovery}\label{app:explicit-rates}

We derive the approximation and estimation bounds for the network classes used in the risk analysis.

For $t\geq2$, set
\begin{equation}\label{eq:chi-rate}
 \chi_t(v):=\min\left\{1,\sqrt{\frac{(v+1)\log(et)}{t}}\right\},\qquad
 \mathsf r_t(q,\beta):=t^{-\beta/(2\beta+q)}.
\end{equation}
Use the fully connected architecture quantities of Definition \ref{ass:downstream-sieves}, with $1\leq\mathcal D_J+1\leq\mathcal S_J$. There is a universal constant $c_0$ such that we may use the capacity bound
\begin{equation}\label{eq:network-capacity}
 v_J:=c_0\mathcal S_J(\mathcal D_J+1)\log(e\mathcal S_J)
 \quad\text{as an upper bound for }Pdim(J).
\end{equation}
For vector-valued sieves, $v_{\mathcal F_{n_R}}$ bounds each scalar coordinate class. The capacity and entropy proofs are given below. Constants below may depend on fixed dimensions, cube sizes, smoothness radii, $\mathcal B,M,c_r,C_r$, the fixed penalty coefficient $\lambda$, and the stated recovery and discriminator bounds; they do not depend on $n,m,n_R$. Dependence on dimension is displayed in the exponents, and dimension-dependent constants are not claimed uniform as $p,d$ grow.

For a bounded domain-score sieve $J\subset[-M,M]$ and any $t$ inputs $v_1,\ldots,v_t$, write $\|f\|_{\infty,t}:=\max_{i\leq t}|f(v_i)|$. Its empirical covering number obeys
\begin{equation}\label{eq:domain-entropy}
 \log\mathcal N(\delta,\|\cdot\|_{\infty,t},J)
 \leq v_J\log\{Ct(1+M/\delta)\},\qquad0<\delta\leq M.
\end{equation}
The same bound applies to raw-space score sieves and to restrictions of these classes, without a parameter-magnitude bound. These empirical covers are used through symmetrization and localization below.

\begin{lemma}\label{prop:architecture-rate}
Under the independence and score bounds of Lemma \ref{lem:latent-oracle},
\begin{equation}\label{eq:complexity-expanded}
 \mathbb E\mathfrak C_{n,A}\leq C e^M(L_{\mathcal B}+\mathcal B)\chi_n(v_{\mathcal H}).
\end{equation}
For learned-space domain-score ERM, set $J=\ddot\Phi_M$, $\phi_0=\log r_{\hat R}$ and
$A_J:=\inf_{\phi\in J}\|\phi-\phi_0\|_{L^2(P_\mu^{\hat R})}$, conditional on the independent representation fold. Then
\begin{equation}\label{eq:ratio-slow-rate}
 \mathbb E\epsilon_{r,A}\leq C_M\{\sqrt{\mathbb E A_J^2}+\chi_m(v_J)^{1/2}\}.
\end{equation}
For fixed $v_J$, the stochastic term is $O((\log(em)/m)^{1/4})$.
\end{lemma}

For the representation-recovery extension, evaluate $\delta_H$ from \eqref{eq:representation-comparison} on $K_U$, use $\delta_F$ from Assumption \ref{ass:explicit-rates}(ii), and define
\begin{equation*}
 \delta_A(R):=\inf_{\phi\in\ddot\Phi_M}\|\phi-\log r_R\|_{\infty,K_U}.
\end{equation*}
Use the Euclidean norm for representation outputs. In this subsection put $\nu:=\beta_{\mathcal H}\wedge1$, and let $C_h$ be the corresponding H\"older modulus of $h^*$ on $K_U$. Let $v_{DF}$ bound the pseudodimension of $\{D\circ R\}$ by applying \eqref{eq:network-capacity} to the composed architecture. Define
\begin{equation}\label{eq:representation-rate-budget}
 \rho_{n_R}:=C_{\rm id}\left[(36+\lambda K_D)\delta_F
 +C\{\chi_{n_R}(dv_{\mathcal F_{n_R}})+\lambda\chi_{n_R}(v_{DF})+\lambda\chi_{n_R}(v_{\mathcal D_{n_R}})\}
 +3\lambda\varepsilon_{\mathcal D,n_R}\right],
\end{equation}
and, for a domain-score class $J$,
\begin{equation*}
 T_J(m):=\frac{(v_J+1)\log(em)}{m}.
\end{equation*}
For the representation proof it is convenient to write
\begin{equation*}
 D_{n_R}^{\rm pop}(R):=\sup_{D\in\mathcal D_{n_R}}
 \{\mathbb E_{P_{\mathcal S}}D(R(X))-\mathbb E_{\gamma_d}f^*(D(W))\}.
\end{equation*}
Thus \eqref{eq:disc-bias} is the uniform difference between the full divergence and this restricted variational objective.
Let $K_f$ be a fixed bound for both $|f^*|$ and its Lipschitz constant on $[-A_D,A_D]$, whose existence is assumed in Assumption \ref{ass:explicit-rates}(iii).

\begin{lemma}\label{thm:expanded-rates}
Under Theorem \ref{thm:representation-recovery}'s conditions, before balancing the network sizes,
\begin{equation}\label{eq:all-terms-expanded}
\begin{split}
 \mathbb E\epsilon_{r,A}&\leq C_M\{\sqrt{\mathbb E\delta_A(\hat R)^2}+\sqrt{T_{\ddot\Phi_M}(m)}\},\\
 \mathbb E\|\hat R-R^*\|_{L^1(P_{\mathcal S})}&\leq\rho_{n_R},\qquad
 \mathbb E b_{\hat R}\leq\tfrac12C_rC_h\rho_{n_R}^{\nu},\qquad
 \mathbb E a_{\hat R}\leq\delta_H+C_rC_h\rho_{n_R}^{\nu}.
\end{split}
\end{equation}
Consequently,
\begin{equation}\label{eq:expanded-oracles}
 \mathbb E\mathcal E_{\mathcal T}(\hat h,\hat R)
 \leq C\left[\chi_n(v_{\mathcal H})+\delta_H+\sqrt{\mathbb E\delta_A(\hat R)^2}
                    +\sqrt{T_{\ddot\Phi_M}(m)}+\rho_{n_R}^{\nu}\right].
\end{equation}
The representation inequalities do not depend on how the density ratio is estimated and also apply to Method \eqref{MeB}.
\end{lemma}

\begin{lemma}\label{cor:complete-rates}
For the sieve choices in Theorem \ref{thm:representation-recovery},
\begin{equation}\label{eq:individual-polynomial-rates}
\begin{split}
 \mathbb E\mathfrak C_{n,A},\ \delta_H
   &=\widetilde O(\mathsf r_n(d,\beta_{\mathcal H})),\qquad
 \mathbb E\epsilon_{r,A}=\widetilde O(\mathsf r_m(d,\beta_{\ddot\Phi})),\\
 \mathbb E b_{\hat R}&=\widetilde O(\mathsf r_{n_R}(p,\beta_{\mathcal F})^{\nu})
                 +O((\lambda\varepsilon_{\mathcal D,n_R})^{\nu}),\\
 \mathbb E a_{\hat R}&=\widetilde O(\mathsf r_n(d,\beta_{\mathcal H})
                      +\mathsf r_{n_R}(p,\beta_{\mathcal F})^{\nu})
                 +O((\lambda\varepsilon_{\mathcal D,n_R})^{\nu}).
\end{split}
\end{equation}

\end{lemma}

\subsubsection{Capacity, Entropy and Bounded-Class Estimates}
\begin{proof}[Capacity and entropy bounds \eqref{eq:network-capacity}--\eqref{eq:domain-entropy}]
All weights and biases of the fully connected architecture are free parameters, and output clipping is included in the budgets. Let $\Pi_J(k)$ count subgraph sign patterns on $k$ input-threshold pairs. For each fixed layer-width sequence, apply the polynomial sign-pattern bound in \cite[Theorem 7]{bartlett2019nearly} to $(x,z)\mapsto\mathbf1\{f(x)>z\}$. If depths and widths vary within the budgets, there are at most $(\mathcal S_J+1)^{\mathcal D_J+1}$ possible layer-width sequences, since each hidden width is at most $\mathcal S_J$. Multiplying by this number and absorbing it into the bound gives
\begin{equation*}
 \Pi_J(k)\leq(Ck\mathcal S_J)^{C\mathcal S_J(\mathcal D_J+1)},
 \qquad k\geq\mathcal S_J(\mathcal D_J+1).
\end{equation*}
If $k$ points are pseudo-shattered, then
\begin{equation*}
 k\log2\leq C\mathcal S_J(\mathcal D_J+1)\log(Ck\mathcal S_J).
\end{equation*}
Put $x=k/[\mathcal S_J(\mathcal D_J+1)]$. Since $\mathcal D_J+1\leq\mathcal S_J$, the last inequality implies $x\leq C_1\log x+C_2\log(e\mathcal S_J)$. Using $\log x\leq x/(2C_1)+\log(2C_1)-1$ yields $x\leq C\log(e\mathcal S_J)$, and hence \eqref{eq:network-capacity}. The case $k<\mathcal S_J(\mathcal D_J+1)$ satisfies the same bound. Restrictions and scalar coordinate projections cannot increase capacity.

For \eqref{eq:domain-entropy}, quantize $[-M,M]$ into intervals of length at most $\delta$. The evaluation bins on $t$ inputs are determined by at most $t\lceil2M/\delta\rceil$ input-threshold comparisons. Sauer's bound gives at most $\{Ct(1+M/\delta)\}^{v_J}$ bin patterns. Choose an evaluation vector from each occupied bin pattern; vectors in the same pattern differ by at most $\delta$ in $\|\cdot\|_{\infty,t}$. This proves the empirical covering bound, including for restricted sieves by choosing representatives in those sieves.
\end{proof}

\begin{proof}[Proof of Lemma \ref{prop:architecture-rate}]
Let $\mathcal J\subset[-K,K]$, $K>0$, have pseudodimension at most $v$. On $t$ fixed inputs, a mesh-$\eta$ quantization uses at most $t(1+2K/\eta)$ input-threshold comparisons. Sauer's bound therefore gives an empirical sup-norm cover of cardinality
\begin{equation*}
 N_\eta\leq\{et(1+2K/\eta)\}^v\quad(v\geq1),\qquad N_\eta=1\quad(v=0).
\end{equation*}
For its representative vectors $a_j\in[-K,K]^t$ and $u>0$,
\begin{equation*}
\begin{split}
 \mathbb E_\zeta\exp\left(\frac u t\sum_i\zeta_i a_{j,i}\right)
 &\leq\exp\left(\frac{u^2K^2}{2t}\right),\\
 \mathbb E_\zeta\max_{j\leq N_\eta}
 \left|\frac1t\sum_i\zeta_i a_{j,i}\right|
 &\leq\frac{\log(2N_\eta)}u+\frac{uK^2}{2t}.
\end{split}
\end{equation*}
Minimizing over $u$ and adding the quantization error proves
\begin{equation}\label{eq:counting-rademacher}
 \mathbb E_\zeta\sup_{f\in\mathcal J}
 \left|\frac1t\sum_i\zeta_i f(x_i)\right|
 \leq\eta+K\sqrt{\frac{2\{\log2+v\log(et(1+2K/\eta))\}}t}
 \leq CK\chi_t(v),
\end{equation}
where the last inequality uses $\eta=K/t$ and the trivial bound $K$.

Conditional on $\mathscr G$, an empirical score cover for $h(\hat R(X_i))$ at mesh $\mathcal B/n$ induces a weighted-loss cover at mesh $e^M\mathcal B/n$:
\begin{equation*}
 |\hat r_d(\hat R(X_i))\ell(h(\hat R(X_i)),Y_i)
 -\hat r_d(\hat R(X_i))\ell(\bar h(\hat R(X_i)),Y_i)|
 \leq e^M|h(\hat R(X_i))-\bar h(\hat R(X_i))|.
\end{equation*}
The loss envelope is $e^ML_{\mathcal B}$. The preceding finite-cover calculation gives
\begin{equation*}
 \mathfrak C_{n,A}\leq
 \frac{e^M\mathcal B}n+
 e^ML_{\mathcal B}\sqrt{\frac{2\{\log2+v_{\mathcal H}\log(Cn^2)\}}n},
 \qquad \mathfrak C_{n,A}\leq e^ML_{\mathcal B}.
\end{equation*}
These inequalities imply \eqref{eq:complexity-expanded} after averaging over $\mathscr G$.

For the slow domain bound, fix the independent representation fold and put $\phi_0=\log r_{\hat R}$. ERM and symmetrization imply
\begin{equation*}
\begin{split}
 \mathbb E[q_A\mid\hat R]
 &\leq\inf_{\phi\in J}
       \{\mathcal L_{\mu,\hat R}(\phi)-\mathcal L_{\mu,\hat R}(\phi_0)\}
 +2\mathbb E\sup_{\phi\in J}
       |\mathcal L_{\mu,\hat R}(\phi)-\widehat{\mathcal L}_{\mu,\hat R}(\phi)|\\
 &\leq A_J^2/8+CL_M\chi_m(v_J),\qquad L_M=\log(1+e^M).
\end{split}
\end{equation*}
Here Lemma \ref{lem2.1} gives the approximation term; \eqref{eq:counting-rademacher} applies to a score cover followed by the 1-Lipschitz loss. Equation \eqref{eq:ratio-root-conversion} and $\sqrt{x+y}\leq\sqrt x+\sqrt y$ give \eqref{eq:ratio-slow-rate}.
\end{proof}

\subsubsection{A Bounded-Logistic ERM Inequality}
For the following lemma, let $O=(V,Z)$ have the balanced domain law, with optimal score $\phi_0$, $|\phi_0|\leq M$, and let $J\subset[-M,M]$. Define
\begin{equation*}
 F_\phi(O)=\log(1+e^{-Z\phi(V)})-\log(1+e^{-Z\phi_0(V)}),\quad
 q(\phi)=\mathbb E F_\phi,\quad
 \widehat P_mF=\frac1m\sum_{i=1}^mF(O_i).
\end{equation*}
For learned features, $(V,\phi_0)=(\hat R(X),\log r_{\hat R})$; for raw features, $(V,\phi_0)=(X,\phi^*)$.

\begin{lemma}\label{lem:fast-domain-erm}
Let $\hat\phi$ be ERM over $J$, with $\operatorname{Pdim}(J)\leq v_J$, and use $T_J(m)=(v_J+1)\log(em)/m$. Then
\begin{equation}\label{eq:fast-domain-erm}
\begin{split}
 \mathbb E q(\hat\phi)
 &\leq2\inf_{\phi\in J}q(\phi)+C_MT_J(m)\\
 &\leq\frac14\inf_{\phi\in J}\|\phi-\phi_0\|_\infty^2+C_MT_J(m).
\end{split}
\end{equation}
For learned features the assertion is conditional on the independent representation fold. Neither convexity of $J$ nor membership $\phi_0\in J$ is required.
\end{lemma}
\begin{proof}
If $M=0$, all scores vanish and the assertion is immediate. Suppose $M>0$. Score Lipschitzness and conditional logistic curvature give
\begin{equation}\label{eq:logistic-bernstein}
 q(\phi)\geq0,\qquad |F_\phi|\leq2M,\qquad
 \mathbb E[F_\phi(O)^2]\leq\mathbb E[(\phi(V)-\phi_0(V))^2]
                  \leq\frac2{\kappa_M}q(\phi).
\end{equation}
Put $a=T_J(m)$ and, for $r\geq0$, define only within this proof
\begin{equation*}
\begin{gathered}
 \mathcal E_r:=\{F_\phi:\phi\in J,\ \mathbb E[F_\phi(O)^2]\leq r\}\cup\{0\},\\
 \mf{R}_m(\mathcal E_r\cup(-\mathcal E_r))=\mathbb E\mathbb E_\zeta\sup_{F\in\mathcal E_r}
 \left|\frac1m\sum_{i=1}^m\zeta_iF(O_i)\right|.
\end{gathered}
\end{equation*}
An empirical score cover at mesh $M/m$ induces a loss-difference cover at that mesh. By \eqref{eq:domain-entropy}, representatives chosen in $\mathcal E_r$ have cardinality $N$ satisfying $\log(2N)\leq C_M(v_J+1)\log(em)$. Their empirical squared norms are at most $\sup_{F\in\mathcal E_r}\widehat P_mF^2$. The finite-cover exponential calculation used in \eqref{eq:counting-rademacher}, now using these norms, gives
\begin{equation*}
 \mf{R}_m(\mathcal E_r\cup(-\mathcal E_r))\leq\frac{C_M}{m}
 +C_M\sqrt a\left(\mathbb E\sup_{F\in\mathcal E_r}\widehat P_mF^2\right)^{1/2}.
\end{equation*}
The representatives may depend on $O_1,\ldots,O_m$, but are chosen before the independent signs $\zeta_i$. Thus this conditional Rademacher calculation is valid. Symmetrization and contraction for $u\mapsto u^2$, which is $4M$-Lipschitz on $[-2M,2M]$, yield
\begin{equation*}
\begin{split}
 \mathbb E\sup_{F\in\mathcal E_r}\widehat P_mF^2
 &\leq r+\mathbb E\sup_{F\in\mathcal E_r}\left|\widehat P_mF^2-\mathbb E[F(O)^2]\right|\\
 &\leq r+2\mathbb E\mathbb E_\zeta\sup_{F\in\mathcal E_r}
       \left|\frac1m\sum_i\zeta_iF(O_i)^2\right|\\
 &\leq r+C_M\mf{R}_m(\mathcal E_r\cup(-\mathcal E_r)).
\end{split}
\end{equation*}
Consequently, since $a\geq1/m$,
\begin{equation*}
\begin{split}
 \mf{R}_m(\mathcal E_r\cup(-\mathcal E_r))
 &\leq C_M\{a+\sqrt{ar}+\sqrt{a\mf{R}_m(\mathcal E_r\cup(-\mathcal E_r))}\}\\
 &\leq\tfrac12\mf{R}_m(\mathcal E_r\cup(-\mathcal E_r))+C_M\{a+\sqrt{ar}\},
\end{split}
\end{equation*}
and hence $\mf{R}_m(\mathcal E_r\cup(-\mathcal E_r))\leq C_M\{\sqrt{ar}+a\}$.

To apply \cite[Theorem 3.3, part 1]{bartlett2005local}, take the excess-loss class $\{F_\phi:\phi\in J\}\cup\{0\}$ and the functional $T(F)=\mathbb E[F(O)^2]$. Equation \eqref{eq:logistic-bernstein} verifies
\begin{equation*}
 \operatorname{Var}(F(O))\leq T(F)\leq(2/\kappa_M)\mathbb E[F(O)].
\end{equation*}
Choose a sufficiently large $C_M$ in $\psi(r)=C_M(\sqrt{ar}+a)$. This function is nonnegative and nondecreasing, $\psi(r)/\sqrt r$ is nonincreasing, and it dominates $(2/\kappa_M)\mf{R}_m(\mathcal E_r\cup(-\mathcal E_r))$. Its fixed point $r_\psi$ satisfies $r_\psi\leq C_Ma$: indeed, with $u=\sqrt{r_\psi/a}$, the equation becomes $u^2=C_M(u+1)$, whose positive solution is bounded by a constant depending only on $M$. The cited theorem with $K=2$ therefore gives, for each $x>0$, with probability at least $1-e^{-x}$,
\begin{equation*}
 q(\phi)\leq2\widehat P_mF_\phi+C_M\{a+x/m\}\qquad\text{for all }\phi\in J.
\end{equation*}
For any deterministic $\bar\phi\in J$, ERM gives $\widehat P_mF_{\hat\phi}\leq \widehat P_mF_{\bar\phi}$. Thus, for $W=q(\hat\phi)-2\widehat P_mF_{\bar\phi}-C_Ma$,
\begin{equation*}
 \mathbb P\{W>C_Mx/m\}\leq e^{-x},\qquad
 \mathbb EW\leq\mathbb EW_+
 =\int_0^\infty\mathbb P\{W>u\}\,du\leq C_M/m.
\end{equation*}
Taking expectations, using $\mathbb E\widehat P_mF_{\bar\phi}=q(\bar\phi)$ and infimizing over $\bar\phi$ proves the first inequality of \eqref{eq:fast-domain-erm}. Lemma \ref{lem2.1} gives $q(\phi)\leq\|\phi-\phi_0\|_\infty^2/8$, proving the second. Infima not attained are handled by comparators within an arbitrary positive tolerance. All arguments hold conditional on the independent representation fold, with constants uniform in that fold.
\end{proof}

\subsubsection{Uniform Representation Bounds and Variational Bias}
\begin{proof}[Representation bounds in Lemma \ref{thm:expanded-rates}]
Let $D_U=\operatorname{diam}(K_U)$ and $O_i=(X_i,Y_i)$. Set
\begin{equation*}
 K_R(O_1,\ldots,O_4)=K((R(X_1),Y_1),\ldots,(R(X_4),Y_4)).
\end{equation*}
The ordered-pair kernel in Appendix \ref{Au_results} has absolute coefficient totals $3,9,6$. Since $|Y_i-Y_j|\leq1$,
\begin{equation}\label{eq:kernel-envelope-lipschitz}
 |K_R|\leq(3+9+6)D_U=18D_U,\qquad
 |K_R-K_{R'}|\leq36\max_{i\leq4}\|R(X_i)-R'(X_i)\|.
\end{equation}
The second bound follows termwise from
$|\|R(X_i)-R(X_j)\|-\|R'(X_i)-R'(X_j)\||
\leq2\max_{i\leq4}\|R(X_i)-R'(X_i)\|$.

Put $b=\lfloor n_R/4\rfloor$. For each permutation $\pi$ define the block average
\begin{equation*}
 B_\pi(R)=\frac1b\sum_{j=1}^b
 K_R(O_{\pi(4j-3)},\ldots,O_{\pi(4j)}).
\end{equation*}
Each four-element subset appears $b\,4!(n_R-4)!$ times in the sum over permutations. Kernel symmetry and unbiasedness therefore imply
\begin{equation*}
 \widehat{\mathcal V}_{n_R}[R(X),Y]
 =\frac1{n_R!}\sum_\pi B_\pi(R),\qquad
 \mathbb E B_\pi(R)=\mathcal V[R(X),Y].
\end{equation*}
Convexity of the supremum and symmetrization over the $b$ independent blocks give
\begin{equation*}
\begin{split}
 \mathbb E\sup_R|\widehat{\mathcal V}_{n_R}-\mathcal V|
 &\leq\frac1{n_R!}\sum_\pi\mathbb E\sup_R|B_\pi(R)-\mathbb EB_\pi(R)|\\
 &\leq2\mathbb E\mathbb E_\zeta\sup_R
 \left|\frac1b\sum_{j=1}^b\zeta_jK_R(O_{4j-3},\ldots,O_{4j})\right|.
\end{split}
\end{equation*}
Quantize each representation coordinate on the $4b$ inputs at mesh $\eta/\sqrt d$. Sauer's bound gives a joint evaluation cover with
\begin{equation*}
 \log N_\eta\leq d v_{\mathcal F_{n_R}}\log(Cb/\eta),\qquad0<\eta\leq1.
\end{equation*}
The Euclidean error on each input is at most $\eta$, so the kernel cover error is at most $36\eta$. The finite-cover calculation in \eqref{eq:counting-rademacher} consequently yields
\begin{equation*}
 \mathbb E_\zeta\sup_R\left|\frac1b\sum_j\zeta_jK_R\right|
 \leq36\eta+18D_U
       \sqrt{\frac{2\{\log2+d v_{\mathcal F_{n_R}}\log(Cb/\eta)\}}b}.
\end{equation*}
Choose $\eta=1/b$, use $b\geq n_R/7$ for $n_R\geq4$, and combine with the envelope bound. This proves
\begin{equation}\label{eq:ustat-uniform-rate}
 \mathbb E\sup_R|\widehat{\mathcal V}_{n_R}[R(X),Y]-\mathcal V[R(X),Y]|
 \leq C\chi_{n_R}(d v_{\mathcal F_{n_R}}).
\end{equation}

For the variational penalty, write $P_{n_R}^{\mathcal S}$ and $P_{n_R}^{\gamma}$ for empirical averaging over the source and reference observations within this proof. Then
\begin{equation*}
\begin{split}
 &\sup_R|\widehat{\mathbb D}_{f,n_R}(R)-D_{n_R}^{\rm pop}(R)|\\
 &\leq\sup_{R,D}|(P_{n_R}^{\mathcal S}-P_{\mathcal S})(D\circ R)|
       +\sup_D|(P_{n_R}^{\gamma}-\gamma_d)(f^*\circ D)|.
\end{split}
\end{equation*}
For the first class use envelope $A_D$ and capacity $v_{DF}$. For the second, an empirical mesh-$\eta$ cover for $D$ induces a mesh-$K_f\eta$ cover for $f^*\circ D$, with envelope $K_f$. Symmetrization and \eqref{eq:counting-rademacher} give
\begin{equation}\label{eq:variational-uniform-rate}
 \mathbb E\sup_R|\widehat{\mathbb D}_{f,n_R}(R)-D_{n_R}^{\rm pop}(R)|
 \leq C\{\chi_{n_R}(v_{DF})+\chi_{n_R}(v_{\mathcal D_{n_R}})\}.
\end{equation}
This empirical-cover bound holds for every realized reference sample, including Gaussian observations.

Within this proof put
\begin{equation*}
 B(R)=\mathbb D_f(P_{\mathcal S}^R\|\gamma_d)-D_{n_R}^{\rm pop}(R),\qquad
 0\leq B(R)\leq\varepsilon_{\mathcal D,n_R}.
\end{equation*}
The inequalities follow from Fenchel's inequality and \eqref{eq:disc-bias}. For KL, let $r=dP_{\mathcal S}^R/d\gamma_d$, using finiteness of the assumed divergence. For every admissible $D$,
\begin{equation*}
\begin{split}
 \mathbb D_{\rm KL}(P_{\mathcal S}^R\|\gamma_d)
 -\{\mathbb E_{P_{\mathcal S}}D(R(X))-\mathbb E_{\gamma_d}e^{D(W)-1}\}
 &=\int[r\log r-rD+e^{D-1}]\,d\gamma_d\\
 &\geq\int_{K_U^c}e^{D-1}\,d\gamma_d
 \geq e^{-A_D-1}\gamma_d(K_U^c).
\end{split}
\end{equation*}
Here $r=0$ on $K_U^c$, $0\log0=0$, and the integrand is nonnegative by Fenchel's inequality. Infimizing over $D$ and then taking the supremum over $R$ proves
\begin{equation}\label{eq:kl-discriminator-bias-floor}
 \varepsilon_{\mathcal D,n_R}\geq e^{-A_D-1}\gamma_d(K_U^c)>0.
\end{equation}

Define $Z_{n_R}=\sup_{R\in\mathcal F_{n_R}}|\widehat{\mathcal L}_{n_R}(R)-\mathcal L(R)|$ within this proof. Equations \eqref{eq:ustat-uniform-rate}--\eqref{eq:variational-uniform-rate} imply
\begin{equation*}
 \mathbb EZ_{n_R}\leq C\{\chi_{n_R}(dv_{\mathcal F_{n_R}})
 +\lambda\chi_{n_R}(v_{DF})+\lambda\chi_{n_R}(v_{\mathcal D_{n_R}})\}
 +\lambda\varepsilon_{\mathcal D,n_R}.
\end{equation*}
For $R_0\in\mathcal F_{n_R}$ with $\|R_0-R^*\|_{\infty,K_X}\leq\delta_F+\epsilon$, kernel Lipschitzness and cancellation of the reference term give
\begin{equation*}
\begin{split}
 |\mathcal V[R_0(X),Y]-\mathcal V[R^*(X),Y]|&\leq36(\delta_F+\epsilon),\\
 |D_{n_R}^{\rm pop}(R_0)-D_{n_R}^{\rm pop}(R^*)|&\leq K_D(\delta_F+\epsilon),\\
 \mathcal L(R_0)-\mathcal L(R^*)&\leq
 (36+\lambda K_D)(\delta_F+\epsilon)+\lambda\varepsilon_{\mathcal D,n_R}.
\end{split}
\end{equation*}
Population optimality of $R^*$, empirical optimality of $\hat R$ and the recovery condition give
\begin{equation*}
\begin{split}
 0\leq\mathcal L(\hat R)-\mathcal L(R^*)
 &=\mathcal L(\hat R)-\widehat{\mathcal L}_{n_R}(\hat R)
 +\widehat{\mathcal L}_{n_R}(\hat R)-\widehat{\mathcal L}_{n_R}(R_0)\\
 &\quad+\widehat{\mathcal L}_{n_R}(R_0)-\mathcal L(R_0)
       +\mathcal L(R_0)-\mathcal L(R^*)\\
 &\leq2Z_{n_R}+(36+\lambda K_D)(\delta_F+\epsilon)
               +\lambda\varepsilon_{\mathcal D,n_R},\\
 \mathbb E\|\hat R-R^*\|_{L^1(P_{\mathcal S})}
 &\leq C_{\rm id}\mathbb E[\mathcal L(\hat R)-\mathcal L(R^*)]
 \leq\rho_{n_R}\quad(\epsilon\downarrow0).
\end{split}
\end{equation*}
Thus the bias coefficient in \eqref{eq:representation-rate-budget} is $2\lambda+\lambda=3\lambda$.
Equations \eqref{eq:representation-comparison} and Jensen's inequality give
\begin{equation*}
 \mathbb E b_{\hat R}\leq\tfrac12C_rC_h\rho_{n_R}^\nu,\qquad
 \mathbb E a_{\hat R}\leq\delta_H+C_rC_h\rho_{n_R}^\nu.
\end{equation*}
For the ratio term, apply Lemma \ref{lem:fast-domain-erm} conditional on the representation fold. The localized ERM bound and the ratio conversion imply
\begin{equation*}
\begin{split}
 \mathbb E[q_A\mid\hat R]&\leq C_M\{\delta_A(\hat R)^2+T_{\ddot\Phi_M}(m)\},\\
 \mathbb E\epsilon_{r,A}&\leq C_M\{\sqrt{\mathbb E\delta_A(\hat R)^2}
                                      +\sqrt{T_{\ddot\Phi_M}(m)}\}.
\end{split}
\end{equation*}
These are all the assertions of \eqref{eq:all-terms-expanded}. Substitution into \eqref{eq:latent-oracle-bound} proves \eqref{eq:expanded-oracles}.
\end{proof}

\subsubsection{Uniform Approximation and Explicit Rates}
\begin{lemma}\label{lem:uniform-network-approximation}
For a fixed scalar H\"older ball of index $\beta>0$ on a fixed $q$-dimensional cube, the clipped sieve in Definition \ref{ass:downstream-sieves} satisfies
\begin{equation*}
 \sup_f\inf_{g\in J}\|g-f\|_\infty\leq C\mathsf r_t(q,\beta).
\end{equation*}
The supremum is over targets in that ball with values in the clipping interval. Coordinatewise application gives the same order for a fixed number of outputs.
\end{lemma}
\begin{proof}
After affine rescaling of the fixed input cube, apply \cite[Corollary 3.1]{jiao2023deep}. With the convention $\alpha=\lceil\beta\rceil-1$ of Definition \ref{def:hc} and H\"older radius $B_0$, for any positive integers $N_1,N_2$ there is a ReLU network $g$ such that
\begin{equation*}
\begin{split}
 \mathcal W&\leq38(\alpha+1)^2 3^q q^{\alpha+1}N_1\lceil\log_2(8N_1)\rceil,\\
 \mathcal D_{\rm net}&\leq21(\alpha+1)^2N_2\lceil\log_2(8N_2)\rceil+2q,\\
 \|g-f\|_{\infty,[0,1]^q}
 &\leq19B_0(\alpha+1)^2q^{\alpha+(\beta\vee1)/2}(N_1N_2)^{-2\beta/q}.
\end{split}
\end{equation*}
This is uniform on the entire cube. In particular, it does not require the input law to be absolutely continuous. Choose
\begin{equation*}
 N_1=N_2=\left\lceil t^{\frac{q}{4(2\beta+q)}}\right\rceil.
\end{equation*}
Since $t\geq2$ and $q,\beta$ are fixed,
\begin{equation*}
\begin{split}
 N_i\lceil\log_2(8N_i)\rceil
 &\leq Ct^{\frac{q}{4(2\beta+q)}}\log(et),\quad i=1,2,\\
 (N_1N_2)^{-2\beta/q}&\leq t^{-\beta/(2\beta+q)}.
\end{split}
\end{equation*}
These bounds give the depth, width and approximation rate in \eqref{eq:sieve-architecture}, for sufficiently large fixed $C_D,C_W$. Count every parameter position, including zero-valued ones, to obtain
\begin{equation*}
 \mathcal S\leq C(\mathcal D_{\rm net}+1)\mathcal W^2
 \leq Ct^{\frac{3q}{4(2\beta+q)}}\log^3(et).
\end{equation*}
Affine input rescaling changes only the fixed constants. If the clipping interval is $[a,b]$, then $\operatorname{clip}_{[a,b]}(z)=a+\varrho(z-a)-\varrho(z-b)$ and
\begin{equation*}
 |\operatorname{clip}_{[a,b]}(g(x))-f(x)|\leq|g(x)-f(x)|
 \quad\text{whenever }f(x)\in[a,b].
\end{equation*}
This adds a fixed number of ReLU layers, which the budgets include. A construction with some zero coefficients belongs to the fully connected class; the class itself imposes no restriction on their number. Parallelizing a fixed number of coordinate networks and clipping coordinatewise proves the vector assertion, with the Euclidean error bounded by the square root of that number times the largest coordinate error. The width constant $C_W$ includes the factor $3^q$. For the admissible representation restriction, retention of this approximation order is the explicit condition in Assumption \ref{ass:explicit-rates}(ii).
\end{proof}

\begin{proof}[Proof of Lemma \ref{cor:complete-rates} and Theorem \ref{thm:representation-recovery}]
For each of the following choices apply \eqref{eq:sieve-architecture}:
\begin{equation*}
 \begin{array}{c|ccc}
 J&q&\beta&t\\\hline
 \mathcal H&d&\beta_{\mathcal H}&n\\
 \ddot\Phi_M&d&\beta_{\ddot\Phi}&m\\
 \mathcal F_{n_R}&p&\beta_{\mathcal F}&n_R.
 \end{array}
\end{equation*}
Direct substitution into \eqref{eq:network-capacity} and the definition of $T_J$ gives
\begin{equation*}
\begin{split}
 \mathcal D_J+1,\mathcal W_J
 &\leq Ct^{\frac{q}{4(2\beta+q)}}\log(et),\qquad
 \mathcal S_J\leq Ct^{\frac{3q}{4(2\beta+q)}}\log^3(et),\\
 v_J&\leq Ct^{q/(2\beta+q)}\log^5(et),\\
 \chi_t(v_J)&\leq Ct^{-\beta/(2\beta+q)}\log^3(et),\\
 T_J(t)&\leq Ct^{-2\beta/(2\beta+q)}\log^6(et).
\end{split}
\end{equation*}
For $D\circ R$, concatenate the two architectures, including representation clipping. Keep the intermediate $d$-dimensional output as a bottleneck: replace $u=R(x)$ by $(\varrho(u),\varrho(-u))$ and the discriminator's input affine map $Au+b$ by $[A,-A](\varrho(u),\varrho(-u))+b$. Since $u=\varrho(u)-\varrho(-u)$, this computes exactly $D\circ R$. Doubling the representation's output rows and the discriminator's input columns adds only a fixed factor to their total parameter counts and one hidden layer. Thus the composed fully connected architecture has total parameter count at most $C(\mathcal S_{\mathcal F_{n_R}}+\mathcal S_{\mathcal D_{n_R}})$ and depth at most $C((\mathcal D_{\mathcal F_{n_R}}+1)+(\mathcal D_{\mathcal D_{n_R}}+1))$. Assumption \ref{ass:explicit-rates}(iii) therefore yields
\begin{equation*}
 v_{DF}+v_{\mathcal D_{n_R}}+dv_{\mathcal F_{n_R}}
 \leq Cn_R^{p/(2\beta_{\mathcal F}+p)}\log^5(en_R).
\end{equation*}
Lemma \ref{lem:uniform-network-approximation} and the retained representation approximant give
\begin{equation*}
 \delta_H\leq C\mathsf r_n(d,\beta_{\mathcal H}),\quad
 \delta_A(\hat R)\leq C\mathsf r_m(d,\beta_{\ddot\Phi})\ \text{a.s.},\quad
 \delta_F\leq C\mathsf r_{n_R}(p,\beta_{\mathcal F}).
\end{equation*}
The bound for $\delta_A(\hat R)$ uses only actual fitted-output regularity, with constants uniform in the representation fold. It follows that
\begin{equation*}
\begin{split}
 \mathbb E\mathfrak C_{n,A}&\leq C\mathsf r_n(d,\beta_{\mathcal H})\log^3(en),\\
 \mathbb E\epsilon_{r,A}&\leq C\mathsf r_m(d,\beta_{\ddot\Phi})\log^3(em),\\
 \rho_{n_R}&\leq C\{\mathsf r_{n_R}(p,\beta_{\mathcal F})\log^3(en_R)
                                     +\lambda\varepsilon_{\mathcal D,n_R}\},\\
 \rho_{n_R}^\nu&\leq C\{\mathsf r_{n_R}(p,\beta_{\mathcal F})^\nu\log^{3\nu}(en_R)
                                     +(\lambda\varepsilon_{\mathcal D,n_R})^\nu\}.
\end{split}
\end{equation*}
The last inequality uses $(x+y)^\nu\leq x^\nu+y^\nu$ for $x,y\geq0$ and $0<\nu\leq1$. Substituting into \eqref{eq:all-terms-expanded} proves every line of \eqref{eq:individual-polynomial-rates}. Substitution into Lemma \ref{lem:latent-oracle} gives
\begin{equation}\label{eq:complete-log-rates}
\begin{split}
 \mathbb E\!\left[\mathcal R_{\mathcal T}(\hat h\circ\hat R)
                 -\mathcal R_{\mathcal T}(h^*\circ R^*)\right]
 \leq C\bigg[&n^{-\frac{\beta_{\mathcal H}}{2\beta_{\mathcal H}+d}}\log^3(en)
 +m^{-\frac{\beta_{\ddot\Phi}}{2\beta_{\ddot\Phi}+d}}\log^3(em)\\
 &+n_R^{-\frac{(\beta_{\mathcal H}\wedge1)\beta_{\mathcal F}}{2\beta_{\mathcal F}+p}}
      \log^{3(\beta_{\mathcal H}\wedge1)}(en_R)
 +(\lambda\varepsilon_{\mathcal D,n_R})^{\beta_{\mathcal H}\wedge1}\bigg].
\end{split}
\end{equation}
Suppressing the displayed logarithmic factors proves \eqref{eq:complete-polynomial-rates}.
\end{proof}

\subsection{Proof of the Sufficient-Representation Theorem and Recovery Sample Allocation}\label{app:sufficiency}
\begin{proof}[Proof of Proposition \ref{prop:learned-sufficiency}]
Fix a realization of the representation-training data, reference sample and training randomness for which source sufficiency holds. Original-space covariate shift, implied by Assumptions \ref{ass2.0}--\ref{ass2.1}, gives the common class probability $\eta$ defined after \eqref{eq2.02}. Source sufficiency implies $\eta(X)=\eta_{\mathcal S}^{\hat R}(\hat R(X))$ $P_{\mathcal S}$-almost surely. Since $P_{\mathcal T}(X)\ll P_{\mathcal S}(X)$, the same identity holds under the target feature law. Therefore
\begin{equation*}
 \eta_{\mathcal T}^{\hat R}(\hat R(X))
 =\mathbb E_{P_{\mathcal T}}[\eta(X)\mid\hat R(X)]
 =\eta_{\mathcal S}^{\hat R}(\hat R(X))
 \quad P_{\mathcal T}\text{-a.s.}
\end{equation*}
Because labels are binary, this also gives target sufficiency and equality of the entire conditional label distributions. Integrating the absolute difference proves $b_{\hat R}=0$. The argument holds for almost every training realization, so the expected mismatch term in Lemma \ref{lem:latent-oracle} is zero.
\end{proof}

\begin{proof}[Proof of Theorem \ref{Th5}]
Condition first on the independent representation-training data, reference sample and training randomness. Proposition \ref{prop:learned-sufficiency} gives $b_{\hat R}=0$ and $h_{\hat R}^*\circ\hat R=h^*\circ R^*$ almost surely under both domains. The bound on $h_{\hat R}^*$ and the clipped predictor class imply the score bounds of Assumption \ref{assum4}. Consequently,
\begin{equation}\label{eq:sufficient-approximation}
\begin{split}
 a_{\hat R}
 &=\inf_{h\in\mathcal H}\{
 \mathcal R_{\mathcal T}(h\circ\hat R)
 -\mathcal R_{\mathcal T}(h_{\hat R}^*\circ\hat R)\}\\
 &\leq\inf_{h\in\mathcal H}\|h-h_{\hat R}^*\|_{\infty,K_U}
 \leq C n^{-\frac{\beta_{\mathcal H}}{2\beta_{\mathcal H}+d}}.
\end{split}
\end{equation}
The first inequality uses the logistic loss's score Lipschitz constant one, and the last follows from the uniform learned-score regularity and the predictor sieve.

The bounded-class estimate \eqref{eq:complexity-expanded}, with the architecture in Definition \ref{ass:downstream-sieves}, gives
\begin{equation*}
 \mathbb E\mathfrak C_{n,A}
 \leq C n^{-\frac{\beta_{\mathcal H}}{2\beta_{\mathcal H}+d}}\log^3(en).
\end{equation*}
For the domain classifier apply \eqref{eq:fast-domain-erm} and \eqref{eq:domain-entropy} with $J=\ddot\Phi_M$ and $\phi_0=\log r_{\hat R}$, conditional on the representation fold. The uniform learned-ratio approximation and architecture bounds yield
\begin{equation*}
 \mathbb E[q_A\mid\hat R]
 \leq C m^{-\frac{2\beta_{\ddot\Phi}}{2\beta_{\ddot\Phi}+d}}\log^6(em).
\end{equation*}
The constant is uniform in that fold, so averaging and the conversion \eqref{eq:ratio-root-conversion} give
\begin{equation*}
 \mathbb E\epsilon_{r,A}
 \leq C m^{-\frac{\beta_{\ddot\Phi}}{2\beta_{\ddot\Phi}+d}}\log^3(em).
\end{equation*}
Substituting these estimates, \eqref{eq:sufficient-approximation} and $b_{\hat R}=0$ into Lemma \ref{lem:latent-oracle} gives
\begin{equation}\label{eq:strict-sufficient-log-rate}
\begin{split}
 \mathbb E\!\left[\mathcal R_{\mathcal T}(\hat h\circ\hat R)
                 -\mathcal R_{\mathcal T}(h^*\circ R^*)\right]
 \leq C\left[
 n^{-\frac{\beta_{\mathcal H}}{2\beta_{\mathcal H}+d}}\log^3(en)
 +m^{-\frac{\beta_{\ddot\Phi}}{2\beta_{\ddot\Phi}+d}}\log^3(em)
 \right],
\end{split}
\end{equation}
Here $C$ depends only on the fixed dimensions, cubes, regularity radii, score bounds and overlap constants, and is uniform in the fold sizes and training realization. This proves \eqref{eq:strict-sufficient-rate}.
\end{proof}

\begin{proof}[Proof of Corollary \ref{cor:recover-latent-rate}]
Use the following abbreviations only within this proof:
\begin{equation*}
 \nu=\beta_{\mathcal H}\wedge1,\qquad
 a=\frac{\beta_{\mathcal H}}{2\beta_{\mathcal H}+d},\qquad
 b=\frac{\beta_{\ddot\Phi}}{2\beta_{\ddot\Phi}+d},\qquad
 c=\frac{\nu\beta_{\mathcal F}}{2\beta_{\mathcal F}+p}.
\end{equation*}
In \eqref{eq:complete-log-rates}, the four summands have orders
\begin{equation*}
 n^{-a}\log^3(en),\quad m^{-b}\log^3(em),\quad
 n_R^{-c}\log^{3\nu}(en_R),\quad(\lambda\varepsilon_{\mathcal D,n_R})^\nu.
\end{equation*}
For $u>0$ and $v\geq0$,
\begin{equation*}
 \frac{d}{dt}\{t^{-u}\log^v(et)\}
 =t^{-u-1}\log^v(et)\left[-u+\frac v{\log(et)}\right]\leq0
 \quad\text{when }\log(et)\geq v/u.
\end{equation*}
Choose fixed $c_m,c_R>0$ such that $m\geq c_mn^{a/b}$ and $n_R\geq c_Rn^{a/c}$. Evaluation at these lower bounds yields, for sufficiently large $n$,
\begin{equation*}
 m^{-b}\log^3(em)\leq C n^{-a}\log^3(en),\qquad
 n_R^{-c}\log^{3\nu}(en_R)\leq C n^{-a}\log^{3\nu}(en).
\end{equation*}
Since $\nu\leq1$, both terms are bounded by $C n^{-a}\log^3(en)$. The bias condition therefore gives
\begin{equation*}
 \mathbb E\!\left[\mathcal R_{\mathcal T}(\hat h\circ\hat R)
                 -\mathcal R_{\mathcal T}(h^*\circ R^*)\right]
 \leq C n^{-\frac{\beta_{\mathcal H}}{2\beta_{\mathcal H}+d}}\log^3(en).
\end{equation*}
This proves \eqref{eq:recovered-latent-rate}; enlarging the constant covers the finitely many smaller sample sizes. If $n_R\asymp n$, $c\geq a$ ensures the representation term has the required order. For $\beta_{\mathcal H}\geq1$, $\nu=1$ and this condition is equivalent to
\begin{equation*}
 \beta_{\mathcal F}(2\beta_{\mathcal H}+d)\geq\beta_{\mathcal H}(2\beta_{\mathcal F}+p)
 \quad\Longleftrightarrow\quad d\geq p\beta_{\mathcal H}/\beta_{\mathcal F}.
\end{equation*}
For $0<\beta_{\mathcal H}<1$, $\nu=\beta_{\mathcal H}$ and
\begin{equation*}
 c\geq a
 \quad\Longleftrightarrow\quad
 \beta_{\mathcal F}(2\beta_{\mathcal H}+d)\geq2\beta_{\mathcal F}+p.
\end{equation*}
For compact-support KL with fixed $\lambda>0$ and $A_D$, \eqref{eq:kl-discriminator-bias-floor} gives
\begin{equation*}
(\lambda\varepsilon_{\mathcal D,n_R})^\nu\geq[\lambda e^{-A_D-1}\gamma_d(K_U^c)]^\nu>0,
\end{equation*}
whereas $n^{-a}\log^3(en)\to0$. The bias condition of the corollary then cannot hold for all sufficiently large $n$.

\end{proof}

The sufficient-representation conditions admit a concrete family with covariate shift in both relevant and nuisance features. Let $X=(U,V)\in[0,1]^d\times[0,1]^{p-d}$, let the two feature laws be $P_k^U\otimes P_k^V$ with smooth densities bounded above and away from zero, and let $P_k(Y=1\mid U,V)=\sigma(h^*(U))$ in both domains. If Stage 1 outputs $\hat R(X)=T(U)$, where $T$ and $T^{-1}$ have uniform smoothness bounds and map the latent cube bijectively to itself, then
\begin{equation}\label{eq:sufficient-family}
 h_{\hat R}^*=h^*\circ T^{-1},\qquad
 r_{\hat R}=\frac{dP_{\mathcal T}^U}{dP_{\mathcal S}^U}\circ T^{-1},\qquad
 r^*(X)=\frac{dP_{\mathcal T}^U}{dP_{\mathcal S}^U}(U)
          \frac{dP_{\mathcal T}^V}{dP_{\mathcal S}^V}(V).
\end{equation}
With sufficiently smooth factors this family satisfies the learned-space regularity conditions, and a nonconstant nuisance ratio gives strict variance reduction in Proposition \ref{prop:weight-variance}. The fitted transformation $T$ may depend on the independent representation fold. This family realizes the sufficient-representation conditions and strict true-weight variance reduction.

\subsection{Proofs of the Comparison Results}\label{app:comparison-proofs}
The comparison bounds reuse the proposed method's ERM decomposition and network estimates, with raw-space change of measure and the appropriate input dimensions.
\begin{lemma}\label{prop:composition-smoothness}
Under Assumption \ref{assum5}, cube coverage $R^*(K_X)\subset K_U$, and a bounded extension of $h^*$ on $K_U$, the score $g^*:=h^*\circ R^*$ has H\"older index $\min\{\beta_{\mathcal H},\beta_{\mathcal F},\beta_{\mathcal H}\beta_{\mathcal F}\}$, with a radius depending only on the fixed dimensions, cubes and original smoothness radii. Thus the choice \eqref{eq:composition-index} satisfies Assumption \ref{ass:comparison-rates}(ii)'s regularity requirement.
\end{lemma}
\begin{proof}
Put $q=\min\{\beta_{\mathcal H},\beta_{\mathcal F}\}$. If $q\leq1$, bounded regularity radii on the fixed cubes imply
\begin{equation*}
\begin{split}
 |h^*(R^*(x))-h^*(R^*(x'))|
 &\leq C\|R^*(x)-R^*(x')\|^{\beta_{\mathcal H}\wedge1}\\
 &\leq C\|x-x'\|^{(\beta_{\mathcal H}\wedge1)(\beta_{\mathcal F}\wedge1)},\\
 (\beta_{\mathcal H}\wedge1)(\beta_{\mathcal F}\wedge1)
 &=\min\{\beta_{\mathcal H},\beta_{\mathcal F},\beta_{\mathcal H}\beta_{\mathcal F}\}.
\end{split}
\end{equation*}
The first inequality uses H\"older continuity when $\beta_{\mathcal H}\leq1$, and bounded first derivatives when $\beta_{\mathcal H}>1$; the inner modulus follows in the same way.

If $q>1$, put $\alpha=\lceil q\rceil-1$ and $r=q-\alpha\in(0,1]$. Decreasing both indices to $q$ on their fixed cubes gives bounded derivatives through order $\alpha$, with order-$\alpha$ derivatives $r$-H\"older. Lower-order derivatives are Lipschitz by integration of the next derivative along line segments. For $1\leq j\leq\alpha$, the iterated chain rule is
\begin{equation*}
 D^j(h^*\circ R^*)(x)[v_1,\ldots,v_j]
 =\sum_{\pi\in\Pi_j}
 D^{|\pi|}h^*(R^*(x))
 \left[D^{|B|}R^*(x)[v_i:i\in B]:B\in\pi\right],
\end{equation*}
where $\Pi_j$ is the finite set of partitions of $\{1,\ldots,j\}$. The formula follows by induction: differentiating a term either adds a singleton block to the outer derivative or adds the new direction to one existing inner block.
Every factor in each component of an order-$\alpha$ derivative is bounded and has an $r$-H\"older modulus in $x$. For the outer factors use Lipschitz continuity of $R^*$; for lower-order derivative factors use $\|x-x'\|\leq C\|x-x'\|^r$ on the fixed cube. For bounded scalar factors,
\begin{equation*}
 \left|\prod_{i=1}^ka_i-\prod_{i=1}^kb_i\right|
 \leq\sum_{i=1}^k|a_i-b_i|
       \prod_{l<i}|a_l|\prod_{l>i}|b_l|.
\end{equation*}
Consequently,
\begin{equation*}
 \max_{j\leq\alpha}\|D^j(h^*\circ R^*)\|_\infty\leq C,\qquad
 \|D^\alpha(h^*\circ R^*)(x)-D^\alpha(h^*\circ R^*)(x')\|
 \leq C\|x-x'\|^r.
\end{equation*}
For boundary values of $R^*$, first replace it by $(1-\epsilon)R^*+\epsilon u_0$, where $u_0$ is an interior point of $K_U$. Apply the formulas in the interior and let $\epsilon\downarrow0$; continuity of all derivatives on the compact cubes preserves the bounds. Hence the composite has index $q$. Since both original indices exceed one, $q$ equals the claimed minimum. This argument also covers integer $q$, for which $r=1$.
\end{proof}

For the remaining statements define $\epsilon_{r,X}:=\|\hat r-r^*\|_{L^1(P_{\mathcal S}(X))}$ and
$a_X:=\inf_{g\in\ddot{\mathcal H}}\mathcal R_{\mathcal T}(g)-\mathcal R_{\mathcal T}(h^*\circ R^*)$.
The empirical raw-weight risk is $\widehat{\mathcal R}_r(f)=n^{-1}\sum_i r(X_i)\ell(f(X_i),Y_i)$. Conditional on the first stages, let
\begin{equation*}
\begin{split}
 \mathfrak C_{n,B}&:=\mathbb E_{D_{\rm pred},\zeta}\!\left[
 \sup_{h\in\mathcal H}\left|\frac1n\sum_i\zeta_i\hat r(X_i)
       \ell(h(\hat R(X_i)),Y_i)\right|\;\middle|\;\mathscr G\right],\\
 \mathfrak C_{n,D}&:=\mathbb E_{D_{\rm pred},\zeta}\!\left[
 \sup_{g\in\ddot{\mathcal H}}\left|\frac1n\sum_i\zeta_i\hat r(X_i)
       \ell(g(X_i),Y_i)\right|\;\middle|\;\mathscr G\right].
\end{split}
\end{equation*}
Here $\mathscr G$ includes the representation-training data for Method \eqref{MeB}, and only weight-training data and randomness for Method \eqref{MeD}.

\begin{lemma}\label{lem:comparison-oracles}
Under original-space covariate shift, with $h^*\circ R^*$ the target Bayes score, prediction scores bounded by $\mathcal B$, fitted weights at most $e^M$, and independent prediction folds, the weighted ERMs satisfy
\begin{equation}\label{eq:comparison-oracles}
\begin{split}
 \mathbb E\mathcal E_{\mathcal T}(\hat h,\hat R)
 &\leq4\mathbb E\mathfrak C_{n,B}
          +2L_{\mathcal B}\mathbb E\epsilon_{r,X}+\mathbb E a_{\hat R},\\
 \mathbb E[\mathcal R_{\mathcal T}(\hat h)-\mathcal R_{\mathcal T}(h^*\circ R^*)]
 &\leq4\mathbb E\mathfrak C_{n,D}
          +2L_{\mathcal B}\mathbb E\epsilon_{r,X}+a_X.
\end{split}
\end{equation}
\end{lemma}
\begin{proof}
For every bounded ambient score $f$, covariate shift implies
\begin{equation*}
\begin{split}
 \mathcal R_{\mathcal T}(f)
 &=\int\mathbb E_{P_{\mathcal S}}[\ell(f(x),Y)\mid X=x]\,dP_{\mathcal T}(X)(x)
 =\mathcal R_{r^*}(f),\\
 |\mathcal R_{\mathcal T}(f)-\mathcal R_{\hat r}(f)|
 &\leq\mathbb E_{P_{\mathcal S}}[|r^*(X)-\hat r(X)|\ell(f(X),Y)]
 \leq L_{\mathcal B}\epsilon_{r,X}.
\end{split}
\end{equation*}
Within this proof take $\mathcal J=\{h\circ\hat R:h\in\mathcal H\}$ or $\mathcal J=\ddot{\mathcal H}$, and put
$G_n(\mathcal J)=\sup_{f\in\mathcal J}|\mathcal R_{\hat r}(f)-\widehat{\mathcal R}_{\hat r}(f)|$.
For an ERM $\hat f$ and any $f\in\mathcal J$,
\begin{equation*}
\begin{split}
 \mathcal R_{\mathcal T}(\hat f)-\mathcal R_{\mathcal T}(f)
 &\leq2L_{\mathcal B}\epsilon_{r,X}
       +\mathcal R_{\hat r}(\hat f)-\mathcal R_{\hat r}(f)\\
 &\leq2L_{\mathcal B}\epsilon_{r,X}+2G_n(\mathcal J)
       +\widehat{\mathcal R}_{\hat r}(\hat f)-\widehat{\mathcal R}_{\hat r}(f)\\
 &\leq2L_{\mathcal B}\epsilon_{r,X}+2G_n(\mathcal J).
\end{split}
\end{equation*}
Taking an approximation infimum gives
\begin{equation*}
 \mathcal R_{\mathcal T}(\hat f)-\mathcal R_{\mathcal T}(h^*\circ R^*)
 \leq2L_{\mathcal B}\epsilon_{r,X}+2G_n(\mathcal J)
 +\begin{cases}a_{\hat R},&\mathcal J=\mathcal H\circ\hat R,\\a_X,&\mathcal J=\ddot{\mathcal H}.\end{cases}
\end{equation*}
The ghost-sample symmetrization proved after \eqref{eq_errdecomp}, now with raw weights, yields
\begin{equation*}
 \mathbb E[G_n(\mathcal H\circ\hat R)\mid\mathscr G]\leq2\mathfrak C_{n,B},\qquad
 \mathbb E[G_n(\ddot{\mathcal H})\mid\mathscr G]\leq2\mathfrak C_{n,D}.
\end{equation*}
Averaging proves both lines of \eqref{eq:comparison-oracles}.
\end{proof}

\subsubsection{Proof of Proposition \ref{th4}}\label{pfth4}
\begin{proof}
The weighted-cover argument uses only the envelope $e^M$ of the fitted weights, so it applies to raw weights. The predictor architecture in Theorem \ref{Th5} gives
\begin{equation*}
 \mathbb E\mathfrak C_{n,B}\leq C n^{-\beta_{\mathcal H}/(2\beta_{\mathcal H}+d)}\log^3(en),\qquad
 \mathbb E a_{\hat R}\leq C n^{-\beta_{\mathcal H}/(2\beta_{\mathcal H}+d)}.
\end{equation*}
The second inequality is \eqref{eq:sufficient-approximation} for the same fitted representation. For raw domain ERM, use $(q,\beta,t)=(p,\beta_\Phi,m)$ in Lemmas \ref{lem:uniform-network-approximation} and \ref{lem:fast-domain-erm}. Then
\begin{equation*}
\begin{split}
 \inf_{\phi\in\Phi_M}\|\phi-\phi^*\|_\infty
 &\leq C m^{-\beta_\Phi/(2\beta_\Phi+p)},\\
 \mathbb E[\mathcal L_r(\hat\phi)-\mathcal L_r(\phi^*)]
 &\leq C m^{-2\beta_\Phi/(2\beta_\Phi+p)}\log^6(em),\\
 \mathbb E\epsilon_{r,X}
 &\leq\frac{2e^M}{\sqrt{\kappa_M}}
       \sqrt{\mathbb E[\mathcal L_r(\hat\phi)-\mathcal L_r(\phi^*)]}
 \leq C m^{-\beta_\Phi/(2\beta_\Phi+p)}\log^3(em).
\end{split}
\end{equation*}
Here the ratio conversion uses $P_\mu=(P_{\mathcal S}(X)+P_{\mathcal T}(X))/2$, as in the proof of \eqref{eq:ratio-root-conversion}. Substitution into the first line of \eqref{eq:comparison-oracles} yields
\begin{equation*}
 \mathbb E\mathcal E_{\mathcal T}(\hat h,\hat R)
 \leq C\{n^{-\beta_{\mathcal H}/(2\beta_{\mathcal H}+d)}\log^3(en)
          +m^{-\beta_\Phi/(2\beta_\Phi+p)}\log^3(em)\}.
\end{equation*}
This proves \eqref{eq:strict-comparison-representation-rate}.
\end{proof}

\subsubsection{Proof of Proposition \ref{prop:original-weights-recovery}}\label{pf:original-weights-recovery}
\begin{proof}
Within this proof put $\nu=\beta_{\mathcal H}\wedge1$. Use the raw-ratio estimate from the preceding proof and the predictor-cover estimate. The recovery proof in Lemma \ref{thm:expanded-rates} uses only Stage 1, and gives
\begin{equation*}
\begin{split}
 \mathbb E\mathfrak C_{n,B}&\leq C\mathsf r_n(d,\beta_{\mathcal H})\log^3(en),\\
 \mathbb E\epsilon_{r,X}&\leq C\mathsf r_m(p,\beta_\Phi)\log^3(em),\\
 \mathbb E a_{\hat R}&\leq C\{\mathsf r_n(d,\beta_{\mathcal H})
       +\mathsf r_{n_R}(p,\beta_{\mathcal F})^\nu\log^{3\nu}(en_R)
       +(\lambda\varepsilon_{\mathcal D,n_R})^\nu\}.
\end{split}
\end{equation*}
Substitution into \eqref{eq:comparison-oracles} gives
\begin{equation*}
\begin{split}
 \mathbb E\mathcal E_{\mathcal T}(\hat h,\hat R)\leq C\{&
 n^{-\beta_{\mathcal H}/(2\beta_{\mathcal H}+d)}\log^3(en)
 +m^{-\beta_\Phi/(2\beta_\Phi+p)}\log^3(em)\\
 &+n_R^{-\nu\beta_{\mathcal F}/(2\beta_{\mathcal F}+p)}\log^{3\nu}(en_R)
 +(\lambda\varepsilon_{\mathcal D,n_R})^\nu\}.
\end{split}
\end{equation*}
This proves \eqref{eq:comparison-representation-rate}; learned-ratio regularity was not used.
\end{proof}

\subsubsection{Proof of Proposition \ref{contl gen}}\label{pfcnge}
\begin{proof}
Apply the network estimates with $(q,\beta,t)=(p,\beta_{\ddot{\mathcal H}},n)$ to the ambient predictor sieve. Then
\begin{equation*}
\begin{split}
 \mathbb E\mathfrak C_{n,D}
 &\leq C n^{-\beta_{\ddot{\mathcal H}}/(2\beta_{\ddot{\mathcal H}}+p)}\log^3(en),\\
 a_X&\leq\inf_{g\in\ddot{\mathcal H}}\|g-h^*\circ R^*\|_{\infty,K_X}
 \leq C n^{-\beta_{\ddot{\mathcal H}}/(2\beta_{\ddot{\mathcal H}}+p)}.
\end{split}
\end{equation*}
The raw-ratio estimate from Proposition \ref{th4}'s proof remains unchanged. The second line of \eqref{eq:comparison-oracles} therefore gives
\begin{equation*}
 \mathbb E[\mathcal R_{\mathcal T}(\hat h)-\mathcal R_{\mathcal T}(h^*\circ R^*)]
 \leq C\{n^{-\beta_{\ddot{\mathcal H}}/(2\beta_{\ddot{\mathcal H}}+p)}\log^3(en)
       +m^{-\beta_\Phi/(2\beta_\Phi+p)}\log^3(em)\}.
\end{equation*}
This proves \eqref{eq:comparison-ambient-rate}.
\end{proof}

\subsubsection{Proof of Corollary \ref{cor:dimension-advantage}}
\begin{proof}
With the common index $\beta=\beta_{\mathcal H}$ and $m\asymp n$, put $a_d=\beta/(2\beta+d)$ and $a_p=\beta/(2\beta+p)$ within this proof. Since
\begin{equation*}
 a_d-a_p=\frac{\beta(p-d)}{(2\beta+d)(2\beta+p)}>0,\qquad
 n^{-a_d}+m^{-a_d}\asymp n^{-a_d},\qquad
 n^{-a_d}+m^{-a_p}\asymp n^{-a_p},
\end{equation*}
Theorem \ref{Th5} and Propositions \ref{th4}, \ref{contl gen} give exactly the three displayed upper bounds in \eqref{eq:dimension-advantage}.
\end{proof}

\subsubsection{Proof of Proposition \ref{prop:weight-variance}}
\begin{proof}
Fix the fitted representation and predictor, and put $U=\hat R(X)$. Source sufficiency $Y\Vbar X\mid U$ implies, for $j=1,2$,
\begin{equation*}
 \mathbb E_{P_{\mathcal S}}[(r^*(X))^j\mid U,Y]
 =\mathbb E_{P_{\mathcal S}}[(r^*(X))^j\mid U].
\end{equation*}
For $j=1$, \eqref{eq:target-risk-bridge} identifies this conditional mean as $r_{\hat R}(U)$. For $j=2$, subtracting the squared means gives
\begin{equation*}
 \operatorname{Var}_{\mathcal S}(r^*(X)\mid U,Y)
 =\operatorname{Var}_{\mathcal S}(r^*(X)\mid U).
\end{equation*}
The loss is $(U,Y)$-measurable, hence
\begin{equation*}
\begin{split}
 \mathbb E_{P_{\mathcal S}}[r^*(X)\ell(h(U),Y)\mid U,Y]
 &=r_{\hat R}(U)\ell(h(U),Y),\\
 \operatorname{Var}_{\mathcal S}(r^*(X)\ell(h(U),Y)\mid U,Y)
 &=\ell(h(U),Y)^2\operatorname{Var}_{\mathcal S}(r^*(X)\mid U).
\end{split}
\end{equation*}
The two weighted losses have the same mean by total expectation. Total variance gives
\begin{equation*}
\begin{split}
 &\operatorname{Var}_{\mathcal S}[r^*(X)\ell(h(U),Y)]\\
 &=\operatorname{Var}_{\mathcal S}[r_{\hat R}(U)\ell(h(U),Y)]
 +\mathbb E_{P_{\mathcal S}}[\ell(h(U),Y)^2
                         \operatorname{Var}_{\mathcal S}(r^*(X)\mid U)].
\end{split}
\end{equation*}
All second moments are finite by the weight and score bounds. Moreover,
\begin{equation*}
 \mathbb E_{P_{\mathcal S}}[\ell(h(U),Y)^2\operatorname{Var}_{\mathcal S}(r^*(X)\mid U)]
 \geq\log^2(1+e^{-\mathcal B})\,
        \mathbb E_{P_{\mathcal S}}\operatorname{Var}_{\mathcal S}(r^*(X)\mid U).
\end{equation*}
The last expression is positive exactly when the conditional weight variance is positive on a set of positive source probability. This proves both the identity and its strictness assertion.
\end{proof}

\subsection{Persistent Error from an Insufficient Representation}
An explicit fixed-class example illustrates the obstruction even if Assumption \ref{assum3} is retained. Let the source feature law be uniform on $[0,1]^p$, $p\geq2$, and let the target feature density relative to it be
$r^*(x)=1+\varepsilon(2x_1-1)$, $0<\varepsilon<1$.
In both domains set $P(Y=1\mid X=x)=\eta^*(x_1):=\sigma(\tau(x_1-1/2))$, $\tau>0$.
Then $R^*(x)=x_1$ is a shared sufficient smooth representation, $h^*(u)=\tau(u-1/2)$ is bounded and smooth, and $\log r^*$ is smooth with bounded positive overlap. Thus oracle smoothness, bounded overlap and covariate shift hold. Without a sieve approximation condition, choose the fixed representation class $\mathcal F=\{R(x)=x_2\}$. This class is a valid fixed ReLU class, and Stage 1 returns $\hat R=R$ for every $n_R$.
In the source domain $R(X)$ is independent of $Y$, while $R^*(X)$ is not. Both source feature marginals are uniform on $[0,1]$, so their finite KL divergences to the standard Gaussian are identical. Consequently
$\mathcal L(R)-\mathcal L(R^*)=\mathcal V[X_1,Y]>0$ and
$\|R-R^*\|_{L^1(P_{\mathcal S})}=1/3$.
Assumption \ref{assum3} therefore holds with a fixed constant at least $1/(3\mathcal V[X_1,Y])$, but supplies no vanishing recovery rate.
The learned-space feature marginals coincide, so $r_R=1$, whereas
\begin{equation*}
 b_R=\varepsilon\int_0^1(2u-1)\eta^*(u)\,du>0.
\end{equation*}
Let $\mathsf H(t)=-t\log t-(1-t)\log(1-t)$ denote binary entropy and $q_T=\mb E_{P_{\mathcal T}}\eta^*(X_1)$. Since $X_2$ is independent of $(X_1,Y)$ in the target domain, even an unrestricted predictor using $R$ has excess risk
\begin{equation*}
 \mathsf H(q_T)-\mb E_{P_{\mathcal T}}\mathsf H(\eta^*(X_1))>0
\end{equation*}
by strict concavity. Thus $a_R$ is bounded below by a positive constant for every prediction sample size and every predictor sieve. One can estimate $r_R$ perfectly and still retain both representation errors. This example fails the learned-sufficiency condition of Theorem \ref{Th5}. It shows that smooth oracle functions, bounded true weights and the loss-to-coordinate comparison alone do not establish a vanishing representation error. Approximation of the chosen representation, population objective optimality, alignment, and controlled variational estimation must be supplied separately, as in Assumption \ref{ass:explicit-rates}.

\subsection{Proofs of the Prediction-Risk Comparison and the Learnable Model}\label{app:fitted-risk-proofs}

\subsubsection{Proof of Theorem \ref{thm:fitted-risk-comparison}}
The argument uses a smooth finite-dimensional M-estimation expansion \citep[Chapter~5]{van2000asymptotic}, with explicit remainder bounds to obtain convergence of expected excess risk.
\begin{proof}
Within this proof put
\begin{equation*}
\begin{gathered}
 s(X,Y):=(\sigma(h_{\theta_{\mathcal T}}(U))-Y)z(U),\qquad U=R(X),\\
 w_A(X)=r_R(U),\quad w_B(X)=r^*(X),\qquad
 \hat w_A(X)=\hat r_d(U),\quad \hat w_B(X)=\hat r(X),\\
 P_n g=\frac1n\sum_{i=1}^n g(X_i,Y_i).
\end{gathered}
\end{equation*}
For every bounded measurable $g$, source sufficiency and covariate shift imply
\begin{equation}\label{eq:score-change-measure}
\begin{split}
 \mathbb E_{P_{\mathcal S}}[w_A g(U,Y)]
 &=\mathbb E_{P_{\mathcal S}}[w_B g(U,Y)]
 =\mathbb E_{P_{\mathcal T}}g(U,Y),\\
 \mathbb E_{P_{\mathcal S}}[w_A s]
 &=\mathbb E_{P_{\mathcal S}}[w_B s]=0.
\end{split}
\end{equation}
The first equality follows from
$\mathbb E_{P_{\mathcal S}}[r^*(X)\mid U,Y]=r_R(U)$;
the second line is the interior first-order condition at $\theta_{\mathcal T}$.
Since $\sigma'(t)\geq\kappa_{\mathcal B}:=\sigma(\mathcal B)(1-\sigma(\mathcal B))>0$
on $[-\mathcal B,\mathcal B]$,
\begin{equation*}
 \mathbf H_{\mathcal T}\succeq
 \kappa_{\mathcal B}\mathbb E_{P_{\mathcal T}}[z(U)z(U)^\top]\succ0.
\end{equation*}
The same lower bound on the Hessian throughout $\Theta$ gives uniqueness of $\theta_{\mathcal T}$.

Fix $j\in\{A,B\}$ and put $\Delta_n=\hat\theta_j-\theta_{\mathcal T}$.
Let $g_0>0$ be the smallest eigenvalue of
$\mathbb E_{P_{\mathcal S}}[z(U)z(U)^\top]$; positivity follows from
$\mathbb E_{P_{\mathcal T}}zz^\top\preceq C_r\mathbb E_{P_{\mathcal S}}zz^\top$.
For $\|z\|\leq K_z$, entrywise Hoeffding bounds and
$\|\mathbf A\|_{\rm op}\leq k\max_{a,b}|(\mathbf A)_{ab}|$ for
$\mathbf A\in\mathbb R^{k\times k}$ give
\begin{equation}\label{eq:empirical-gram-event}
\begin{gathered}
 G_n:=\{\lambda_{\min}(P_nzz^\top)\geq g_0/2\},\\
 P(G_n^c)\leq2k^2
       \exp\!\left(-\frac{ng_0^2}{8k^2K_z^4}\right).
\end{gathered}
\end{equation}
On $G_n$, the empirical weighted Hessian is at least
$c\kappa_{\mathcal B}g_0 I_k/2$. Convexity of $\Theta$ and the constrained
minimizer's first-order inequality therefore yield
\begin{equation}\label{eq:parameter-score-control}
\begin{split}
 \tfrac12c\kappa_{\mathcal B}g_0\|\Delta_n\|^2
 &\leq
 \{P_n[\hat w_j(\sigma(h_{\hat\theta_j}(U))-Y)z(U)]
      -P_n\hat w_j s\}^\top\Delta_n\\
 &\leq-(P_n\hat w_j s)^\top\Delta_n
 \leq\|P_n\hat w_j s\|\,\|\Delta_n\|.
\end{split}
\end{equation}

For a bounded mean-zero i.i.d. scalar sequence $\xi_i$,
\begin{equation}\label{eq:fourth-moment-average}
 \mathbb E\left(\frac1n\sum_i\xi_i\right)^4
 =\frac{\mathbb E\xi_1^4}{n^3}
  +\frac{3(n-1)}{n^3}(\mathbb E\xi_1^2)^2=O(n^{-2}).
\end{equation}
Applying this coordinatewise conditional on the weights, using their uniform
bounds and fixed $k$, gives
\begin{equation*}
\begin{split}
 \mathbb E\|P_n\hat w_j s
              -\mathbb E_{P_{\mathcal S}}[\hat w_j s\mid\hat w_j]\|^4
     &\leq Cn^{-2},\\
 \|\mathbb E_{P_{\mathcal S}}[\hat w_j s\mid\hat w_j]\|
     &=\|\mathbb E_{P_{\mathcal S}}[(\hat w_j-w_j)s\mid\hat w_j]\|
       \leq C\|\hat w_j-w_j\|_{L^2(P_{\mathcal S}(X))}.
\end{split}
\end{equation*}
The latent $L^2$ norm equals its pullback norm on $X$.
Equation \eqref{eq:weight-fourth-moment}, \eqref{eq:parameter-score-control},
compactness of $\Theta$, and \eqref{eq:empirical-gram-event} imply
\begin{equation}\label{eq:parameter-fourth-moment}
 \mathbb E\|P_n\hat w_j s\|^4=O(n^{-2}),\qquad
 \mathbb E\|\Delta_n\|^4=O(n^{-2}).
\end{equation}
In particular, $\mathbb E\|\Delta_n\|^2=O(n^{-1})$.
Since $\theta_{\mathcal T}\in\operatorname{int}\Theta$,
\begin{equation*}
 P\{\hat\theta_j\notin\operatorname{int}\Theta\}
 \leq \frac{\mathbb E\|\Delta_n\|^4}
            {\operatorname{dist}(\theta_{\mathcal T},\partial\Theta)^4}
       =O(n^{-2}).
\end{equation*}

To identify the leading term, conditional independence and boundedness give
\begin{equation}\label{eq:weight-score-remainder}
\begin{split}
 \mathbb E\|P_n(\hat w_j-w_j)s\|^2
 &\leq C(1+n^{-1})
       \mathbb E\|\hat w_j-w_j\|_{L^2(P_{\mathcal S}(X))}^2\\
 &\leq C(1+n^{-1})
       \left(\mathbb E\|\hat w_j-w_j\|_{L^2(P_{\mathcal S}(X))}^4\right)^{1/2}
       =o(n^{-1}).
\end{split}
\end{equation}
Define the integrated empirical Hessian
\begin{equation*}
 \bar{\mathbf H}_n=
 \int_0^1P_n\!\left[
 \hat w_j\sigma'(h_{\theta_{\mathcal T}+t\Delta_n}(U))zz^\top\right]dt .
\end{equation*}
Boundedness and Lipschitz continuity of $\sigma'$ imply
\begin{equation*}
\begin{split}
 \|\bar{\mathbf H}_n-\mathbf H_{\mathcal T}\|_{\rm op}
 \leq C\|\Delta_n\|
 &+C P_n|\hat w_j-w_j|\\
 &+\|(P_n-P_{\mathcal S})
      [w_j\sigma'(h_{\theta_{\mathcal T}}(U))zz^\top]\|_{\rm op}.
\end{split}
\end{equation*}
The last term converges to zero in probability by the entrywise law of large
numbers. Also
$\mathbb E P_n|\hat w_j-w_j|
 \leq(\mathbb E\|\hat w_j-w_j\|_{L^2(P_{\mathcal S}(X))}^2)^{1/2}\to0$.
Thus $\bar{\mathbf H}_n\to\mathbf H_{\mathcal T}$ in probability.
These matrices are uniformly bounded, so
\begin{equation*}
 \mathbb E\|\bar{\mathbf H}_n-\mathbf H_{\mathcal T}\|_{\rm op}^4=o(1),\qquad
 \mathbb E\|(\bar{\mathbf H}_n-\mathbf H_{\mathcal T})\Delta_n\|^2
 \leq
 \left(\mathbb E\|\bar{\mathbf H}_n-\mathbf H_{\mathcal T}\|_{\rm op}^4
       \mathbb E\|\Delta_n\|^4\right)^{1/2}=o(n^{-1}).
\end{equation*}
When $\hat\theta_j$ is interior, its score equation gives
\begin{equation*}
 \Delta_n+\mathbf H_{\mathcal T}^{-1}P_nw_js
 =-\mathbf H_{\mathcal T}^{-1}
       \{P_n(\hat w_j-w_j)s+
          (\bar{\mathbf H}_n-\mathbf H_{\mathcal T})\Delta_n\}.
\end{equation*}
On the complementary event the left side is uniformly bounded, and that
event has probability $O(n^{-2})$. Hence
\begin{equation}\label{eq:parameter-L2-expansion}
 \mathbb E\|\Delta_n+\mathbf H_{\mathcal T}^{-1}P_nw_js\|^2=o(n^{-1}).
\end{equation}
By \eqref{eq:score-change-measure},
$\mathbb E(P_nw_js)(P_nw_js)^\top=\mathbf V_j/n$.
Cauchy--Schwarz applied to \eqref{eq:parameter-L2-expansion} now gives
\begin{equation*}
 \mathbb E[\Delta_n^\top\mathbf H_{\mathcal T}\Delta_n]
 =\frac1n\operatorname{tr}(\mathbf H_{\mathcal T}^{-1}\mathbf V_j)+o(n^{-1}).
\end{equation*}
The third derivative of $\mathcal R_{\mathcal T}(h_\theta\circ R)$ is
uniformly bounded on $\Theta$, because $z$ and $\sigma''$ are bounded.
Taylor's formula and \eqref{eq:parameter-fourth-moment} yield
\begin{equation*}
\begin{split}
 \left|\mathbb E[
 \mathcal R_{\mathcal T}(h_{\hat\theta_j}\circ R)
 -\mathcal R_{\mathcal T}(h_{\theta_{\mathcal T}}\circ R)]
 -\tfrac12\mathbb E[\Delta_n^\top\mathbf H_{\mathcal T}\Delta_n]\right|
 &\leq C\mathbb E\|\Delta_n\|^3\\
 &\leq C(\mathbb E\|\Delta_n\|^4)^{3/4}
 =O(n^{-3/2}).
\end{split}
\end{equation*}
This proves \eqref{eq:fitted-risk-expansion}.

Finally, sufficiency gives
$\mathbb E[(r^*)^2\mid U,Y]=\mathbb E[(r^*)^2\mid U]$
and $\mathbb E[r^*\mid U]=r_R(U)$. Therefore
\begin{equation*}
\begin{split}
 \mathbf V_B-\mathbf V_A
 &=\mathbb E_{P_{\mathcal S}}[
   \{\mathbb E[(r^*)^2\mid U]-r_R(U)^2\}ss^\top]\\
 &=\mathbb E_{P_{\mathcal S}}[
       \operatorname{Var}_{\mathcal S}(r^*(X)\mid U)ss^\top]\succeq0,\\
 \operatorname{tr}\{\mathbf H_{\mathcal T}^{-1}(\mathbf V_B-\mathbf V_A)\}
 &=\mathbb E_{P_{\mathcal S}}[
   \operatorname{Var}_{\mathcal S}(r^*(X)\mid U)
   (\sigma(h_{\theta_{\mathcal T}}(U))-Y)^2
   z(U)^\top\mathbf H_{\mathcal T}^{-1}z(U)] .
\end{split}
\end{equation*}
For binary $Y$ and finite scores, the squared residual is strictly positive.
Positive definiteness of $\mathbf H_{\mathcal T}$ proves strict positivity
under \eqref{eq:fitted-strict-condition}.
\end{proof}

\subsubsection{Proof of Proposition \ref{prop:learnable-coordinate}}
\begin{proof}
For $j\ne j^*$, $X_j$ is independent of $Y$ under the source law, so
$\mathcal V[X_j,Y]=0$.
On the input cube, $x_j=\varrho(x_j)$ is realized by one hidden unit and
an affine identity output. Definition \ref{def:dnn} gives
$\mathcal W_{\mathcal F}=\max\{p,1,1\}=p$ and
$\mathcal S_{\mathcal F}=1(p+1)+1(1+1)=p+3$.
For $U=X_{j^*}$, write $q=\mathbb E_{P_{\mathcal S}}Y=1/2$ within this proof,
and $\eta^*(u)=1/2+a(u-1/2)$.
For independent copies $(U_i,Y_i)$ and $U,U'$ of the source variables,
\begin{equation*}
\begin{aligned}
 \mathbb E[|U_1-U_2||Y_1-Y_2|]
 &=2\mathbb E[|U-U'|\eta^*(U)]
      -2\mathbb E[|U-U'|\eta^*(U)\eta^*(U')],\\
 \mathbb E[|U_1-U_2||Y_1-Y_3|]
 &=(1-2q)\mathbb E[|U-U'|\eta^*(U)]+q\mathbb E|U-U'|,\\
 \mathbb E|U-U'|\,\mathbb E|Y_1-Y_2|
 &=2q(1-q)\mathbb E|U-U'|.
\end{aligned}
\end{equation*}
The distance-moment identity in Lemma \ref{lem:dcov-kernel} therefore gives
\begin{equation*}
\begin{split}
 \mathcal V[U,Y]
 &=-2\mathbb E[|U-U'|(\eta^*(U)-q)(\eta^*(U')-q)]\\
 &=-2a^2\int_0^1\!\!\int_0^1|u-v|(u-\tfrac12)(v-\tfrac12)\,du\,dv
 =\frac{a^2}{30}=:\Delta>0.
\end{split}
\end{equation*}
To verify the integral, put $g(u)=u-1/2$.
Since $\int_0^1g=0$ and
$|u-v|=\int_0^1|\mathbf1_{\{u>t\}}-\mathbf1_{\{v>t\}}|^2dt$,
\begin{equation*}
 \int_0^1\!\!\int_0^1|u-v|g(u)g(v)\,du\,dv
 =-2\int_0^1\left(\int_0^t g(u)\,du\right)^2dt
 =-\tfrac12\int_0^1t^2(1-t)^2dt=-\frac1{60}.
\end{equation*}

All $X_j$ have the same uniform source law, whose KL divergence to
$\gamma_1$ is finite:
\begin{equation*}
 \mathbb D_{\rm KL}(\operatorname{Unif}[0,1]\|\gamma_1)
 =\int_0^1\{\tfrac12\log(2\pi)+u^2/2\}\,du
 =\tfrac12\log(2\pi)+\tfrac16.
\end{equation*}
Consequently, for $R(x)=x_j$ with $j\ne j^*$,
\begin{equation*}
 \mathcal L(R)-\mathcal L(R^*)=\Delta,\qquad
 \|R-R^*\|_{L^1(P_{\mathcal S})}
 =\int_0^1\!\!\int_0^1|u-v|\,du\,dv=\tfrac13 .
\end{equation*}
For $j=j^*$ both sides vanish. This proves \eqref{eq:coordinate-recovery}
with $C_{\rm id}=1/(3\Delta)=10/a^2$.
The restricted population discriminator objective is also identical for
all candidates. Thus its difference from the full divergence is a common
additive constant, including when the KL bias floor
\eqref{eq:kl-discriminator-bias-floor} is positive.

We next control the empirical selection. Let $P_{n_R,j}$ be the empirical
law of the $j$th coordinate on the representation fold, $F_{n_R,j}$ its
distribution function, and $Q_{n_R}$ the common empirical reference law.
With $\widehat{\mathbb D}_j=
 \sup_D\{P_{n_R,j}D-Q_{n_R}f^*(D)\}$,
\begin{equation}\label{eq:coordinate-penalty-cancellation}
\begin{split}
 |\widehat{\mathbb D}_j-\widehat{\mathbb D}_{j^*}|
 &\leq\sup_D|(P_{n_R,j}-P_{n_R,j^*})D|\\
 &\leq K_D\int_0^1|F_{n_R,j}(u)-F_{n_R,j^*}(u)|\,du\\
 &\leq K_D\{\|F_{n_R,j}-\operatorname{id}\|_\infty
               +\|F_{n_R,j^*}-\operatorname{id}\|_\infty\}.
\end{split}
\end{equation}
Here $\operatorname{id}(u)=u$. The second inequality follows by integration by parts: Lipschitz $D$
is absolutely continuous on $[0,1]$, $|D'|\leq K_D$ almost everywhere,
and the two probability measures have equal mass. The common reference
term cancels before taking any probability bound.

For completeness, an elementary grid bound controls these distribution
functions. For $0<t<1$ and $L=\lceil2/t\rceil$, monotonicity gives
\begin{equation*}
 \|F_{n_R,j}-\operatorname{id}\|_\infty
 \leq\max_{0\leq l\leq L}|F_{n_R,j}(l/L)-l/L|+L^{-1}.
\end{equation*}
Hoeffding's inequality at each grid point, followed by a union bound, yields
\begin{equation}\label{eq:coordinate-cdf-bound}
 P\left\{\max_{j\leq p}\|F_{n_R,j}-\operatorname{id}\|_\infty>t\right\}
 \leq2p(L+1)e^{-n_Rt^2/2}.
\end{equation}

The order-four distance-covariance kernel satisfies $|K|\leq18$
for $X_j\in[0,1]$ and $Y\in\{0,1\}$, by
\eqref{eq:kernel-envelope-lipschitz}. Put $b=\lfloor n_R/4\rfloor$.
The U-statistic is the average, over all permutations, of the averages
of $b$ disjoint kernels. Convexity of the exponential and the bounded
variable moment-generating-function bound imply
\begin{equation*}
 \mathbb E\exp\{s(\widehat{\mathcal V}_{n_R}-\mathcal V)\}
 \leq\exp\{s^2\,36^2/(8b)\},\qquad s\in\mathbb R.
\end{equation*}
Indeed, each kernel has range length at most $36$, and the $b$ kernels
in a fixed permutation are independent. Optimizing the exponential
Markov bound and applying it to both signs gives
\begin{equation}\label{eq:coordinate-ustat-bound}
 P\left\{\max_{j\leq p}|\widehat{\mathcal V}_{n_R}[X_j,Y]
                          -\mathcal V[X_j,Y]|>\Delta/4\right\}
 \leq2p\exp(-b\Delta^2/10368).
\end{equation}
Take
\begin{equation*}
 t=\frac{\Delta}{8(1+\lambda K_D)},\quad L=\lceil2/t\rceil,\quad
 c=\min\{\Delta^2/82944,t^2/2\},\quad C=2(L+2).
\end{equation*}
Outside the two exceptional events in
\eqref{eq:coordinate-cdf-bound}--\eqref{eq:coordinate-ustat-bound},
for every $j\ne j^*$,
\begin{equation*}
\begin{split}
 \widehat{\mathcal L}_{n_R}(x\mapsto x_j)
 -\widehat{\mathcal L}_{n_R}(R^*)
 &\geq \Delta-2(\Delta/4)-2\lambda K_Dt
 \geq\Delta/4>0 .
\end{split}
\end{equation*}
Since $b\geq n_R/8$ for $n_R\geq4$, a union bound proves
$P\{\hat j\ne j^*\}\leq Cp e^{-cn_R}$, for any rule resolving empirical ties.

Finally, $(1-\rho)^p\leq r^*\leq(1+\rho)^p$.
On $\{\hat j=j^*\}$,
\begin{equation*}
 r_{\hat R}(u)=1+\tau_{j^*}(2u-1),\qquad
 h_{\hat R}^*(u)=
 \log\frac{\frac12+a(u-\frac12)}{\frac12-a(u-\frac12)}.
\end{equation*}
These functions have bounded derivatives of every fixed order on $[0,1]$,
uniformly in $|\tau_{j^*}|\leq\rho<1$.
For any fixed H\"older indices the required extensions, radii and score
bounds therefore exist. Conditional on a representation-fold realization
with $\hat j=j^*$, the proof of Theorem \ref{Th5} applies with uniform
constants and $d=1$; the other folds remain independent.
On $\{\hat j\ne j^*\}$, bounded prediction scores give
$0\leq\mathcal E_{\mathcal T}\leq L_{\mathcal B}$.
Splitting expectation over these events gives
\begin{equation*}
 \mathbb E\mathcal E_{\mathcal T}
 \leq\widetilde O\!\left(
 n^{-\frac{\beta_{\mathcal H}}{2\beta_{\mathcal H}+1}}
 +m^{-\frac{\beta_{\ddot\Phi}}{2\beta_{\ddot\Phi}+1}}\right)
 +L_{\mathcal B}P\{\hat j\ne j^*\},
\end{equation*}
which is \eqref{eq:coordinate-neural-rate}.
\end{proof}

\subsubsection{Proof of Proposition \ref{prop:coordinate-comparison}}
\begin{proof}
The target marginal density of $T_j=2X_j-1$ on $[-1,1]$ is
$(1+\tau_jt)/2$, and
\begin{equation*}
 \mathbb E_{P_{\mathcal T}}T_j=\tau_j/3.
\end{equation*}
Projection onto $[-\rho,\rho]$ cannot increase distance to $\tau_j$.
Apply \eqref{eq:fourth-moment-average} to the centered bounded variables
$3T_j-\tau_j$. For a constant independent of $m$,
\begin{equation}\label{eq:coordinate-ratio-moments}
 \mathbb E|\hat\tau_j-\tau_j|^4\leq C m^{-2}.
\end{equation}
Each fitted and true factor is in $[1-\rho,1+\rho]$.
The telescoping product identity gives
\begin{equation*}
\begin{split}
 \|\hat r-r^*\|_\infty
 &\leq(1+\rho)^{p-1}\sum_{j=1}^p|\hat\tau_j-\tau_j|,\\
 \mathbb E\|\hat r-r^*\|_{L^2(P_{\mathcal S})}^4
 &\leq(1+\rho)^{4(p-1)}p^3\sum_{j=1}^p
             \mathbb E|\hat\tau_j-\tau_j|^4
 \leq C m^{-2}.
\end{split}
\end{equation*}
For any coordinate $j$ the latent error satisfies
\begin{equation*}
 \|[1+\hat\tau_j(2u-1)]-[1+\tau_j(2u-1)]\|_{L^2(\operatorname{Unif}[0,1])}^2
 =|\hat\tau_j-\tau_j|^2/3.
\end{equation*}
Thus \eqref{eq:weight-fourth-moment} holds for the oracle coordinate when
$m/n\to\infty$; all weight bounds are fixed for fixed $p,\rho$.
This verifies the weight condition by an explicit estimator.

With $U=X_{j^*}$ and $T=2U-1$,
\begin{equation*}
 \eta^*(U)=\tfrac12+\tfrac a2 T,\qquad
 q_{\mathcal T}=\tfrac12+\tfrac a2\,\mathbb E_{P_{\mathcal T}}T
               =\tfrac12+a\tau_{j^*}/6.
\end{equation*}
Both $1/2$ and $q_{\mathcal T}$ lie strictly inside
$[\sigma(-\mathcal B),\sigma(\mathcal B)]$.
For any positive weights $w_i$ the unique empirical minimizing constant has
\begin{equation}\label{eq:constant-weighted-erm}
 \sigma(\hat h)=
 \Pi_{[\sigma(-\mathcal B),\sigma(\mathcal B)]}
       \left(\frac{\sum_iw_iY_i}{\sum_iw_i}\right),
\end{equation}
as follows by differentiating
$\sum_iw_i\{\log(1+e^t)-Y_it\}$.
Theorem \ref{thm:fitted-risk-comparison} applies with $z(u)=1$,
$\Theta=[-\mathcal B,\mathcal B]$, and
$\mathbf H_{\mathcal T}=q_{\mathcal T}(1-q_{\mathcal T})$.

Under the source law $T$ is uniform on $[-1,1]$. Abbreviate
$\tau=\tau_{j^*}$ within this calculation. Conditional on $T$,
\begin{equation*}
 \mathbb E[(Y-q_{\mathcal T})^2\mid T]
 =\eta^*(U)(1-2q_{\mathcal T})+q_{\mathcal T}^2
 =\tfrac14+\frac{a^2\tau^2}{36}-\frac{a^2\tau}{6}T.
\end{equation*}
Using $\mathbb ET=0$, $\mathbb ET^2=1/3$, and $\mathbb ET^3=0$,
\begin{equation*}
\begin{split}
 \mathbf V_A
 &=\mathbb E\left[(1+\tau T)^2
       \left(\tfrac14+\frac{a^2\tau^2}{36}-\frac{a^2\tau}{6}T\right)\right]\\
 &=\left(\tfrac14+\frac{a^2\tau^2}{36}\right)(1+\tau^2/3)
       -a^2\tau^2/9\\
 &=\tfrac14+\frac{(1-a^2)\tau^2}{12}+\frac{a^2\tau^4}{108}.
\end{split}
\end{equation*}
The source nuisance coordinates are mutually independent and independent of
$(U,Y)$, so
\begin{equation*}
 \mathbf V_B
 =\mathbf V_A\prod_{j\ne j^*}
       \mathbb E_{P_{\mathcal S}}\bigl[(1+\tau_j(2X_j-1))^2\bigr]
 =\mathbf V_A\prod_{j\ne j^*}(1+\tau_j^2/3).
\end{equation*}
Since $q_{\mathcal T}(1-q_{\mathcal T})=1/4-a^2\tau^2/36$,
equation \eqref{eq:fitted-risk-expansion} gives the displayed constants
when the oracle coordinate is used.

Let $\tilde h_A$ minimize the constant-score empirical risk with weights
$1+\hat\tau_{j^*}(2X_{ij^*}-1)$ on the same domain and prediction folds.
Equation \eqref{eq:constant-weighted-erm} gives
$\hat h_A=\tilde h_A$ on $\{\hat j=j^*\}$.
Their risks lie in $[0,L_{\mathcal B}]$, hence
\begin{equation*}
 \left|\mathbb E\mathcal R_{\mathcal T}(\hat h_A\circ\hat R)
      -\mathbb E\mathcal R_{\mathcal T}(\tilde h_A\circ R^*)\right|
 \leq L_{\mathcal B}Cp e^{-cn_R}=o(n^{-1}).
\end{equation*}
Raw-weight and unit-weight constant estimators do not depend on the selected
coordinate. This proves the first two lines of
\eqref{eq:coordinate-risk-comparison} for the learned representation.

For unit weights, $n^{-1}\sum_iY_i\to1/2$ in probability, so
\eqref{eq:constant-weighted-erm} gives $\hat h_C\to0$.
Bounded scores imply convergence of expected target risks.
For any constant score $h(u)=\log(q/(1-q))$, $q\in(0,1)$,
\begin{equation*}
\begin{split}
 \mathcal R_{\mathcal T}(h\circ R^*)
 &=-q_{\mathcal T}\log q-(1-q_{\mathcal T})\log(1-q),\\
 \mathcal R_{\mathcal T}(0)
 -\mathcal R_{\mathcal T}(h_{\mathcal H,\mathcal T}^*\circ R^*)
 &=q_{\mathcal T}\log(2q_{\mathcal T})
   +(1-q_{\mathcal T})\log(2(1-q_{\mathcal T})).
\end{split}
\end{equation*}
This equals the asserted Bernoulli KL divergence and is positive when
$\tau_{j^*}\ne0$. The product in the raw-weight constant exceeds one
exactly when some $\tau_j\ne0$ for $j\ne j^*$.
For all three estimators, changing the benchmark from the best constant to
the Bayes score adds the same deterministic approximation error
\begin{equation*}
 \mathcal R_{\mathcal T}(h_{\mathcal H,\mathcal T}^*\circ R^*)
 -\mathcal R_{\mathcal T}(h^*\circ R^*).
\end{equation*}
\end{proof}

\subsection{Preliminary Lemmas}\label{sec:prl}
\begin{definition}[Rademacher complexity]
Let $\{\zeta_i\}_{i =1}^n$ be $n$ i.i.d.  Rademacher random variables, each uniformly taking  values $-1$ or $1$ with probability $1/2$. Let $\{X_i\}_{i =1}^n$ be $n$  i.i.d. random variables on $\mathcal X$, 
and let $\mathcal H$ be a set of measurable functions defined  on $\mathcal X$.  
Then, the empirical Rademacher complexity of $\mathcal H$ is defined as
$$
\hat{\mf{R}}_n(\mathcal{H})
= \mathbb{E}_{\zeta}\left[\sup_{h \in \mathcal{H}} \frac{1}{n}\sum_{i = 1}^n \zeta_i  h(X_i)  \right].
$$    
And the Rademacher complexity of $\mc{H}$ is the expectation of $\hat{\mf{R}}_{n}(\mc{H})$, that is, 
\begin{equation*}
   \mf{R}_n(\mathcal{H}) = \mb{E}_{X}[\hat{\mf{R}}_n(\mathcal{H})].
\end{equation*}
\end{definition}

\begin{definition}[$\epsilon$-covering and Covering number]\label{epsilon cover}
Let $(\mathcal{H}, \bar{d})$ be a metric space, where the metric $\bar{d}: \mathcal{H} \times \mathcal{H} \to \mathbb{R}^+$. For any $\epsilon > 0 $, 
 we define $\{h_i \in \mathcal{H}, i = 1, \dots, N\}$ as an $\epsilon$-cover of $\mathcal{H}$,
  if for each $h \in \mathcal{H}$, there exists a $h_i$ such that $\bar{d}(h, h_i) \leq \epsilon.$
The $\epsilon$-covering number of $\mathcal{H}$ is the minimal cardinality of an $\epsilon$-cover:
    \begin{equation*}
        \mathcal{N}(\epsilon, \bar{d}, \mathcal{H}) = \min \{|\mathcal{C}|: \text{ $\mathcal{C}$ is an $\epsilon$-cover of $\mathcal{H}$}\}.
    \end{equation*}   
\end{definition}

\begin{definition}[Uniform covering number]
 Given $x = (x_1,\dots,x_n) \in \mathcal{X}^n$, denote $\mathcal{H}_{x} = \{(h(x_1), \dots,h(x_n)) : h \in \mathcal{H} \}$ as a subset of $\mathbb{R}^n$.
Let $\mathcal{N}(\delta,\|\cdot\|_{\infty},\mathcal{H}_{x})$ be the covering number of $\mc{H}_{x}$
with radius $\delta>0$ and metric $\|\cdot\|_{\infty}$.
Then, the uniform covering number $\mathcal{N}_{n}(\delta,\|\cdot\|_{\infty},\mathcal{H})$ is defined as
    \begin{equation*}
       \mathcal{N}_{n}(\delta,\|\cdot\|_{\infty},\mathcal{H}):= \sup\{{\mathcal{N}(\delta,\|\cdot\|_{\infty},\mathcal{H}_{x})} :x \in \mathcal{X}^n\}.
    \end{equation*}   
\end{definition}

\begin{definition}[Vapnik-Chervonenkis (VC)-dimension]
    Let $\mc{H}$ be a function class that maps $\mc{X}$ to $\{\pm 1\}$. For any non-negative integer $n$, the growth function of $\mc{H}$ is defined as 
    \begin{equation*}
        \Pi_{\mc{H}}(n) = \max_{\{x_i\}_{i =1}^n \subseteq \mc{X}} |\{(h(x_1), \dots, h(x_n)) : h \in \mc{H}\}|.
    \end{equation*}
    The set $A = \{x_i\}_{i =1}^n$ is said to be shattered by $\mc H$  if $|\{(h(x_1), \dots, h(x_n)) : h \in \mc{H}\}| = 2^n$. 
    The VC-dimension of $\mc{H}$, denoted by $VCdim(\mc H)$, 
    is the supremum of the cardinalities of sets shattered by $\mc{H}$, that is,
    $VCdim(\mc H) = \sup\{n : \Pi_{\mc{H}}(n) = 2^n\}$, with value $+\infty$ when these cardinalities are unbounded.
\end{definition}

\begin{definition}[Pseudo dimension]
Let $\mc{H}$ be a set of real valued functions from input space $\mc{X}$. The set $A = \{x_i\}_{i =1}^n$ is pseudo-shattered by $\mc{H}$ if there exists a fixed vector $w\in\mathbb R^n$ such that, for every $b\in\{\pm1\}^n$, some $h\in\mc H$ satisfies $b_i(h(x_i)-w_i)>0$ for all $i$. The pseudo dimension is
    \begin{equation*}
        Pdim(\mc H):=\sup\{|A|: A\text{ is pseudo-shattered by }\mc H\}.
    \end{equation*}
Here, $|\cdot|$ is the cardinality of a set.    
\end{definition}

\begin{lemma}\label{lemE10-01}
Let $\Phi$ be a real-valued function class and let $\phi^*$ be a fixed function on $\mathcal X$. On any fixed sample $X_1,\ldots,X_n$, the ordinary empirical Rademacher complexity satisfies
\begin{equation*}
 \hat{\mathfrak R}_n(\Phi)=\hat{\mathfrak R}_n(\Phi-\phi^*),\qquad
 \Phi-\phi^*:=\{\phi-\phi^*:\phi\in\Phi\}.
\end{equation*}
\end{lemma}
\begin{proof}
The shift is independent of the choice of $\phi$, and each Rademacher sign has mean zero. Thus
\begin{equation*}
\begin{split}
 \hat{\mathfrak R}_n(\Phi-\phi^*)
 &=\mathbb E_\zeta\!\left[\sup_{\phi\in\Phi}\frac1n\sum_{i=1}^n\zeta_i\phi(X_i)
                 -\frac1n\sum_{i=1}^n\zeta_i\phi^*(X_i)\right]\\
 &=\hat{\mathfrak R}_n(\Phi)
   -\frac1n\sum_{i=1}^n\phi^*(X_i)\mathbb E_\zeta\zeta_i
 =\hat{\mathfrak R}_n(\Phi).
\end{split}
\end{equation*}
This identity concerns the non-absolute complexity defined in Appendix \ref{sec:prl}; it need not hold with an absolute value inside the supremum.
\end{proof}

\begin{lemma}\label{lem2.1}
For a balanced source--target domain mixture with source label $Z=-1$ and target label $Z=1$, use $r^*$ from Section \ref{problem setting}, $\phi^*=\log r^*$, and $|\phi|,|\phi^*|\leq M$. For $\mathcal L_r(\phi)=\mb E\log(1+e^{-Z\phi(X)})$,
\begin{equation}\label{eq2.4}
 \frac{\kappa_M}{2}\|\phi-\phi^*\|_{L^2(P_\mu)}^2
 \leq\mathcal L_r(\phi)-\mathcal L_r(\phi^*)
 \leq\frac18\|\phi-\phi^*\|_{L^2(P_\mu)}^2.
\end{equation}
The same result applies conditional on a fixed learned representation, replacing the distributions and ratio by their pushforwards.
\end{lemma}
\begin{proof}
Given $X=x$, the target domain probability is $\pi(x)=r^*(x)/(1+r^*(x))$. The conditional domain loss is
$F_x(t)=\log(1+e^t)-\pi(x)t$, with $F_x'(\phi^*(x))=0$ and
$\kappa_M\leq F_x''(t)=\sigma(t)(1-\sigma(t))\leq1/4$ for $|t|\leq M$. Taylor's theorem, including its factor $1/2$, gives the two bounds pointwise for the conditional loss. Integrating with respect to $P_\mu$ proves \eqref{eq2.4}. The cancellation of the linear term occurs after conditioning on $X$, not for a single realized domain label.
\end{proof}

\begin{lemma}\label{lem2.2}
If $|\hat\phi|,|\phi^*|\leq M$, $\hat r=e^{\hat\phi}$, and $r^*=e^{\phi^*}$, then
\begin{equation*}
 \|\hat r-r^*\|_{L^2(P_\mu)}\leq e^M\|\hat\phi-\phi^*\|_{L^2(P_\mu)}.
\end{equation*}
\end{lemma}
\begin{proof}
The exponential function is $e^M$-Lipschitz on $[-M,M]$. Square the pointwise inequality and integrate. For learned features, replace $r^*,\phi^*,P_\mu$ by $r_{\hat R},\log r_{\hat R},P_\mu^{\hat R}$, respectively. The fitted-score bound is explicit; boundedness of the true ratio alone does not imply it.
\end{proof}

 \subsection{Auxiliary Results}\label{Au_results}
We use the squared distance-covariance functional and unbiased estimator of \cite{szekely2007measuring,huo2016fast}, and the variational divergence penalty used in \cite{huang2024deep}. The proofs below verify the kernel normalization and variational identity used in our analysis.

\noindent
\textit{Sufficient and Disentanglement Representation.}
Sufficiency means $Y\Vbar X\mid R^*(X)$. Component independence is separate; matching a standard Gaussian reference encourages it but does not by itself preserve label information. Use the source feature pushforward $P_{\mc S}^R$ and Gaussian reference $\gamma_d$ defined in Section \ref{problem setting}. The population representation objective motivating \eqref{eq:rep-training} is
\begin{equation}\label{eq2.7}
 \mc L(R):=-\mc V[R(X),Y]+\lambda\mb D_f(P_{\mc S}^R\|\gamma_d),\qquad\lambda\geq0.
\end{equation}

The dependence reward is evaluated in the source domain. We use squared distance covariance under finite first moments; it is nonnegative and vanishes exactly at independence \citep{szekely2007measuring}. It does not satisfy a data-processing inequality under every measurable transformation: for $c>0$, $\mc V[cX,Y]=c\mc V[X,Y]$, so scaling alone can increase it. Hence dependence maximization alone does not characterize sufficiency; the sufficiency and recovery conditions are stated separately in Assumptions \ref{ass:learned-sufficient} and \ref{ass:explicit-rates}.
For a normalized strictly convex $f$, zero divergence means equality of the feature law and the Gaussian reference, which is stronger than component independence. The reference penalty encourages this match in the source domain; it neither assumes exact Gaussianity nor forces the two domains to have the same latent marginal. The general finite-sample bound of Lemma \ref{lem:latent-oracle} requires no Gaussian reference or global coordinate recovery. We adopt the representation-training strategy of \cite{huang2024deep}, while retaining its resulting representation quality explicitly in $a_{\hat R}$ and $b_{\hat R}$.

\noindent
\textit{Distance Covariance \citep{szekely2007measuring}.} 
Let $U\in\mathbb R^d$ and $Y\in\mathbb R^q$, where $q$ is the response-embedding dimension in this paragraph ($q=1$ for binary labels), and assume $\mb E\|U\|+\mb E\|Y\|<\infty$. Write $\chi_U$, $\chi_Y$, and $\chi_{U,Y}$ for their characteristic functions. With $\mathrm i^2=-1$, for example $\chi_{U,Y}(t,s)=\mb E\exp(\mathrm i(t^\top U+s^\top Y))$. Here $\mc V$ denotes squared distance covariance:
\begin{equation*}
 \mc V[U,Y]=\int_{\mathbb R^d\times\mathbb R^q}
 \frac{|\chi_{U,Y}(t,s)-\chi_U(t)\chi_Y(s)|^2}
 {c_dc_q\|t\|^{d+1}\|s\|^{q+1}}\,dt\,ds,
 \qquad c_j:=\frac{\pi^{(j+1)/2}}{\Gamma((j+1)/2)}.
\end{equation*}
Here $\Gamma$ is the Euler gamma function. For the unbiased estimator below, also assume $\mb E[\|U\|\|Y\|]<\infty$, which is automatic for bounded binary labels under the preceding first-moment condition. For $n\geq4$ i.i.d. observations, the order-four U-statistic is \citep{huo2016fast}
\begin{equation*}
 \widehat{\mc V}_n=\binom n4^{-1}\sum_{1\leq i_1<i_2<i_3<i_4\leq n}
 K((U_{i_1},Y_{i_1}),\ldots,(U_{i_4},Y_{i_4})).
\end{equation*}
For four observations set $a_{ij}=\|u_i-u_j\|$, $b_{ij}=\|y_i-y_j\|$, $a_{ii}=b_{ii}=0$. Then
\begin{equation*}
 K=\frac14\sum_{i\ne j}a_{ij}b_{ij}
   -\frac14\sum_{i=1}^4\Big(\sum_{j\ne i}a_{ij}\Big)\Big(\sum_{j\ne i}b_{ij}\Big)
   +\frac1{24}\Big(\sum_{i\ne j}a_{ij}\Big)\Big(\sum_{i\ne j}b_{ij}\Big),
\end{equation*}
where all pair sums run over ordered pairs in $\{1,2,3,4\}$. The estimator can be negative in a finite sample even though the population functional is nonnegative. For binary scalar labels, $|y_i-y_j|$ is 0 or 1. If labels are embedded as one-hot vectors, the corresponding Euclidean distance is 0 or $\sqrt2$; the same chosen encoding is used throughout representation training.
\begin{lemma}\label{lem:dcov-kernel}
Under the stated moment conditions, $\mathbb EK=\mathcal V[U,Y]$, and hence $\mathbb E\widehat{\mathcal V}_n=\mathcal V[U,Y]$. Also, for $c>0$,
\begin{equation*}
 \mathcal V[cU,Y]=c\mathcal V[U,Y],\qquad
 \mathcal V[U,Y]=0\ \Longleftrightarrow\ U\Vbar Y.
\end{equation*}
\end{lemma}
\begin{proof}
For independent copies $(U_i,Y_i)$, use $a_{ij}=\|U_i-U_j\|$ and $b_{ij}=\|Y_i-Y_j\|$ within this proof, and put
\begin{equation*}
 A=\mathbb E a_{12}b_{12},\quad B=\mathbb E a_{12}b_{13},\quad
 C=\mathbb E a_{12}b_{34}=\mathbb E a_{12}\mathbb E b_{12}.
\end{equation*}
These quantities are finite by the moment assumptions. The distance-moment identity for squared distance covariance \cite{szekely2007measuring} is
\begin{equation*}
 \mathcal V[U,Y]=A-2B+C.
\end{equation*}
In the first kernel sum there are 12 ordered pairs. In each row product there are three identical-pair terms and six terms sharing one observation. In the product of the two full pair sums, each of the 12 first pairs has two second pairs with the same two indices, eight with one common index, and two with disjoint indices. Exchangeability therefore gives
\begin{equation*}
\begin{split}
 \mathbb E\!\left[\tfrac14\sum_{i\ne j}a_{ij}b_{ij}\right]&=3A,\\
 \mathbb E\!\left[\tfrac14\sum_i\Big(\sum_{j\ne i}a_{ij}\Big)
                                      \Big(\sum_{j\ne i}b_{ij}\Big)\right]&=3A+6B,\\
 \mathbb E\!\left[\tfrac1{24}\Big(\sum_{i\ne j}a_{ij}\Big)
                                 \Big(\sum_{i\ne j}b_{ij}\Big)\right]&=A+4B+C.
\end{split}
\end{equation*}
Subtracting proves $\mathbb EK=A-2B+C$. Every four-observation subset has this expectation, proving unbiasedness. Replacing $U$ by $cU$ multiplies every $a_{ij}$, and thus $A,B,C$, by $c$, proving homogeneity.
The defining characteristic-function integral is nonnegative. If it is zero, continuity makes $\chi_{U,Y}(t,s)=\chi_U(t)\chi_Y(s)$ everywhere: a nonzero value off the coordinate axes would give a positive integral on a neighborhood, and equality on the axes holds identically. Uniqueness of characteristic functions then gives independence. Conversely, independence makes the integrand zero.
\end{proof}
In particular, if $\mathcal V[X,Y]>0$ and $c>1$, then $\mathcal V[cX,Y]>\mathcal V[X,Y]$. Thus an unrestricted data-processing inequality for squared distance covariance cannot be used to prove population optimality of $R^*$.

\noindent
 \textit{$f$-divergence \citep{ali1966general}.}
Let $\mu,\nu$ be probability measures with $\mu\ll\nu$. Let $f$ be a proper lower-semicontinuous convex function on $[0,\infty)$, finite and differentiable on $(0,\infty)$, and normalized by $f(1)=0$; extend $f$ to $+\infty$ on negative arguments. Then
\begin{equation}\label{f_divergence}
 \mb D_f(\mu\|\nu):=\int f\!\left(\frac{d\mu}{d\nu}\right)d\nu.
\end{equation}
This is nonnegative by Jensen's inequality. With $f$ strictly convex, equality holds if and only if $\mu=\nu$. Its Fenchel conjugate is
\begin{equation}\label{conjugate-f}
 f^*(t):=\sup_{x\geq0}\{tx-f(x)\},\qquad t\in\mathbb R.
\end{equation}
\begin{lemma}[Variational form, cf. \cite{huang2024deep}]
Suppose additionally that $r=d\mu/d\nu>0$ $\nu$-almost everywhere, $\mb D_f(\mu\|\nu)<\infty$, and $f'(r)$ belongs to the admissible measurable discriminator class with finite expectations below. Then
\begin{equation}\label{f_divergence2}
 \mb D_f(\mu\|\nu)=\sup_D\{\mb E_\mu D-\mb E_\nu f^*(D)\},
\end{equation}
where the supremum is over measurable discriminators for which both expectations are finite, and equality is attained at $D=f'(r)$.
\end{lemma}
\begin{proof}
For every admissible $D$, the Radon--Nikodym identity and Fenchel's inequality give
\begin{equation*}
 \mathbb E_\mu D-\mathbb E_\nu f^*(D)
 =\int\{rD-f^*(D)\}\,d\nu
 \leq\int f(r)\,d\nu=\mathbb D_f(\mu\|\nu).
\end{equation*}
For $r>0$, convexity gives $f(x)\geq f(r)+f'(r)(x-r)$ for every $x\geq0$. Consequently,
\begin{equation*}
 f^*(f'(r))=rf'(r)-f(r),\qquad
 \mathbb E_\mu f'(r)-\mathbb E_\nu f^*(f'(r))=\mathbb D_f(\mu\|\nu).
\end{equation*}
The assumed integrability makes both expectations well defined, so the upper bound is attained. If the discriminator class is restricted and does not contain $f'(r)$, only the upper bound follows; the gap is the variational bias retained in \eqref{eq:disc-bias}.
\end{proof}
For the empirical penalty in \eqref{eq:rep-training}, use the variational expression with a specified discriminator-network class $\mc D$ for which the displayed supremum is finite ($\mc D=\mc D_{n_R}$ in the recovery analysis), representation-fold observations $D_{\rm rep}=\{(X_i,Y_i)\}_{i=1}^{n_R}$, and independent reference observations $W_i\sim\gamma_d$:
\begin{equation*}
 \widehat{\mb D}_{f,n_R}(R):=\sup_{D\in\mc D}
 \left\{\frac1{n_R}\sum_{i=1}^{n_R}D(R(X_i))
          -\frac1{n_R}\sum_{i=1}^{n_R}f^*(D(W_i))\right\}.
\end{equation*}
The restricted variational expression defines the empirical training penalty. It differs from the full divergence between the discrete empirical feature measure and the continuous Gaussian reference. Equation \eqref{eq:disc-bias} quantifies the population variational approximation bias in the recovery analysis.

\subsection{\texorpdfstring{Experimental Details}{Experimental Details}}\label{experiment}

Figure \ref{fig:density ratio} illustrates density-ratio estimation in a Gaussian example. Systematic evaluations of domain overlap, latent dimension, representation sufficiency, density-ratio error and run-to-run uncertainty remain future work.

\noindent
\textit{The dataset details:}
\begin{table}[tbp]
    \centering
    \renewcommand{\arraystretch}{1.3}
    
    \small\setlength{\tabcolsep}{3pt}
\begin{tabular}{ccccccc}
    \toprule
         & \textbf{Datasets} & \shortstack{\textbf{Number-}\\\textbf{Source}} & \shortstack{\textbf{Number-}\\\textbf{Target}} & \textbf{Covariates} &\textbf{Labels} & \textbf{Features} \\ 
        \multirow{5}{*}{\textbf{Tabular}} & Circle & 10000 & 10000 & 2  & 2 & (1, 32, 32) \\ 
        ~ & Moon & 10000 & 10000 &2 & 2 & (1, 32, 32) \\ 
        ~ & Breast cancer & 365 & 318 & 9 & 2 & (1, 20, 20) \\ 
        ~ & Blood & 369 & 379 & 3 & 2 & (1, 20, 20) \\ 
        ~ & Haberman & 158 & 148 & 3 & 2 & (1, 20, 20) \\ \cline{1-7}
                \multirow{4}{*}{\textbf{Office-Caltech}} & Caltech (C)& 1123 & - & - & 10 & (Variable) \\
        ~ & DSLR (D) & 157 & - & - & 10 & (3, 1000, 1000) \\ 
        ~ & Webcam (W) & 295 & - & - & 10 & (Variable) \\ 
        ~ & Amazon (A) & - & 958 & - & 10 & (3, 300, 300) \\ 
        \bottomrule
    \end{tabular}
\caption{Dataset sizes and feature dimensions. Office-Caltech counts are shown for transfer from C, D and W to A.}
    \label{tab1}
\end{table}

\noindent
\textit{Data Transformations.}
We apply the following image transformations to Office-Caltech:
\begin{itemize}
    \item Resizing each image as $(256, 256)$.
    \item Each image is normalized by channel by mean $(0.485, 0.456, 0.406)$, standard deviation $(0.229, 0.224, 0.225)$.
\end{itemize}

\noindent
\textit{Hyperparameter Settings.}
\begin{table}[tbp]
    \centering
    \renewcommand{\arraystretch}{1.3}
    
    \begin{tabular}{cccc}
    \toprule
        \textbf{Dataset} & \textbf{$lr$}& \textbf{$\alpha_w$} & \textbf{latent dim} \\ 
        Circle & 0.5 & 0.9 & 2 \\ 
        Moon & 3.0 & 0.4 & 2 \\ 
         Blood& 1.0 & 0.5 & 3 \\ 
         Breast cancer& 3.0 & 0.1 & 6 \\
         Haberman& 1.5 & 0.3 & 3 \\
        Office-Caltech & 0.2 & 0.8 & 128 \\
            \bottomrule    
    \end{tabular}
\caption{Hyperparameters. $lr$ is the learning rate for representation training.}
    \label{tab3}
\end{table}

\noindent
\textit{Network Architectures.} We display the network architecture details for three different stages here.
\begin{itemize}
\item We employ the ResNet-18 \cite{he2016deep} as the backbone for representation learning stage, which is shown in Table \ref{tab_neural1}.
    \begin{table}[tbp]%
    \centering
     \renewcommand{\arraystretch}{1.1}
    
    \begin{tabular}{ccc}
    \toprule
        Layers & Details & Output shape  \\ \midrule
        Conv-1  & $7 \times 7 $ Conv    & $(64, 112, 112) $  \\ 
        Maxpool-1  & $3 \times 3$ Maxpool    & $ (64, 56, 56) $  \\ 
       \multirow{2}{*}{Conv-2}  & \multirow{2}{*}{$\left[\begin{matrix}
           3 \times 3 ~  \text{Conv, BN, ReLU}\\
           3 \times 3 ~ \text{Conv, BN, ReLU}
       \end{matrix} \right]\times 2$} & \multirow{2}{*}{$ (64, 56, 56) $}\\
          ~ &  ~     & $ ~ $\\
        \multirow{2}{*}{Conv-3}   & \multirow{2}{*}{$\left[\begin{matrix}
           3 \times 3 ~  \text{Conv, BN, ReLU}\\
           3 \times 3 ~ \text{Conv, BN, ReLU}
       \end{matrix} \right]\times 2$}  & \multirow{2}{*}{$ (128, 28, 28) $}\\
       ~ & ~    & $ ~ $\\
        \multirow{2}{*}{Conv-4}   &  \multirow{2}{*}{$\left[\begin{matrix}
           3 \times 3 ~  \text{Conv, BN, ReLU}\\
           3 \times 3 ~ \text{Conv, BN, ReLU}
       \end{matrix} \right]\times 2$}  &\multirow{2}{*}{ $ (256, 14, 14) $}\\
        ~ & ~    & $ ~ $\\
        \multirow{2}{*}{Conv-5}   & \multirow{2}{*}{$\left[\begin{matrix}
           3 \times 3 ~  \text{Conv, BN, ReLU}\\
           3 \times 3 ~ \text{Conv, BN, ReLU}
       \end{matrix} \right]\times 2$}   & \multirow{2}{*}{$ (512, 7, 7) $}\\
        ~ & ~    & $ ~ $\\
        Avgpool   & Adaptive Average Pool    & $ (512, 1, 1) $\\
        Linear-1   & Linear   & $ (\text{latent dimension}) $\\
    \bottomrule    
    \end{tabular}
\caption{The network architecture based on ResNet-18.}
    \label{tab_neural1}
\end{table}

\item We display the DenseNet architectures below.
 \begin{table}[tbp]%
    \centering
     \renewcommand{\arraystretch}{1.1}
    
    \begin{tabular}{ccc}
    \toprule
        Layers & Details & Output shape  \\ \midrule
        
        Conv-1  & $3 \times 3 $ Conv    & $(24, 12, 12) $  \\ 
        
       \multirow{2}{*}{Dense Block-1}  & \multirow{2}{*}{$\left[\begin{matrix}
           1 \times 1 ~  \text{Conv, BN}\\
           3 \times 3 ~ \text{Conv, BN}
       \end{matrix} \right]\times 2$} & \multirow{2}{*}{$ (48, 12, 12) $}\\
          ~ &  ~     & $ ~ $\\
          
         Transition Layer-1  &BN, ReLU, $2 \times 2$ Average Pool, $1 \times 1$ Conv,     & $(24, 6, 6) $  \\ 
         
        \multirow{2}{*}{Dense Block-2}   & \multirow{2}{*}{$\left[\begin{matrix}
           1 \times 1 ~  \text{Conv, BN}\\
           3 \times 3 ~ \text{Conv, BN}
       \end{matrix} \right]\times 2$}  & \multirow{2}{*}{$ (48, 6, 6) $}\\
       ~ & ~    & $ ~ $\\

       Transition Layer-2  &BN, ReLU, $2 \times 2$ Average Pool, $1 \times 1$ Conv,     & $(24, 3, 3) $  \\ 
       
        \multirow{2}{*}{Dense Block-3}   &  \multirow{2}{*}{$\left[\begin{matrix}
           1 \times 1 ~  \text{Conv, BN}\\
           3 \times 3 ~ \text{Conv, BN}
       \end{matrix} \right]\times 2$}  &\multirow{2}{*}{ $ (48, 3, 3) $}\\
        ~ & ~    & $ ~ $\\
        
        Avgpool   & Adaptive Average Pool    & $ (48, 1, 1) $\\
        Linear-1   & Linear   & $ (1) $\\
    \bottomrule    
    \end{tabular}
\caption{The DenseNet architectures.}
    \label{densenet}
\end{table}

\item  The fully connected network with Leaky ReLU activations serves as the representation discriminator and the domain-score network for latent features.
    \begin{table}[tbp]%
    \centering
     \renewcommand{\arraystretch}{1.1}
    
    \begin{tabular}{ccc}
    \toprule
        Layers & Details & Output shape  \\ \midrule
        Linear-1  & Linear    & $(64) $  \\ 
       \multirow{2}{*}{Linear-2}  & \multirow{2}{*}{$\left[\begin{matrix}
           \text{Linear, LeakyReLU(0.2)}\\
            \text{Linear, LeakyReLU(0.2)}
       \end{matrix} \right]\times 2$} & \multirow{2}{*}{$ (128) $}\\
          ~ &  ~     & $ ~ $\\
        Linear-3   & \text{LeakyReLU(0.2), Linear}   & $ (1) $\\
    \bottomrule    
    \end{tabular}
\caption{The network architecture of full connect linear network with Leaky ReLU}
    \label{tab_neural2}
\end{table}

\item We use the network in Table \ref{tab_neural1} with a latent dimension equal to $1$ as the density ratio function if the data is raw high-dimensional images. Otherwise, we use the network in Table \ref{tab_neural2} for latent features.

\item The prediction head is a three-layer fully connected neural network.
\begin{table}[tbp]%
    \centering
     \renewcommand{\arraystretch}{1.1}
    
    \begin{tabular}{ccc}
    \toprule
        Layers & Details & Output shape  \\ \midrule
        Linear-1  & Linear, ReLU    & $(128) $  \\ 
       Linear-2   & Linear, ReLU   & $ (128) $\\
        Linear-3   & Linear   & $ (\text{category number}) $\\
    \bottomrule    
    \end{tabular}
\caption{The network architecture of prediction head}
    \label{tab5}
\end{table}      
\end{itemize}

\begin{figure}[tbp]
    \centering
    \includegraphics[width = 0.6\textwidth]{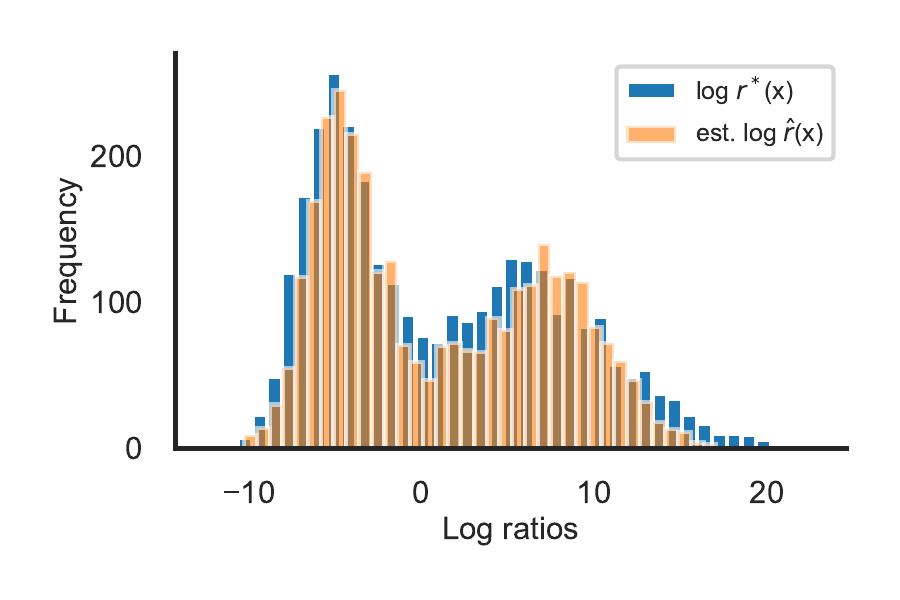}
    \caption{True and estimated log density ratios for $N(0, 0.8)$ and $N(3, 1.0)$. Blue and orange histograms show $\log r^*$ and $\log\hat r$, respectively.}
    \label{fig:density ratio}
\end{figure}

\clearpage
\bibliography{refer}
\end{document}